\documentclass{IEEEoj}

\usepackage{cite}
\usepackage{url}
\usepackage{verbatim}
\usepackage{textcomp}
\usepackage{silence} % Filter warnings
\usepackage{microtype} % Improves justification and reduces overfull/underfull warnings

\usepackage{amsmath,amssymb,amsfonts,mathrsfs}
\usepackage{mathtools}
\usepackage{amsthm}
\usepackage{siunitx}

\usepackage{graphicx}
\usepackage[dvipsnames]{xcolor}
\usepackage{svg}
\usepackage{tikz}
\usetikzlibrary{external, spy}

\usepackage{pgfplots}
\usepgfplotslibrary{fillbetween, statistics}
\pgfplotsset{compat=1.18}
\usepackage{tikzducks} % For dummy figures

\usepackage[font={footnotesize}]{caption}
\usepackage[font={footnotesize}]{subcaption}

\usepackage{float}
\usepackage{array}
\usepackage{stfloats}
\usepackage{booktabs}
\usepackage{adjustbox}
\usepackage{multirow}
\let\labelindent\relax
\usepackage{enumitem}

\usepackage{algorithm}
\usepackage[noEnd=true]{algpseudocodex}

\usepackage{hyperref}

\graphicspath{{figures/}}

\theoremstyle{plain}

\newtheorem{lemma}{Lemma}

\theoremstyle{definition}
\newtheorem{definition}{Definition}

\newtheorem*{runningexample}{Running Example}

\algrenewcommand\algorithmicprocedure{\textbf{Function}}

\definecolor{BitA}{RGB}{238,242,255} % soft indigo
\definecolor{BitB}{RGB}{236,253,245} % soft teal
\definecolor{BitC}{RGB}{255,247,237} % soft amber
\definecolor{BitAframe}{RGB}{99,102,241}
\definecolor{BitBframe}{RGB}{16,185,129}
\definecolor{BitCframe}{RGB}{245,158,11}

\newcommand{\BitBoxA}[1]{\fcolorbox{BitAframe}{BitA}{\strut #1}}
\newcommand{\BitBoxB}[1]{\fcolorbox{BitBframe}{BitB}{\strut #1}}
\newcommand{\BitBoxC}[1]{\fcolorbox{BitCframe}{BitC}{\strut #1}}

\newcommand{\bx}{\boldsymbol{x}}
\newcommand{\by}{\boldsymbol{y}}
\newcommand{\bu}{\boldsymbol{u}}

\newcommand{\bc}{\boldsymbol{c}}
\newcommand{\bq}{\boldsymbol{q}}
\newcommand{\bkappa}{\boldsymbol{\kappa}}
\newcommand{\beps}{\boldsymbol{\epsilon}}
\newcommand{\bSigma}{\boldsymbol{\Sigma}}

\DeclareMathOperator{\sign}{sgn}
\newcommand{\predrob}{\varrho}
\newcommand{\amin}{\bigtriangledown}
\newcommand{\amax}{\bigtriangleup}

\newcommand{\prelex}{\preceq_{\text{lex}}}      % Lex Precedes Equal
\newcommand{\preslex}{\prec_{\text{lex}}}      % Lex Precedes Equal
\newcommand{\succlex}{\succeq_{\text{lex}}}      % Lex Succeeds Equal
\newcommand{\pres}{\preceq_{\text{spec}}}       % Rule Precedes Equal
\newcommand{\predisc}{\preceq_{\text{disc}}}    % Discrete Precedes Equal
\newcommand{\presdisc}{\prec_{\text{disc}}}    % Discrete Precedes Strict
\newcommand{\precont}{\preceq_{\text{cont}}}    % Continuous Precedes Equal
\newcommand{\prescont}{\prec_{\text{cont}}}    % Continuous Precedes Strict

\makeatletter
\renewcommand{\p@subsection}{\thesection\mbox{-}}
\renewcommand{\p@subsubsection}{\thesection\mbox{-}\thesubsection}
\makeatother

\AtBeginDocument{\definecolor{ojcolor}{RGB}{0, 115, 120}}
\def\OJlogo{\vspace{-14pt}\includegraphics[height=23pt]{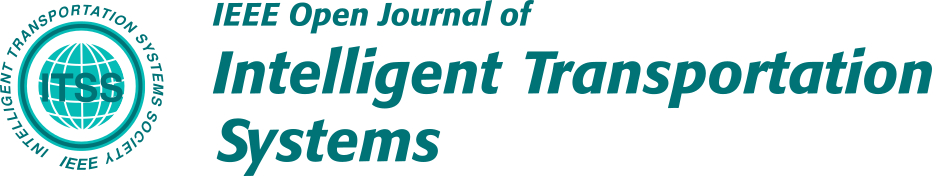}}

\begin{document}

% Running headers
\receiveddate{XX Month, XXXX}
\reviseddate{XX Month, XXXX}
\accepteddate{XX Month, XXXX}
\publisheddate{XX Month, XXXX}
\currentdate{XX Month, XXXX}
\doiinfo{OJITS.2026.1234567}
\jvol{XX}
\pubyear{2026}

% Title and running headers
\title{Lexicographic Minimum-Violation Motion Planning using Signal Temporal Logic}
\markboth{Lexicographic Minimum-Violation Motion Planning using Signal Temporal Logic}{Halder \textit{et al.}}

% Authors
\author{PATRICK HALDER\textsuperscript{1,2},
    LOTHAR KILTZ\textsuperscript{2},
    HANNES HOMBURGER\textsuperscript{3},
    JOHANNES REUTER\textsuperscript{3},
    AND MATTHIAS ALTHOFF\textsuperscript{1}}

% Affiliations
\affil{School of Computation, Information and Technology, Technical University of Munich, 85748 Garching, Germany}
\affil{ZF Friedrichshafen AG, 88046 Friedrichshafen, Germany}
\affil{Institute of System Dynamics, HTWG Konstanz, 78562 Konstanz, Germany}
\corresp{CORRESPONDING AUTHOR: P. HALDER (e-mail: patrick.halder@tum.de)}

% Abstract and keywords
\begin{abstract}
    Motion planning for autonomous vehicles often requires satisfying multiple conditionally conflicting specifications. In situations where not all specifications can be met simultaneously, minimum-violation motion planning maintains system operation by minimizing violations of specifications in accordance with their priorities. Signal temporal logic (STL) provides a formal language for rigorously defining these specifications and enables the quantitative evaluation of their violations. However, a total ordering of specifications yields a lexicographic optimization problem, which is typically computationally expensive to solve using standard methods. We address this problem by discretizing the multi-objective lexicographic optimization problem via non-uniform quantization and transforming it into a single-objective optimization problem using bit-shifting. Specifically, we extend a deterministic model predictive path integral (MPPI) solver to efficiently solve optimization problems without quadratic input cost. Additionally, a novel predicate robustness measure that combines spatial and temporal violations is introduced. Our results show that the proposed method offers an interpretable and scalable solution for lexicographic STL minimum-violation motion planning using a single-objective solver.
\end{abstract}


\maketitle

% Copyright notice (uncomment if needed)
%\copyrightnotice

% ===================================
% Main Content Sections
% ===================================

\section{Introduction}
\IEEEPARstart{M}{otion} planning for autonomous vehicles requires satisfying numerous specifications, such as collision avoidance, traffic rules, and schedules \cite{Mehdipour2023, Althoff2025}. In complex environments, these specifications are often in conflict and cannot be satisfied simultaneously. To resolve such situations systematically, specifications can be prioritized, enabling a deliberate selection of a minimum-violation maneuver. Fig.~\ref{fig:RunningExampleScenario} provides an illustrative example: complying with the lane boundaries prevents the ego vehicle from passing the broken-down vehicle, creating a conflict between making progress and adhering to the in-lane-driving specification.

In this work, we present a novel minimum-violation motion planner utilizing totally ordered signal temporal logic (STL) specifications. By discretizing the underlying multi-objective lexicographic optimization problem and transforming it into a single-objective optimization problem using bit-shifting, we enable the application of standard single-objective solvers. Specifically, we extend the deterministic model predictive path integral (MPPI) solver from our previous work~\cite{Halder2025ICRA} by eliminating the quadratic input cost and dynamically adapting the number of used samples. The resulting framework is highly flexible and supports various STL robustness definitions without restricting STL fragments or predicates. Furthermore, we introduce a new predicate robustness measure that jointly quantifies spatial and temporal violations.

\vspace{3mm}

\begin{figure}[htb]
    \centering
    \def\svgwidth{\linewidth}
    {
        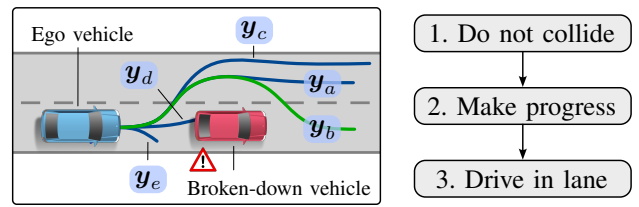
    }
    \caption{An overtaking scenario involving a broken-down vehicle and five example trajectories: $\by_a, \by_b, \by_c, \by_d$, and $\by_e$. According to the specification prioritization (the top specification has the highest priority; the bottom specification has the lowest), $\by_d$ is the worst because it violates the collision avoidance specification. Trajectory $\by_c$ is collision-free and makes the most progress, but fails to eventually return to the lane. Trajectories $\by_a$ and $\by_b$ make similar progress to $\by_c$, but $\by_a$ also fails to eventually return to the lane, making $\by_b$ the best trajectory.}
    \label{fig:RunningExampleScenario}
\end{figure}

\subsection{Related Work}
We subsequently review related work on motion planning subject to STL specifications, prioritized motion planning, and minimum-violation motion planning.

% ============================================================================
\subsubsection{Motion Planning subject to STL specifications}
STL-based motion planning requires formalizing objectives, defining their robustness measures, and optimizing these measures.

\paragraph*{Formalizing Objectives in Temporal Logic}
Objectives for cyber-physical systems are often specified in natural language, which is inherently vague and ambiguous. A formal representation of objectives is therefore crucial: it removes linguistic ambiguity, enables machine interpretability, and supports the formal verification of system behavior. Temporal logic is a common choice, as it can describe system behavior over time while remaining intuitive to humans~\cite{Mehdipour2023}. It is widely used in verification, runtime monitoring, and planning, as it provides a rigorous mathematical foundation that enables formal analysis of system behavior.
One of the most commonly used temporal logics is linear temporal logic (LTL)~\cite{Rizaldi2017, Esterle2020, Tumova2013a, ReyesCastro2013, Rong2020}. While LTL is widely adopted as a specification language in verification, its standard Boolean semantics is less suitable for motion planning, where a quantitative measure of satisfaction is often required to guide optimization algorithms. Consequently, applying LTL to motion planning usually necessitates defining auxiliary measures to quantify satisfaction and violation~\cite{Schlueter2018, Vasile2017}.

In contrast to LTL, metric temporal logic (MTL) and STL provide an inherent robustness measure~\cite{Bartocci2018}, making them well-suited for expressing spatio-temporal properties. Consequently, motion planning for autonomous vehicles predominantly employs MTL~\cite{Maierhofer2020, Maierhofer2022, Esterle2020, Krasowski2021, Shi2025} and STL~\cite{Hekmatnejad2019, Arechiga2019, Halder2023}.

In~\cite{Siu2023, Hurley2024}, the interpretability of STL specifications is questioned based on user studies. However, in engineering practice, temporal logic specifications are either formalized by experts~\cite{Maierhofer2020} or derived from data~\cite{Linard2022, Khanna2024}.

\paragraph*{STL Robustness Measures}
Many works on temporal-logic motion planning optimize the standard space robustness of STL~\cite{Belta2019, Kurtz2022}. In contrast, only a few works focus on the time robustness~\cite{Rodionova2022}, and, to the best of our knowledge, no work optimizes the standard space-time robustness of~\cite{Donze2010}. Instead, a wide range of robustness measures is proposed, each targeting specific use cases or properties, e.g., smoothness and differentiability, monotonicity, soundness, or completeness. For instance, smooth approximations~\cite{Pant2017, Pant2018, Dawson2022, Gilpin2021, Kurtz2021, Welikala2023, Uzun2024} provide differentiability for gradient-based solvers, while weighted variants of STL enable prioritization of subformulas~\cite{Mehdipour2021, Cardona2023, Baharisangari2022}. Others define averaging-based and accumulative measures~\cite{Lindemann2017, Haghighi2019, Mehdipour2019, Mehdipour2019a, Varnai2020, Zhao2022, Halder2022} or a model predictive robustness~\cite{Lin2023}. The work in \cite{Mehdipour2025} shows that many of the aforementioned measures are special cases of a generalized mean robustness. Temporal deviations are addressed by time-robustness measures~\cite{Lindemann2022, Rodionova2023, Yu2023, Verhagen2024}. In addition, some works propose custom combinations, such as products of space and time robustness~\cite{Karlsson2020, Karlsson2021}, a weighted sum of space robustness with a measure of time tolerances~\cite{Lin2020}, or a combination of violation magnitude and violation duration~\cite{Finkeldei2025}.

\paragraph*{Optimizing STL Robustness}
Considering the standard space and time robustness measures in optimization problems poses significant challenges due to their non-smoothness and non-convexity. Mostly, these robustness measures are encoded as mixed-integer constraints~\cite{Sadraddini2019, Sahin2020, Cardona2023, Aasi2021, Kurtz2022, Yuan2024, Aasi2025, Rodionova2021, Rodionova2022, Rodionova2023, Verhagen2024, Yu2023}. However, these encodings introduce many binary variables, resulting in an exponential computational complexity~\cite{Kurtz2022}.

To enable gradient-based optimization, the minimum and maximum operators are often replaced with smooth approximations. This allows the use of sequential quadratic programming~\cite{Pant2017, Pant2018, Gilpin2021, Welikala2023}, first-order methods~\cite{Leung2023, Mehdipour2019, Mehdipour2021}, or convex optimization~\cite{Mao2022, Claudet2024, Takayama2025}. However, these methods trade the exactness of the original non-smooth robustness for computational tractability.

Control barrier functions encode the robustness using barrier certificates~\cite{Lindemann2019, Lindemann2020, Charitidou2021, Xiao2021, Lindemann2020a, Huang2024, Ruo2024, Yu2024}, and are often formulated as mixed-integer or quadratic programs. However, these methods may not find feasible solutions for complex STL specifications, due to the non-convex, combinatorial structure of the resulting constraints.

Sampling-based methods, such as rapidly exploring random trees~\cite{Vasile2017, Karlsson2020, Barbosa2021, Linard2023, Sewlia2023, Marchesini2025} or MPPI~\cite{Halder2025ICRA, Varnai2022}, can handle arbitrary STL robustness measures, though they generally only provide probabilistic completeness. Learning-based methods, such as supervised learning~\cite{Liu2022, Leung2022, Meng2023} or reinforcement learning~\cite{Singh2023, Saxena2023}, enable fast execution but typically lack formal guarantees.

Despite the variety of STL optimization methods, most solvers impose restrictions: they are limited to STL fragments~\cite{Lindemann2019, Barbosa2019, Lindemann2020, Kapoor2024, Buyukkocak2025}, restrict predicates to be linear~\cite{Sadraddini2015, Cardona2023}, piecewise-linear~\cite{Farahani2015}, or convex~\cite{Halder2022, Kurtz2022}, or alter the original problem by smoothing the robustness~\cite{Pant2017, Pant2018, Dawson2022, Gilpin2021, Kurtz2021, Welikala2023}. In this work, we present a method that imposes no restrictions on STL specifications.
To handle conflicting specifications, we discuss prioritization methods next.

% ============================================================================
\subsubsection{Prioritized Motion Planning}
We first review prioritization methods in motion planning, followed by lexicographic optimization methods.

\paragraph*{Methods for Prioritization}
Although conflicts of specifications can be resolved offline by explicitly encoding every possible scenario (e.g., allowing the use of a bike lane to avoid a crash), enumerating all combinations of conditionally conflicting specifications is cumbersome and often intractable~\cite{Gambo2021}. Therefore, such conflicts are typically addressed online during planning.

Weighting methods~\cite{Marler2004} are the most common form of prioritization. The weighted-sum approach~\cite{Paden2016} is widely used, but it cannot capture non-convex regions of the Pareto front~\cite{Das1997, Branke2008}. Other variants include the weighted product method and the weighted Tchebycheff method~\cite{Yi2015, Thoma2023, Wilde2024}. However, weights are highly scenario-dependent~\cite{Abouelazm2024, Halder2023}, prone to misinterpretation, and often require normalization, which complicates tuning tremendously~\cite{Mehdipour2021, Cardona2023, Halder2025IV}. Despite the ease of integrating weighting methods into optimization problems, the high tuning effort remains a key drawback.

On the contrary, hierarchical methods allow qualitative prioritization based on order relations. Lexicographic ordering is the most prevalent method~\cite{Halder2022, Halder2023, Schnurr2018, Khosravani2018, Shan2020, Veer2023, Penlington2024, Molins2025}. In~\cite{Xiao2021}, rules are organized into totally ordered equivalence classes and lexicographic optimization is performed across the classes; within each class, rules are aggregated by a weighted Tchebycheff formulation. A generalization to lexicographic preferences is provided by the rulebook framework~\cite{Censi2019}, where specifications are preordered. In~\cite{Halder2025IV}, the relation of rulebooks to mixed Pareto-lexicographic optimization~\cite{Lai2022} is shown. Rulebooks are more interpretable than machine learning models~\cite{Helou2021}, but they can yield multiple incomparable solutions~\cite{Halder2025IV}, raising the issue of solution selection. While hierarchical methods allow straightforward prioritization of specifications, incorporating them into optimization problems remains a challenge.
In this work, we focus on lexicographic optimization and subsequently review methods for solving such problems.

\paragraph*{Lexicographic Optimization Methods}
Lexicographic optimization methods can be broadly categorized into preemptive, search-based, sampling-based, and scalarization approaches. In the commonly used preemptive scheme, an optimization problem is solved for each objective, with solutions to higher-priority objectives constraining those to lower-priority objectives~\cite{Marler2004, Cococcioni2018, Wang2020, Hussain2021, Zanardi2021, Fiaschi2022, Halder2023}. Consequently, the complexity of the solution method strongly depends on the number of objectives. Search-based methods instead compare path costs lexicographically, e.g., using Dijkstra~\cite{Shan2020, MaristanydelasCasas2023} or A*~\cite{Halder2022} algorithms, but they require cost separability, making them unsuitable for STL robustness. Sampling-based approaches offer fast convergence and can handle arbitrary robustness measures, but they are only probabilistically complete~\cite{Halder2025IV}.

Scalarization methods aim to transform the lexicographic problem into a weighted-sum problem, enabling the use of standard single-objective solvers. However, an order-preserving utility function does not exist in general, i.e., a set of equivalent weights~\cite{Debreu1983}. Equivalent weights can be constructed only under specific assumptions, such as linearity~\cite{Sherali1983, Sathya2021}, convexity~\cite{Schnurr2018, Zarepisheh2011}, or when Boolean objectives are used~\cite{Rahmani2020, Veer2023}. A utility function that approximates the lexicographic order is proposed in~\cite{Penlington2024}.

% ============================================================================
\subsubsection{Minimum-Violation Motion Planning}
Minimum-violation motion planning addresses situations where it is impossible to satisfy all specifications simultaneously. It combines a quantitative violation measure with specification prioritization to determine a minimum-violation trajectory.
When specifications are formalized in LTL, minimum-violation motion planning approaches typically quantify violations using weighted automata~\cite{Lahijanian2015, Karlsson2018} and compare trajectories lexicographically~\cite{Tumova2013, Wongpiromsarn2021}. However, these approaches rely on auxiliary violation measures that are not inherent to the logic itself~\cite{Schlueter2018}, severely limiting their usability.

For STL specifications, violation is typically quantified using space robustness, and the prioritization is expressed through weighted sums~\cite{Karlsson2020, Karlsson2021, Baldini2024}, total orders~\cite{Halder2022, Halder2023}, or preorders~\cite{Censi2019}.

\subsection{Contributions}

We present a minimum-violation motion planning framework that supports various STL robustness definitions without restricting the STL fragment or predicates. Conflicting specifications are resolved using a lexicographic ordering within a single-objective optimization framework. In detail, our contributions are:
\begin{itemize}
    \item introducing a novel predicate robustness measure that combines spatial and temporal robustness;
    \item providing a new scalarization technique that transforms the multi-objective lexicographic optimization problem into a single-objective problem using non-uniform quantization and bit-shifting, allowing the use of single-objective solvers;
    \item extending the deterministic MPPI solver~\cite{Halder2025ICRA} to handle non-smooth, non-convex optimization problems by eliminating the mandatory quadratic input cost. We further provide adaptations improving solution quality and convergence time; and
    \item extensively validating the proposed approach through numerical experiments, which analyze the impact of cost discretization, evaluate the solver adaptations, demonstrate its applicability to complex autonomous driving scenarios, and benchmark various STL robustness measures.
\end{itemize}
The remainder of this article is structured as follows: after introducing the preliminaries in Sec.~\ref{sec:Preliminaries}, we present our problem statement in Sec.~\ref{sec:ProblemStatement}. Our solution is discussed in Sec.~\ref{sec:Solution}, and the extended deterministic MPPI solver is introduced in Sec.~\ref{sec:Solver}. Finally, we evaluate our minimum-violation motion planning approach in Sec.~\ref{sec:NumericalExperiments} and conclude in Sec.~\ref{sec:Conclusions}.

\section{Preliminaries}\label{sec:Preliminaries}
This section introduces fundamentals for our approach: the system model, definitions of order theory, STL, preemptive lexicographic optimization, and the deterministic MPPI solver.

\subsection{System Definition}
Let $k \in \mathbb{N}_0$ denote the discrete time index, with the corresponding continuous time given by $t_k:= k\,\Delta t$, where $\Delta t \in \mathbb{R}_{>0}$ is the time increment. Without loss of generality, the initial time index is $k_0 = 0$, and the final time index is $K \in \mathbb{N}$. We model the vehicle motion using a deterministic, nonlinear, discrete-time system
\begin{equation}
    \label{eq:SystemDefinition}
    \begin{aligned}
        \bx_{k+1} & = f(\bx_k, \bu_k), \\
        \by_k     & = g(\bx_k, \bu_k),
    \end{aligned}
\end{equation}
with $\bx_k \in \mathbb{R}^{n_x}$, $ \bu_k \in \mathbb{R}^{n_u}$, and $\by_k \in \mathbb{R}^{n_y}$, representing the state, input, and output, respectively, at discrete time  $k \in [0, K]$. Further, the inputs are bounded elementwise by $\underline{\bu} \le \bu_k \le \overline{\bu}$ for all $k \in [0,K]$, with $\underline{\bu}, \overline{\bu} \in \mathbb{R}^{n_u}$. The solution of~\eqref{eq:SystemDefinition} over the horizon $[0,K]$ for an initial state $\bx_0 \in \mathbb{R}^{n_x}$ and an input trajectory $\bu([0,K]) := [\bu_0,\ldots,\bu_{K}]$ is the state trajectory $\bx([0,K]) := [\bx_0,\bx_1,\ldots,\bx_K]$. The corresponding output trajectory is $\by([0,K]) := [\by_0,\by_1,\ldots,\by_K]$. Subsequently, we use $\bx$, $\bu$, and $\by$ as shorthand notations for $\bx([0,K])$, $\bu([0,K])$, and $\by([0,K])$, respectively. Further, let $\mathcal{Y}$ be the set of all possible output trajectories.

\subsection{Order-Theoretic Definitions}
In this subsection, we present the order-theoretic notions used in our proposed approach.

\begin{definition}[Partially Ordered Set, Partial Order~{\cite[Def.~1.1]{Schroeder2016}}]
    A \emph{partially ordered set} is a tuple $(\mathcal{Q}, \preceq)$, where $\mathcal{Q}$ is a set and $\preceq\ \subseteq \mathcal{Q} \times \mathcal{Q}$ is a reflexive, antisymmetric, and transitive binary relation. We call $\preceq$ a \emph{partial order}.
\end{definition}
For convenience, we define $q \preceq q' \Longleftrightarrow (q, q') \in\ \preceq$, with $q, q' \in \mathcal{Q}$, and say that $q$ \textit{precedes} $q'$. Let us further introduce the following shorthand notations:
\begin{equation*}
    \begin{alignedat}{2}
        \text{\textit{strict precedence:}} & \ q \prec q' \Longleftrightarrow q \preceq q' \land q' \npreceq q, \\
        \text{\textit{indifference:}}      & \ q \sim q' \Longleftrightarrow q \preceq q' \land q' \preceq q.
    \end{alignedat}
\end{equation*}

\begin{definition}[Totally Ordered Set, Totality~{\cite[Def.~1.1 and~2.1]{Schroeder2016}}]\label{def:TotallyOrderedSet}
    A \emph{totally ordered set} is a tuple $(\mathcal{Q}, \preceq)$, where $\mathcal{Q}$ is a set and $\preceq\ \subseteq \mathcal{Q} \times \mathcal{Q}$ is a partial order that additionally satisfies \emph{totality}; i.e., for all $q, q' \in \mathcal{Q}$ it holds that $q \preceq q' \lor q' \preceq q$.
\end{definition}

\begin{definition}[Lexicographically Ordered Set~{\cite[Ex.~1.2-9]{Schroeder2016}}]\label{def:LexicographicallyOrderedSet}
    A \emph{lexicographically ordered set} is a tuple $(\mathcal{Q}, \prelex)$, where $\mathcal{Q} \subseteq \mathbb{R}^{N}$ and $\prelex\ \subseteq \mathcal{Q} \times \mathcal{Q}$ is defined as follows: for $\bq=[q_1,\dots,q_N]^{\top}\in\mathcal{Q}$ and $\bq'=[q'_1,\dots,q'_N]^{\top}\in\mathcal{Q}$, we write $\bq \prelex \bq'$ iff $q_i = q'_i\ \text{for all } i\in\{1,\dots,N\}$ or there exists an index $i^*\in\{1,\dots,N\}$ such that $q_i = q'_i$ for all $i < i^*$ and $q_{i^*} < q'_{i^*}$.
\end{definition}

\begin{definition}[Utility Representability, Utility Function~{\cite[Sec.~1]{Beardon2002}}]
    Let ${(\mathcal{Q},\preceq)}$ be a totally ordered set. We say that $(\mathcal{Q},\preceq)$ is \emph{representable by a utility function} if there exists a function $h:\mathcal{Q}\to\mathbb{R}$ such that
    \begin{equation*}
        q \preceq q' \iff h(q)\leq h(q'),\ \forall q,q'\in\mathcal{Q}.
    \end{equation*}
    The function $h$ is then called a \emph{utility function}.
\end{definition}

% ----------------------------------------------------------------------------

\subsection{Signal Temporal Logic}\label{sec:SignalTemporalLogicDefinition}

The syntax of STL is given by the grammar~\mbox{\cite[Sec.~2.1]{Bartocci2018}}
\begin{equation*}
    \varphi:= \mu \; \vert\; \lnot \varphi\; |\; \varphi_1 \land \varphi_2 \;|\; \varphi_1 \mathbf{U}_{I} \varphi_2 \;|\; \varphi_1 \mathbf{S}_{I} \varphi_2,
\end{equation*}
where $\mu$ denotes a predicate of the form $\mu := p(\by, k) \ge 0$, with $p: \mathbb{R}^{n_y \times (K+1)} \times \mathbb{N}_0 \rightarrow \mathbb{R}$. The expressions $\varphi$, $\varphi_1$, and $\varphi_2$ represent STL specifications, while $\lnot$ and $\land$ correspond to logical negation and conjunction, respectively. Disjunction is introduced as a derived operator $\varphi_1 \lor \varphi_2 := \lnot(\lnot \varphi_1 \land \lnot \varphi_2)$, and implication is defined as $\varphi_1 \Longrightarrow \varphi_2 := \lnot \varphi_1 \lor \varphi_2$. The temporal operator $\mathbf{U}_{I}$ specifies that $\varphi_1$ must hold until $\varphi_2$ becomes true at some discrete time $k$ in the interval $I := [\underline{k}, \overline{k}] \subseteq \mathbb{N}_0$, with $\underline{k} \leq \overline{k}$. The since operator $\mathbf{S}_{I}$ is defined equivalently. Other derived temporal operators are ${\mathbf{F}_{I}\varphi := \mathrm{True}\mathbf{U}_{I} \varphi}$ (\textit{eventually}), ${\mathbf{O}_{I}\varphi := \mathrm{True}\mathbf{S}_{I} \varphi}$ (\textit{once}), ${\mathbf{G}_{I}\varphi := \lnot \mathbf{F}_{I}\lnot\varphi}$ (\textit{globally}), and ${\mathbf{H}_{I}\varphi := \lnot \mathbf{O}_{I}\lnot\varphi}$ (\textit{historically}). When $I=[0, K]$, we omit $I$ from the operator. Let $(\by, k) \models \varphi$ indicate that trajectory $\by$ satisfies STL specification $\varphi$ at discrete time~$k$.

Next, we introduce a quantitative semantics, called robustness. The robustness $\eta^{\varphi}(\by,k): \mathbb{R}^{n_y \times (K+1)} \times \mathbb{N}_0 \rightarrow \overline{\mathbb{R}}$ characterizes the degree to which an STL specification $\varphi$ is satisfied or violated by a trajectory $\by$ at discrete time $k \in [0, K]$, where $\overline{\mathbb{R}} := \mathbb{R} \cup \{-\infty, +\infty\}$. We define it as follows~\cite[Sec. 2.2]{Bartocci2018}:
\begin{align*}
    \eta^{\mathrm{True}} (\by, k)                    & := +\infty,                                                                                                           \\
    \eta^{\mu} (\by, k)                              & := \predrob^{\mu}(\by, k),                                                                                            \\
    \eta^{\lnot\varphi}(\by, k)                      & := -\eta^{\varphi}(\by, k),                                                                                           \\
    \eta^{\varphi_1 \land \varphi_2}(\by, k)         & := \amin\big(\eta^{\varphi_1}(\by, k), \eta^{\varphi_2}(\by, k)\big),                                                 \\
    \eta^{\varphi_1\mathbf{U}_{I} \varphi_2}(\by, k) & := \amax_{k^\prime \in (k+I) \cap [0, K]}\big(\amin( \eta^{\varphi_2}(\by, k^\prime),                                 \\
                                                     & \hspace{0.55cm}\amin_{k^{\prime\prime} \in [k, k^\prime)}\left(\eta^{\varphi_1}(\by, k^{\prime\prime})\right)) \big), \\
    \eta^{\varphi_1\mathbf{S}_{I} \varphi_2}(\by, k) & := \amax_{k^\prime \in (k-I) \cap [0, K]}\big(\amin(\eta^{\varphi_2}(\by, k^\prime),                                  \\
                                                     & \hspace{0.55cm}\amin_{k^{\prime\prime} \in (k^\prime, k]}\left(\eta^{\varphi_1}(\by, k^{\prime\prime})\right)) \big),
\end{align*}
with $k \mathbin{\circ} I := \{\,k \circ k' \mid k' \in I\,\}$, and $\circ \in \{-,+\}$. Here, $\predrob: \mathbb{R}^{n_y \times (K+1)} \times \mathbb{N}_0 \rightarrow \mathbb{R}$ denotes a predicate robustness measure, and $\amin,\amax$ are custom minimum and maximum functions. Various definitions for $\predrob$, $\amin$, and $\amax$ exist, yielding different robustness measures that each exhibit distinct properties. In this work, we compare the selected subset of robustness measures shown in Tab.~\ref{tab:RobustnessDefinitions}. The corresponding definitions of $\predrob$ are presented in Tab.~\ref{tab:PredicateRobustnessDefinitions}. Our proposed space-left-time robustness for predicates $\predrob_{\text{space-left-time}}^{\mu}$ is discussed further in Sec.~\ref{sec:RobustnessMeasure}. The definitions for $\amin$ and $\amax$ are provided in Appendix~\ref{app:MinMaxDefinitions}. For instance, the space robustness uses the standard $\min$ and $\max$ operators.

\begin{table}[!tb]
    \caption{Definitions of the used robustness measures.}
    \renewcommand{\arraystretch}{1.1}
    \centering
    \begin{adjustbox}{max width=1.0\columnwidth}
        \begin{tabular}{llll}
            \toprule
            \textbf{Measure}                                                 & \textbf{Pred. Rob.} ($\predrob$)    & \textbf{Min} ($\amin$)   & \textbf{Max} ($\amax$)   \\
            \midrule
            Space~\cite{Donze2010} ($\eta_{\text{space}}$)                   & $\predrob_{\text{space}}$           & $\amin_{\text{std}}$     & $\amax_{\text{std}}$     \\
            Left-Time~\cite{Donze2010} ($\eta_{\text{left-time}}$)           & $\predrob_{\text{left-time}}$       & $\amin_{\text{std}}$     & $\amax_{\text{std}}$     \\
            Right-Time~\cite{Donze2010} ($\eta_{\text{right-time}}$)         & $\predrob_{\text{right-time}}$      & $\amin_{\text{std}}$     & $\amax_{\text{std}}$     \\
            Combined-Time~\cite{Rodionova2023} ($\eta_{\text{comb-time}}$)   & $\predrob_{\text{comb-time}}$       & $\amin_{\text{std}}$     & $\amax_{\text{std}}$     \\
            \midrule
            Duration~\cite{Finkeldei2025} ($\eta_{\text{dur}}$)              & $\predrob_{\text{space}}$           & $\amin_{\text{dur}}$     & $\amax_{\text{dur}}$     \\
            Duration-Severity~\cite{Finkeldei2025} ($\eta_{\text{dur-sev}}$) & $\predrob_{\text{space}}$           & $\amin_{\text{dur-sev}}$ & $\amax_{\text{dur-sev}}$ \\
            Smooth~\cite{Gilpin2021} ($\eta_{\text{smooth}}$)                & $\predrob_{\text{space}}$           & $\amin_{\text{smooth}}$  & $\amax_{\text{smooth}}$  \\
            AGM~\cite{Mehdipour2019a} ($\eta_{\text{agm}}$)                  & $\predrob_{\text{space}}$           & $\amin_{\text{agm}}$     & $\amax_{\text{agm}}$     \\
            New~\cite{Varnai2020} ($\eta_{\text{new}}$)                      & $\predrob_{\text{space}}$           & $\amin_{\text{new}}$     & $\amax_{\text{new}}$     \\
            Power-Mean~\cite{Mehdipour2025} ($\eta_{\text{pm}}$)             & $\predrob_{\text{space}}$           & $\amin_{\text{pm}}$      & $\amax_{\text{pm}}$      \\
            \midrule
            Space-Left-Time ($\eta_{\text{space-left-time}}$)                & $\predrob_{\text{space-left-time}}$ & $\amin_{\text{std}}$     & $\amax_{\text{std}}$     \\
            \bottomrule
        \end{tabular}
        \label{tab:RobustnessDefinitions}
    \end{adjustbox}
    \vspace{-4mm}
\end{table}

\begin{table*}[!tb]
    \caption{Predicate robustness definitions. Let $w \in \mathbb{R}_{>0}$ and $\sign(\kappa) := 1$ if $\kappa \ge 0$, else $-1$, with $\kappa \in \mathbb{R}$.}
    \renewcommand{\arraystretch}{1.2}
    \centering
    \begin{adjustbox}{max width=2.0\columnwidth}
        \begin{tabular}{l r@{ $\;:=\;$ } l}
            \toprule
            \textbf{Predicate Rob.}         & \multicolumn{2}{c}{\textbf{Definition}}                                                                                                                                                                                                                                                                                                      \\
            \midrule
            Space~\cite{Donze2010}          & $\predrob_{\text{space}}^{\mu}(\by, k)$                                                                                                      & $p_{\mu}(\by, k)$                                                                                                                                                                             \\
            Left-Time~\cite{Donze2010}      & $\predrob_{\text{left-time}}^{\mu}(\by, k)$                                                                                                  & $\sign(p_{\mu}(\by, k))\,  \max_{\tau \in \mathbb{N}_0} \tau \; \text{s.t.\ } \sign(p_{\mu}(\by, k')) = \sign(p_{\mu}(\by, k)), \forall k' \in [k, k+\tau], k+\tau \leq K$                    \\
            Right-Time~\cite{Donze2010}     & $\predrob_{\text{right-time}}^{\mu}(\by, k)$                                                                                                 & $\sign(p_{\mu}(\by, k))\, \max_{\tau \in \mathbb{N}_0} \tau \; \text{s.t.\ } \sign(p_{\mu}(\by, k')) = \sign(p_{\mu}(\by, k)), \forall k' \in [k-\tau, k], k-\tau \geq 0$                     \\
            Comb.-Time~\cite{Rodionova2023} & $\predrob_{\text{comb-time}}^{\mu}(\by, k)$                                                                                                  & $\sign(p_{\mu}(\by, k))\, \max_{\tau \in \mathbb{N}_0} \tau \; \text{s.t.\ } \sign(p_{\mu}(\by, k')) = \sign(p_{\mu}(\by, k)), \forall k' \in [k-\tau, k+\tau], k-\tau \geq 0, k+\tau \leq K$ \\
            \midrule
            Space-Left-Time                 & $\predrob_{\text{space-left-time}}^{\mu}(\by, k)$                                                                                            & $\sign(p_{\mu}(\by, k))\, \max_{\tau \in \mathbb{N}_0} \left( w \frac{\tau}{K} + |p_{\mu}(\by, k+\tau)| \right)$                                                                              \\
                                            & \multicolumn{2}{l}{\hspace{6cm} $\text{s.t.\ } \sign(p_{\mu}(\by, k')) = \sign(p_{\mu}(\by, k)), \forall k' \in [k, k+\tau], k+\tau \leq K$}                                                                                                                                                                                                 \\
            \bottomrule
        \end{tabular}
        \label{tab:PredicateRobustnessDefinitions}
    \end{adjustbox}
    \vspace{-3mm}
\end{table*}

% -------------------------------------------------------------------------------
\subsection{Preemptive Lexicographic Optimization}\label{sec:PreemptiveLexicographicOptimization}

One common approach to solving lexicographic optimization problems is the \emph{preemptive} scheme, in which cost functions are optimized sequentially in decreasing order of priority. Let $\bc(\by): \mathbb{R}^{n_y \times (K+1)} \rightarrow \overline{\mathbb{R}}^N$ be a vector-valued cost function and let $c_i(\by)$ denote its $i$-th component. Given a fixed initial state $\hat{\bx}_0 \in \mathbb{R}^{n_x}$, the preemptive lexicographic optimization problem is formulated as a sequence of $N$ optimization problems~\cite[Sec.~3.3]{Marler2004}:
\begin{equation}
    \label{eq:LexMinProblem}
    \begin{aligned}
        \min_{\bu, \bx, \by} \; & c_i(\by)                                             \\
        \text{s.t. }\;          & \bx_0 = \hat{\bx}_0,                                 \\
                                & \bx_{k+1} = f(\bx_k, \bu_k), \forall k \in [0, K-1], \\
                                & \by_k = g(\bx_k, \bu_k), \forall k \in [0, K],       \\
                                & c_{j}(\by) = c_{j}(\by^{*}_{j}),                     \\
                                & j = 1,2,\ldots,i-1 \quad \text{when } i>1,           \\
                                & i = 1,2,\ldots,N.
    \end{aligned}
\end{equation}
The equality $c_{j}(\by) = c_{j}(\by^{*}_{j})$ encodes that the $i$-th optimization problem optimizes the $i$-th cost function $c_i(\by)$ subject to the optimal cost values $c_{j}(\by^{*}_{j})$ obtained in the previous $j=1,2,\ldots,i-1$ optimization problems. Thus, each subsequent problem introduces additional constraints and increases complexity, making the overall optimization computationally expensive.

% -------------------------------------------------------------------------------
\subsection{Deterministic MPPI Solver}\label{sec:PathIntegralSolver}

In this work, we modify a deterministic MPPI solver such that it is not restricted to quadratic costs of the input. The underlying deterministic MPPI solver is a variant of the MPPI approach presented in~\cite{Williams2018} tailored to solve deterministic programs~\cite[Sec. IV-B]{Halder2025ICRA}:
\begin{equation}
    \label{eq:PIOptimizationProblem}
    \begin{aligned}
        \min_{\bu, \bx, \by} \; & s(\by) + \sum_{k=0}^{K} \frac{1}{2}\,\bu_k^\top R \bu_k \\
        \text{s.t.}\;           & \bx_0 = \hat{\bx}_0,                                    \\
                                & \bx_{k+1} = f(\bx_k, \bu_k), \forall k \in [0, K-1],    \\
                                & \by_k = g(\bx_k, \bu_k), \forall k \in [0, K],
    \end{aligned}
\end{equation}
with $s: \mathbb{R}^{n_y \times (K+1)} \rightarrow \mathbb{R}$ and $\hat{\bx}_0 \in \mathbb{R}^{n_x}$. The system input is penalized quadratically with $R \in \mathbb{R}^{n_u\times n_u}$, given by ${R := \lambda \bSigma^{-1}}$, where $\bSigma\in\mathbb{R}^{n_u\times n_u}$ is a positive definite covariance matrix and $\lambda\in\mathbb{R}_{>0}$ is a tuning parameter called \emph{temperature}.

In the deterministic MPPI algorithm (cf.~\cite[Sec. IV-B]{Halder2025ICRA}), $\lambda$ and $\bSigma$ are both scaled by a factor $\beta^2$, where $\beta \in (0,1)$ is reduced over successive iterations. As $\beta\rightarrow 0$, the algorithm returns the optimal input $\bu^*$~\cite[Thm.~1]{Homburger2025}. Since $\lambda$ and $\bSigma$ are scaled identically, the scalar $\beta^2$ cancels out, meaning the quadratic input cost weighting $R$ remains constant and is independent of $\beta$ for any $\beta > 0$:
\begin{equation}\label{eq:SOTAScaling}
    (\beta^2 \lambda)\, (\beta^2 \bSigma)^{-1} = \lambda \bSigma^{-1} = R.
\end{equation}
We will later modify this behavior and scale $\lambda$ and $\bSigma$ at different rates, allowing us to eliminate the quadratic input cost term in~\eqref{eq:PIOptimizationProblem} as $\beta \rightarrow 0$.

\section{Problem Statement}\label{sec:ProblemStatement}

Our minimum-violation motion planner is embedded within a typical motion planning framework, as depicted in Fig.~\ref{fig:PlanningApproachOverview}.
\begin{figure}[!htb]
    \centering
    \begin{adjustbox}{max width=0.8\columnwidth}
        \def\svgwidth{\columnwidth} % DO NOT CHANGE
        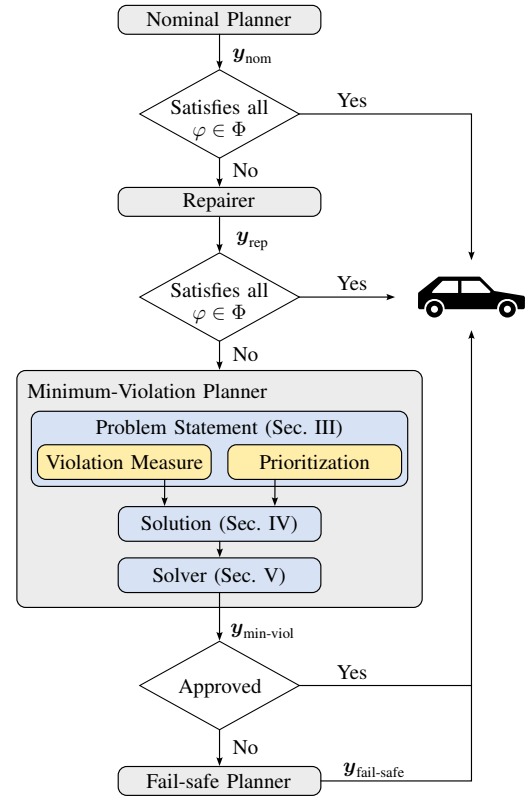
    \end{adjustbox}
    \caption{Overview of the motion planning framework. Figure derived from~\cite[Fig.~4]{Lin2025}.}
    \label{fig:PlanningApproachOverview}
    \vspace{-5mm}
\end{figure}

The planning process begins with a nominal planner that attempts to generate a trajectory $\by_{\text{nom}}$ that satisfies all specifications simultaneously (see, e.g., \cite{Lercher2025}). If this fails, a repairer may be used to generate a repaired trajectory $\by_{\text{rep}}$ (see, e.g., \cite{Lin2025, Tong2025}). However, due to a) other traffic participants violating traffic rules, b) overly restrictive specifications, or c) incorrect assumptions, a trajectory that satisfies all specifications may not always exist. In this case, minimum-violation motion planning is employed to ensure continued vehicle operation by generating a trajectory $\by_{\text{min-viol}}$ that minimizes specification violations according to a violation measure and a prioritization. Due to the intentional violation of the specifications, such driving maneuvers should be approved by a control instance (e.g., an offline safety operator or an automated verification system). Approved trajectories are executed by the vehicle, while unapproved trajectories trigger the execution of a fail-safe maneuver $\by_{\text{fail-safe}}$ that at least ensures collision avoidance (see, e.g., \cite{Pek2021}). In this paper, we focus on the minimum-violation motion planner and subsequently introduce a corresponding planning problem.

Let $\Phi := \{\varphi_1,\varphi_2,\dots,\varphi_N\}$ be a finite set of STL specifications. To quantify the violation of a specification $\varphi \in \Phi$, we define a cost function $c_{\varphi}(\by)$ based on the negative part of its robustness:
\begin{equation} \label{eq:ContinuousCost}
    c_{\varphi}(\by) := -\min\bigl(0,\eta^{\varphi}(\by,0)\bigr).
\end{equation}
In our minimum-violation motion planning context, this definition requires the robustness measure $\eta^{\varphi}$ to be at least \emph{sound}, i.e., $\eta^{\varphi}(\by,k) \geq 0 \Longrightarrow (\by,k)\models \varphi$. Soundness guarantees that we always penalize violations, i.e., $(\by,0) \not\models \varphi \Longrightarrow c_{\varphi}(\by) > 0$, but it does not exclude the possibility that satisfied trajectories are also penalized. If the robustness measure is additionally \emph{reverse sound}, i.e., $\eta^{\varphi}(\by,k) < 0 \Longrightarrow (\by,k) \not\models \varphi$, then we exclusively penalize violating trajectories, i.e., $(\by,0) \not\models \varphi \Longleftrightarrow c_{\varphi}(\by) > 0$. All robustness measures listed in Tab.~\ref{tab:RobustnessDefinitions} are at least sound, and some are additionally reverse sound.

Next, we equip $\Phi$ with a total order $\pres$, such that $(\Phi,\pres)$ forms a finite totally ordered set. Without loss of generality, we index the specifications such that
\begin{equation*}
    \varphi_1 \pres \varphi_2 \pres \cdots \pres \varphi_N.
\end{equation*}
Let ${\bc: \mathbb{R}^{n_y \times (K+1)} \to \overline{\mathbb{R}}^N}$ be the corresponding vector-valued cost function defined by
\begin{align}
    \bc(\by) & := [c_{\varphi_1}(\by), c_{\varphi_2}(\by), \dots, c_{\varphi_N}(\by)]^\top.     \label{eq:ContCostVector}
\end{align}
We compare trajectories lexicographically, which induces a total order $\precont$ on the set of output trajectories $\mathcal{Y}$ according to
\begin{align}
    \by \precont \by' & \Longleftrightarrow \bc(\by) \prelex \bc(\by'),
\end{align}
with $\by, \by' \in \mathcal{Y}$. Intuitively, this expresses that a trajectory $\by$ precedes (i.e., is at least as good as) a trajectory $\by'$ iff its cost vector is lexicographically less than or equal to that of~$\by'$. In contrast to weighted-sum formulations, this lexicographic formulation does not require manual weight tuning to ensure that higher-priority specifications dominate lower-priority ones.

\begin{definition}[Set of Optimal Trajectories~{\cite[Def.~1.23]{Ehrgott2005}}]\label{def:Optimality}
    Let $(\Phi,\pres)$ be a totally ordered set of STL specifications and let $\mathcal{Y}$ denote the set of output trajectories. The \emph{set of optimal trajectories} is defined as
    \begin{equation*}
        \mathcal{Y}_{\text{cont}}^* :=
        \Bigl\{
        \by \in \mathcal{Y}
        \,\Big|\,
        \forall \by' \in \mathcal{Y}:\;
        \by \precont \by'
        \Bigr\}.
    \end{equation*}
\end{definition}

\vspace{1mm}
\begin{definition}[Planning Problem]\label{def:PlanningProblem}
    Given a totally ordered set of STL specifications $(\Phi,\pres)$ and a set $\mathcal{Y}$ of output trajectories, the \emph{planning problem} consists of
    finding a trajectory $\by^*_{\text{cont}} \in \mathcal{Y}_{\text{cont}}^*$.
\end{definition}

However, this continuous problem formulation has two major drawbacks. First, lexicographic comparison is typically overly sensitive to small continuous deviations, e.g., for $[3.005, 100]^\top \preslex[3.006, 2]^\top$, a negligible cost deviation of $0.001$ of the higher-priority specification determines the order, and makes the cost values of lower-priority specifications irrelevant. Second, lexicographic orders on continuous domains generally cannot be represented by scalar utility functions~\cite{Debreu1983}, preventing the use of standard single-objective optimization methods. To overcome these limitations, we propose a respective discretization approach in Sec.~\ref{sec:Solution}.

\section{Solution}\label{sec:Solution}

Subsequently, we present our solution to the planning problem in Def.~\ref{def:PlanningProblem}. We first introduce a robustness measure that combines spatial and temporal violations. Next, we demonstrate how to discretize the continuous cost, analyze the effects of this discretization, and finally employ a utility-based reformulation to transform the discretized multi-objective lexicographic optimization problem into an equivalent single-objective optimization problem.

\subsection{Robustness Measure}\label{sec:RobustnessMeasure}

The standard measures of space and time robustness $\eta_{\text{space}}$, $\eta_{\text{left-time}}$, $\eta_{\text{right-time}}$, and $\eta_{\text{comb-time}}$ are not suitable for minimum-violation motion planning, because they quantify either the magnitude of a violation or its duration. This is problematic because trajectories describing different maneuvers can become indistinguishable, as shown next.

\begin{runningexample}
    Consider the scenario depicted in Fig.~\ref{fig:RunningExampleScenario}, where the ego vehicle approaches an obstacle that is broken down. We introduce STL specifications for collision avoidance, making progress, and in-lane driving:
    \begin{align*}
        \varphi_{\text{coll}} & := \mathbf{G}(\lnot \mu_{\text{collision}}),                           \\
        \varphi_{\text{prog}} & := \mathbf{F}(\mu_{\text{make-progress}}),                             \\
        \varphi_{\text{lane}} & := \mathbf{G}(\mu_{\text{left-bound}} \land \mu_{\text{right-bound}}).
    \end{align*}
    The space robustness for $\varphi_{\text{lane}}$ has the form
    \begin{equation*}
        \eta_{\text{space}}^{\varphi_{\text{lane}}}(\by,k) = \min_{k\in [0,K]}(\min(\predrob_{\text{space}}^{\mu_{\text{left-bound}}}(\by,k), \predrob_{\text{space}}^{\mu_{\text{right-bound}}}(\by,k))).
    \end{equation*}
    Since the trajectories $\by_a$ and $\by_b$ have the same maximum lateral deviation from the left bound of the lane, their space robustness is equal, i.e., $\eta_{\text{space}}^{\varphi_{\text{lane}}}(\by_a,0) = \eta_{\text{space}}^{\varphi_{\text{lane}}}(\by_b,0)$, despite describing different maneuvers.
\end{runningexample}

Thus, we introduce a custom robustness measure, denoted by $\eta_{\text{space-left-time}}$, that combines spatial and temporal violations. It utilizes the standard $\amin_{\text{std}}$ and $\amax_{\text{std}}$ operators (see Tab.~\ref{tab:RobustnessDefinitions} and Appendix~\ref{app:MinMaxDefinitions}) and a custom predicate robustness function $\predrob_{\text{space-left-time}}$ (see final row in Tab.~\ref{tab:PredicateRobustnessDefinitions}). Inspired by the space-time robustness proposed in~\cite{Donze2010}, it is intuitively defined as the maximum sum of the space robustness and the left-time robustness. Fig.~\ref{fig:FirstRunningExample}~(a) provides a visualization.
\begin{figure}[!htb]
    \centering
    \begin{subfigure}[b]{0.48\columnwidth}
        \centering
        {\footnotesize
            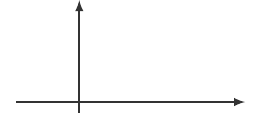
        }
        \vspace{-2mm}
        \caption{Space-left-time robustness.}
    \end{subfigure}
    \hfill
    \begin{subfigure}[b]{0.48\columnwidth}
        \centering
        {\footnotesize
            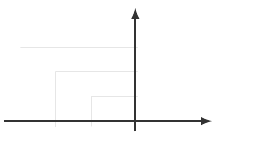
        }
        \vspace{-2mm}
        \caption{Cost functions.}
    \end{subfigure}
    \vspace{-1mm}
    \caption{Example visualizations of predicate robustness and cost functions. (a)~Space-left-time robustness for a predicate $\mu$ at discrete time $k$. (b)~Continuous and discretized cost functions for the specification $\varphi_{\text{lane}}$.}
    \label{fig:FirstRunningExample}
    \vspace{-5mm}
\end{figure}

The weight $w$ in the definition of the space-left-time robustness for predicates $\predrob_{\text{space-left-time}}^{\mu}$ allows one to tune the trade-off between spatial and temporal violations. Specifically, a small $w$ biases the planner to minimize the spatial magnitude of specification violation, while a large $w$ biases the planner to minimize the duration of specification violation. Tuning $w$ is straightforward, since only the value range of the respective predicate function $p_\mu(\by, k)$ needs to be considered. Furthermore, normalizing the temporal violation by the planning time horizon $K$ ensures that its value remains independent of the time discretization. Also note that $\eta_{\text{space-left-time}}$ maintains the properties of soundness and reverse soundness of $\eta_{\text{space}}$ and $\eta_{\text{left-time}}$.

The robustness measures of the second group in Tab.~\ref{tab:RobustnessDefinitions} implicitly combine spatial and temporal violations since they use alternative definitions of the minimum and maximum operators. This works for simple temporal specifications, e.g., $\mathbf{G}_I(\mu)$, as it provides an average violation of the predicate $\mu$ over the time interval $I$. However, for logically composed or nested specifications, e.g., $\mathbf{G}(\mu) \land \mathbf{F}(\mu')$, it yields an unintuitive violation measure, since the violation of time is coupled with the violation between specifications. As an example, using the power-mean robustness for a conjunction where one specification is strongly violated ($-100$) and the other specification is marginally satisfied ($0.1$) results in $\eta_{\text{pm}} = -50$ (using $\nu_5 = 1.0$). This severely underrepresents the actual magnitude of the violation and obscures the influence of the involved sub-specifications. In contrast, our proposed $\eta_{\text{space-left-time}}$ relies on the standard $\amin_{\text{std}}$ and $\amax_{\text{std}}$ operators and combines spatial and temporal violations exclusively at the predicate level. This enables a comprehensible, traceable quantification of violations.

% --------------------------------------------------------------
\subsection{Cost Discretization}

To overcome the drawbacks of the continuous problem formulation in Def.~\ref{def:PlanningProblem}, we discretize the continuous cost function $c_{\varphi}(\by)$. This reduces the sensitivity of lexicographic comparisons to negligible cost variations and allows the cost functions to better reflect the intended severity of specification violations, e.g., exceeding a speed limit by \SI{0.5}{\km\per\hour} versus \SI{1.0}{\km\per\hour} might be considered practically equivalent, whereas a violation of \SI{10}{\km\per\hour} is significantly more severe. Finally, the discretization enables transforming the multi-objective optimization problem into a single-objective optimization problem.

We introduce the discretized cost $\tilde c_{\varphi}(\by)$ by non-uniform quantization of the continuous cost $c_{\varphi}(\by)$. Let $m_\varphi \in \mathbb{N}$ be the number of violation intervals and let $\xi \in \{0, \dots, m_\varphi\}$ be an index. We introduce strictly increasing positive thresholds $\alpha_\xi \in \mathbb{R}_{>0}$ for $\xi \in \{1, \dots, m_\varphi-1\}$, such that ${\alpha_1 < \alpha_2 < \cdots < \alpha_{m_\varphi-1}}$. By additionally defining $\alpha_0 := 0$, the intervals $I_{\xi}$ are constructed as follows:
\begin{equation}\label{eq:DefinitionIntervals}
    I_{\xi} :=
    \begin{cases}
        [0,0],                          & \text{if } \xi = 0,               \\
        (\alpha_{\xi-1}, \alpha_{\xi}], & \text{if } 1 \le \xi < m_\varphi, \\
        (\alpha_{m_\varphi-1}, \infty), & \text{if } \xi = m_\varphi.
    \end{cases}
\end{equation}
The unique assignment defines the discretized cost:
\begin{equation}\label{eq:uniqueAssignment}
    \tilde c_{\varphi}(\by) = \xi \quad \Longleftrightarrow \quad c_{\varphi}(\by) \in I_{\xi}.
\end{equation}
It is important to note that $\tilde c_{\varphi}(\by)$ is bounded, i.e., $\tilde c_{\varphi}(\by) \in \{0,1,\dots, m_{\varphi} \}$.

\begin{runningexample}
    For the specifications $\varphi_{\text{coll}}, \varphi_{\text{prog}}$, and $\varphi_{\text{lane}}$, we define the discretized cost functions using $m_{\varphi_{\text{coll}}} =1,  m_{\varphi_{\text{prog}}} = 6$, and $m_{\varphi_{\text{lane}}} = 3$. Fig.~\ref{fig:FirstRunningExample}~(b) shows $c_{\varphi_{\text{lane}}}(\by)$ and $\tilde c_{\varphi_{\text{lane}}}(\by)$ by example.
\end{runningexample}

Equivalent to the vector-valued continuous cost, we define a vector-valued discrete cost function $\tilde \bc: \mathbb{R}^{n_y \times (K+1)} \to \mathbb{N}_0^N$:
\begin{align}
    \tilde \bc(\by) := [\tilde c_{\varphi_1}(\by), \tilde c_{\varphi_2}(\by), \dots, \tilde c_{\varphi_N}(\by)]^\top, \label{eq:DiscreteCostVector}
\end{align}
based on which the total order $\predisc$ is induced on the set of output trajectories $\mathcal{Y}$ according to
\begin{equation}\label{eq:InducedPartialOrderOriginal}
    \by \predisc \by' \Longleftrightarrow \tilde \bc(\by) \prelex \tilde \bc(\by'),
\end{equation}
with $\by, \by' \in \mathcal{Y}$. Using $\predisc$, the set of optimal trajectories $\mathcal{Y}_{\text{disc}}^*$ and the optimal trajectory $\by^*_{\text{disc}}$ of the respective planning problem are defined analogously to Def.~\ref{def:Optimality} and Def.~\ref{def:PlanningProblem}.

\begin{runningexample}
    Based on Fig.~\ref{fig:RunningExampleScenario}, the order of the specifications is $\varphi_{\text{coll}} \pres \varphi_{\text{prog}} \pres \varphi_{\text{lane}}$, i.e., we prioritize collision avoidance over making progress and making progress over staying in lane. Example discrete cost for the trajectories from Fig.~\ref{fig:RunningExampleScenario} are as follows:
    \begin{equation*}
        \overset{\tilde\bc(\by_b)=}{\begin{bmatrix}0\\1\\1\end{bmatrix}} \prelex
        \overset{\tilde\bc(\by_a)=}{\begin{bmatrix}0\\1\\2\end{bmatrix}} \prelex
        \overset{\tilde\bc(\by_c)=}{\begin{bmatrix}0\\1\\3\end{bmatrix}} \prelex
        \overset{\tilde\bc(\by_e)=}{\begin{bmatrix}0\\5\\0\end{bmatrix}} \prelex
        \overset{\tilde\bc(\by_d)=}{\begin{bmatrix}1\\4\\0\end{bmatrix}}.
    \end{equation*}
    The induced order on the trajectories is $\by_b \predisc \by_a \predisc \by_c \predisc \by_e \predisc \by_d$ and thus, the best trajectory is ${\by^*_{\text{disc}} = \by_b}$.
\end{runningexample}

\subsection{Effects of Discretization}\label{sec:EffectsofDiscretization}

The relation between the set of optimal trajectories of the continuous and the discretized problem, $\mathcal{Y}_{\text{cont}}^*$ and $\mathcal{Y}_{\text{disc}}^*$, is governed by the orders $\precont$ and $\predisc$ on $\mathcal{Y}$. If two trajectories have identical continuous cost vectors, then they also have identical discrete cost vectors, since each component is mapped deterministically to a discrete interval index (cf.~\eqref{eq:uniqueAssignment}). Formally, it holds $\by \sim_{\text{cont}} \by' \Longrightarrow \by \sim_{\text{disc}} \by'$, with $\by, \by' \in \mathcal{Y}$. However, discretization may map distinct continuous costs to the same discrete value. Thus, the critical case is strict continuous precedence, i.e., $\by \prescont \by'$, which yields three possible outcomes for the resulting discrete order:

\begin{enumerate}
    \item \emph{Order preservation:} If $\by \prescont \by' \Longrightarrow \by \presdisc \by'$ for all $\by, \by' \in \mathcal{Y}$, then the sets of optimal trajectories are identical, i.e., $\mathcal{Y}_{\text{cont}}^* = \mathcal{Y}_{\text{disc}}^*$;
    \item \emph{Order relaxation:} If $\by \prescont \by' \Longrightarrow \by \presdisc \by' \lor \by \sim_{\text{disc}} \by'$ for all $\by, \by' \in \mathcal{Y}$, then the set of optimal trajectories of the discretized problem still contains the continuous optima, i.e., $\mathcal{Y}_{\text{cont}}^* \subseteq \mathcal{Y}_{\text{disc}}^*$;
    \item \emph{Order inversion:} If there exist $\by, \by' \in \mathcal{Y}$ such that ${\by \prescont \by'} \Longrightarrow \by' \presdisc \by$, then the continuous optima may not be contained in the discretized optima, i.e., $\mathcal{Y}_{\text{cont}}^* \nsubseteq \mathcal{Y}_{\text{disc}}^*$ may occur.
\end{enumerate}

To understand when these outcomes occur, let $i^*_{\text{cont}} \in \{1,\dots, N\}$ be the smallest index at which the continuous cost vectors differ, i.e., $c_{\varphi_i}(\by) = c_{\varphi_i}(\by')$ for all $i < i^*_{\text{cont}}$ and $c_{\varphi_{i^*_{\text{cont}}}}(\by) < c_{\varphi_{i^*_{\text{cont}}}}(\by')$ (cf. Def.~\ref{def:LexicographicallyOrderedSet}). If these costs lie in different violation intervals, i.e., $c_{\varphi_{i^*_{\text{cont}}}}(\by) \in I_{\xi}$ and $c_{\varphi_{i^*_{\text{cont}}}}(\by') \in I_{\xi'}$ with $\xi < \xi'$, then $\by \presdisc \by'$ is preserved. If, in contrast, both costs lie in the same interval, i.e., $c_{\varphi_{i^*_{\text{cont}}}}(\by), c_{\varphi_{i^*_{\text{cont}}}}(\by') \in I_{\xi}$ for some $\xi$, then the distinction at index $i^*_{\text{cont}}$ is lost. The resulting order between $\by$ and $\by'$ is then determined solely by lower-priority specifications $\varphi_{i^*_{\text{cont}}+1}, \dots, \varphi_N$ and may therefore be preserved ($\by \presdisc \by'$), relaxed to a tie ($\by \sim_{\text{disc}} \by'$), or even inverted ($\by' \presdisc \by$). However, the effect of lower-priority specifications cannot be assessed \textit{a priori} in our highly non-smooth and non-convex problem setup. Nevertheless, our numerical evaluations in Sec.~\ref{sec:NumericalEffectsOfDiscretization} demonstrate that if the intention is to achieve a close approximation of the continuous lexicographic optimum in practical applications, the impact of order relaxation and order inversion can be effectively mitigated through an appropriate discretization granularity and distribution of violation intervals.

\begin{runningexample}
    Consider the trajectories $\by_a$ and $\by_c$ in Fig.~\ref{fig:RunningExampleScenario} with the continuous and discretized cost shown below, and let a violation interval for $\varphi_{\text{prog}}$ be $I_1 = [5.0, 7.0)$:
    \begin{equation*}
        \overset{\bc(\by_c)=}{\begin{bmatrix}0.0\\\mathbf{5.8}\\6.3\end{bmatrix}} \prelex
        \overset{\bc(\by_a)=}{\begin{bmatrix}0.0\\\mathbf{6.5}\\3.2\end{bmatrix}},\hspace{8mm}
        \overset{\tilde\bc(\by_c)=}{\begin{bmatrix}0\\1\\\mathbf{3}\end{bmatrix}} \succlex
        \overset{\tilde\bc(\by_a)=}{\begin{bmatrix}0\\1\\\mathbf{2}\end{bmatrix}}.
    \end{equation*}
    Since $c_{\varphi_{\text{prog}}}(\by_a), c_{\varphi_{\text{prog}}}(\by_c) \in I_1$, the discretization removes the distinction for $\varphi_{\text{prog}}$. The order in the discrete cost case is therefore determined by the lower-priority specification $\varphi_{\text{lane}}$, and consequently, $\by_c \precont \by_a$, whereas $\by_a \predisc \by_c$.
\end{runningexample}

% --------------------------------------------------------------
\subsection{Utility-based Reformulation}\label{sec:UtilityBasedReformulation}

In our problem setup, it is computationally impractical to use the preemptive scheme for lexicographic optimization problems, which requires solving a sequence of optimization problems, because our cost functions are highly non-smooth, non-convex, and expensive to evaluate. Therefore, we seek to optimize a single scalar utility function $s(\by)$ instead. While lexicographic orders on continuous spaces (e.g., $\mathbb{R}^N$) generally do not admit a real-valued utility representation~\cite{Debreu1983}, the strict finiteness of both the set of specifications $\Phi$ and the values of each discrete cost component $\{0,1,\dots,m_{\varphi}\}$ allows us to construct a scalar utility representation.

Inspired by mixed-radix representations~\cite{Kohli2007}, we construct the utility function $s(\by)$ by assigning disjoint bit ranges to the individual cost components according to their priorities. This preserves the lexicographic order on the trajectories and allows each specification to be discretized independently, without having to account for its effects on other specifications. Formally, this is expressed as
\begin{equation}\label{eq:ScalarizationFunction}
    \begin{aligned}
        s(\by) & := \sum_{i=1}^{N} \tilde c_{\varphi_i}(\by)\, 2^{B_i}, \text{ with} \\
        B_i    & := \begin{cases}
                        \sum_{j=i+1}^{N} b_j & \text{if } i < N, \\
                        0                    & \text{if } i = N,
                    \end{cases}                         \\
        b_i    & := \left\lceil \log_2(m_{\varphi_i}+1) \right\rceil.
    \end{aligned}
\end{equation}
The term $b_i$ denotes the word width needed to represent the specification $\varphi_i$, i.e., the number of Bits required to represent the maximum discrete cost value that a specification can take. The factor $2^{B_i}$ acts as a bit shift that strictly separates priority levels. Consequently, conversions between $\tilde \bc(\by)$ and $s(\by)$ reduce to simple bitwise operations, making the conversion fast and computationally efficient. It is worth noting that the exponential structure of~(\ref{eq:ScalarizationFunction}) may lead to large integer values. However, this is no limitation for practical applications, since arbitrary-precision integer computation libraries are widely available.

\begin{runningexample}
    Using the word widths $b_1 = 1$, $b_2 = 3$, and $b_3 = 2$, the discrete cost $\tilde \bc(\by_d) = [\,\BitBoxA{1},\, \BitBoxB{4},\, \BitBoxC{0}\,]^{\top}$ yields the following scalar cost and equivalent binary representation:
    \begin{equation*}
        s(\by_d) = \BitBoxA{1}\cdot 2^{3+2} + \BitBoxB{4}\cdot 2^{2} + \BitBoxC{0} = 48 \;\equiv\;
        \BitBoxA{1}\,
        \overset{\mathclap{b_2+b_3}}{%
            \overleftarrow{%
                \BitBoxB{100}\,
                \underset{\mathclap{b_3}}{%
                    \underleftarrow{%
                        \BitBoxC{00}%
                    }%
                }%
            }%
        }\, .
    \end{equation*}
\end{runningexample}

\begin{lemma}[Order Representation]\label{lem:OrderRepresentation}
    For any output trajectories $\by,\by'\in\mathcal{Y}$, it holds
    \begin{equation*}
        \tilde\bc(\by)\prelex \tilde\bc(\by') \Longleftrightarrow s(\by)\le s(\by').
    \end{equation*}
\end{lemma}

\noindent A formal proof of Lem.~\ref{lem:OrderRepresentation} is provided in Appendix~\ref{app:ProofOrderRepresentation}. Analogously to (\ref{eq:InducedPartialOrderOriginal}), we can infer
\begin{equation*}\label{eq:InducedPartialOrderAdapted}
    \by \predisc \by' \Longleftrightarrow s(\by) \leq s(\by'),
\end{equation*}
with $\by, \by' \in \mathcal{Y}$. Using this utility representation, we can formulate a single-objective problem that is equivalent to the discretized lexicographic problem:
\begin{equation}
    \label{eq:LexMinProblemScalar}
    \begin{aligned}
        \min_{\bx, \bu, \by} \; & s(\by)                                               \\
        \text{s.t. }\;          & \bx_0 = \hat{\bx}_0,                                 \\
                                & \bx_{k+1} = f(\bx_k, \bu_k), \forall k \in [0, K-1], \\
                                & \by_k = g(\bx_k, \bu_k), \forall k \in [0, K].       \\
    \end{aligned}
\end{equation}

\begin{runningexample}
    The scalar cost values are as follows:
    \begin{equation*}
        \underbrace{5}_{=\, s(\by_b)} \leq
        \underbrace{6}_{=\, s(\by_a)} \leq
        \underbrace{7}_{=\, s(\by_c)} \leq
        \underbrace{20}_{=\, s(\by_e)} \leq
        \underbrace{48}_{=\, s(\by_d)}.
    \end{equation*}
    The induced order on the trajectories is again $\by_b \predisc \by_a \predisc \by_c \predisc \by_e \predisc \by_d$ and thus, the best trajectory is $\by^*_{\text{disc}} = \by_b$.
\end{runningexample}

The scalar cost function $s(\by)$ is, by design, non-smooth and non-convex. We thus require an efficient solver capable of handling such challenging cost functions.

\section{Solver}\label{sec:Solver}

We now present our approach to solve \eqref{eq:LexMinProblemScalar}, which is built upon our deterministic MPPI solver introduced in~\cite{Halder2025ICRA}. However, this solver requires a quadratic input cost (cf.~\eqref{eq:PIOptimizationProblem}) and is therefore not directly applicable to our problem. Subsequently, we briefly introduce our algorithm, demonstrate how to eliminate the quadratic input cost, and outline several adaptations that improve the efficiency of the solver in our specific use case.

\subsection{Deterministic MPPI Solver}

Given an initial system state $\hat{\bx}_0$, an initial input trajectory $\hat \bu$, the number of iterations $J$, the initial covariance $\bSigma$, and the initial temperature $\lambda$, Alg.~\ref{alg:PISolver} computes a solution $\by^*$ to \eqref{eq:LexMinProblemScalar}. In each iteration, $M$ input samples are drawn from a truncated normal distribution (see lines~\ref{alg:SamplingStart} to~\ref{alg:SamplingEnd}) and applied to system~\eqref{eq:SystemDefinition} (see lines~\ref{alg:SystemRolloutStart} to~\ref{alg:SystemRolloutEnd}). Subsequently, the cost function is evaluated (see line~\ref{alg:CostEvaluationStart}), and the MPPI correction term is computed according to~\cite[Sec.~3]{Williams2018} (see lines~\ref{alg:CorrectionTermStart} to~\ref{alg:CorrectionTermEnd}). A weight is determined for each sample (see lines~\ref{alg:PIWeightingStart} to~\ref{alg:PIWeightingEnd}), and the input trajectory is updated accordingly (see lines~\ref{alg:InputUpdateStart} and~\ref{alg:InputUpdateEnd}), which serves as the base for the next iteration. By progressively shrinking the initial covariance $\bSigma$ and the initial temperature $\lambda$ across the iterations (see line~\ref{alg:Shrinking}), the algorithm converges to the optimum (cf.~\cite{Halder2025ICRA, Homburger2025}). If the solver operates in a model predictive control (MPC) scheme, the previous solution can be used as the initial input trajectory $\hat \bu$ for warm-starting. Note that MPPI generally exhibits a low sensitivity to a weak initialization of $\hat{\bu}$, since the sampling process implicitly smooths the cost function, reducing the risk of converging to local minima~\cite{Jordana2025}. Next, we describe the parts of Alg.~\ref{alg:PISolver} that distinguish our approach from~\cite{Halder2025ICRA}.

\begin{algorithm}[htb]
    \caption{Deterministic MPPI Solver}\label{alg:PISolver}
    {\small
        \begin{algorithmic}[1]
            \setlength{\itemsep}{1.5pt}  % Increases spacing between items
            \Require Initial system state $\hat{\bx}_0$, initial input trajectory $\hat{\bu}$, number of iterations $J$, initial covariance $\bSigma$, initial temperature $\lambda$
            \Ensure  Optimal output trajectory $\by^*$
            \State $\by_{\text{best}} \gets \emptyset,\ \hat{s}_{\min} \gets \infty$
            \For{$j \in \{1,\dots,J\}$} \Comment{Iterations} \label{alg:Iterations}
            \State $\beta \gets \mathfrak{B}_{\text{cos}}(j)$
            \State $\lambda_j \gets \beta^2 \lambda$, $\bSigma_j \gets \beta \bSigma$ \label{alg:Shrinking}
            \State $M \gets \mathfrak{M}_{\text{cos}}(j)$
            \Statex
            \For{$m \in \{1,\dots,M\}$} \Comment{Samples}
            \For{$k \in \{0,\dots,K\}$} \Comment{Generate inputs}  \label{alg:SamplingStart}
            \State $\beps_k^m \sim \mathcal{N}(0,\bSigma_j)$
            \State $\hat{\bu}_k^m \gets \min(\overline{\bu},\max(\underline{\bu},\hat{\bu}_k+\beps_k^m))$ \Comment{(Comp.-wise)} \label{alg:clipping}
            \State $\beps_k^m \gets \hat{\bu}_k^m - \hat{\bu}_k$
            \EndFor \label{alg:SamplingEnd}
            \Statex
            \State $\bx_0^m = \hat{\bx}_0$ \Comment{Apply to system \eqref{eq:SystemDefinition}} \label{alg:SystemRolloutStart}
            \For{$k \in \{0,\dots,K-1\}$}
            \State $\bx_{k+1}^m \gets f(\bx_k^m,\hat\bu_k^m)$
            \State $\by_k^m \gets g(\bx_k^m,\hat\bu_k^m)$
            \EndFor
            \State $\by^m_K \gets g(\bx^m_K,\hat\bu^m_K)$ \label{alg:SystemRolloutEnd}
            \Statex
            \State $\hat{s}^m \gets s(\by^m)$ \Comment{Evaluate cost function (\ref{eq:ScalarizationFunction})} \label{alg:CostEvaluationStart}
            \If{$\hat{s}^m < \hat{s}_{\min}$} \Comment{Track best sample} \label{alg:TrackingStart}
            \State $\hat{s}_{\min} \gets \hat{s}^m$
            \State $\by_{\text{best}} \gets \by^m$  \label{alg:TrackingEnd}
            \EndIf \label{alg:CostEvaluationEnd}
            \Statex
            \State $\ell^m \gets 0$ \Comment{Compute correction term} \label{alg:CorrectionTermStart}
            \For{$k \in \{0,\dots,K \}$}
            \State $\ell^m \gets \ell^m + \lambda_j {\beps_k^m}^\top \bSigma_j^{-1} \hat{\bu}_k$
            \EndFor
            \State $\ell^m \gets \ell^m + \hat{s}^m$
            \EndFor\label{alg:CorrectionTermEnd}
            \Statex
            \State $\ell_{\min} \gets \min_m \ell^m$ \Comment{Compute sample weights} \label{alg:PIWeightingStart}
            \State $\vartheta \gets \sum_{m=1}^M \exp\!\left(-\frac{1}{\lambda_j}(\ell^m-\ell_{\min})\right)$
            \For{$m \in \{1,\dots,M\}$}
            \State $\omega^m \gets \frac{1}{\vartheta}
                \exp\!\left(-\frac{1}{\lambda_j}(\ell^m-\ell_{\min})\right)$
            \EndFor \label{alg:PIWeightingEnd}
            \Statex
            \For{$k \in \{0,\dots,K \}$}  \Comment{Update input trajectory} \label{alg:InputUpdateStart}
            \State $\hat{\bu}_k \gets \hat{\bu}_k + \sum_{m=1}^M \omega^m \beps_k^m$
            \EndFor \label{alg:InputUpdateEnd}
            \EndFor
            \Statex
            \State $\by^* \gets \by_{\text{best}}$\label{alg:ReturnBestSample} \Comment{Return overall best output trajectory}
            \State \Return $\by^*$
        \end{algorithmic}}
\end{algorithm}

\subsection{Elimination of Quadratic Input Cost}

The deterministic MPPI algorithm~\cite{Halder2025ICRA} incorporates a quadratic input cost term, whose weighting is solely determined by the initial temperature $\lambda$ and the initial covariance matrix $\bSigma$ (cf.~\eqref{eq:SOTAScaling}). These input costs manifest in Alg.~\ref{alg:PISolver} as a correction term (see lines~\ref{alg:CorrectionTermStart} to~\ref{alg:CorrectionTermEnd}). However, our goal in \eqref{eq:LexMinProblemScalar} is to minimize $s(\by)$, meaning the correction term must not influence the solution. To achieve this, in contrast to~\cite{Halder2025ICRA}, we scale $\lambda$ and $\bSigma$ differently (see line~\ref{alg:Shrinking}), resulting in a $\beta$-dependent weighting matrix
\begin{equation}\label{eq:NewRUpdate}
    R(\beta) = \beta^2 \lambda\, (\beta \bSigma)^{-1} = \beta \lambda\, \bSigma^{-1}.
\end{equation}
Consequently, as $\beta \to 0$ over the $J$ iterations, $R(\beta) \to 0$, and the influence of the correction term on the total cost vanishes (cf.~\cite[pp.~84--86]{Homburger2026}). Notably, the convergence of Alg.~\ref{alg:PISolver} to the global optimum of $s(\by)$ for $\lim_{\beta \to 0} R(\beta) = 0$ can be shown using the Laplace method (cf.~\cite[Thm.~1]{Homburger2025}). Our numerical experiments in Sec.~\ref{sec:SolverEvaluations} demonstrate that this approach effectively eliminates the influence of the quadratic input cost.

Since any real implementation is limited to a finite number of trajectory samples, global optimality for optimization problem~\eqref{eq:LexMinProblemScalar} cannot be guaranteed. Thus, the decay behavior of $\beta$ plays a crucial role: reducing the initial covariance $\bSigma$ too quickly causes the solver to get trapped in local optima, whereas maintaining a large covariance introduces suboptimality~\cite{Homburger2025}. Let $\mathfrak{B}: \mathbb{N} \rightarrow \mathbb{R}_{\geq 0}$ be a decay rule returning the value of $\beta$ for iteration $j \in \{1,\dots,J\}$. The exponential decay rule $\mathfrak{B}_{\text{exp}}(j):= \sqrt{\gamma^{\,j-1}}$, with $\gamma \in (0,1)$, which is used in~\cite{Halder2025ICRA, Homburger2025}, causes a rapid reduction of the covariance in early iterations. Consequently, the solver spends many iterations with little exploration, which is highly unfavorable in our nonlinear and non-smooth problem setup. Instead, we propose a cosine-based decay rule
\begin{equation*}
    \mathfrak{B}_{\text{cos}}(j) := 1 - (1 - \beta_{\min}) \frac{1}{2} \left( 1 - \cos \left(\pi \tfrac{j-1}{J-1} \right) \right),
\end{equation*}
with $\beta_{\min} \in (0,1)$. In contrast to the exponential rule, the cosine decay rule maintains a high covariance (and thus an effective exploration) for more iterations, while gradually approaching $\beta_{\text{min}}$. This gradual decay helps to reliably reduce the aforementioned suboptimality (cf.~\cite{Homburger2025Gauss}), ensuring reliable convergence in our experiments.

\begin{runningexample}
    Let $J = 20$, $\gamma = 0.5$ and $\beta_{\text{min}} = 10^{-6}$. Fig.~\ref{fig:solver_example}~(a) shows the exponential and the cosine decay rule, respectively. Fig.~\ref{fig:solver_example}~(b) illustrates the overtaking scenario, including the sampled trajectories at three selected iterations, and the optimal trajectory. The cosine decay rule is used.
\end{runningexample}

\begin{figure}[!htb]
    \centering
    \begin{subfigure}[b]{0.49\columnwidth}
        \centering
        {\footnotesize
            \begin{tikzpicture}
    \begin{axis}[
        width=4.7cm,
        height=4.7cm,
        xlabel={Iteration $j$},
        ylabel={$\beta$},
        ylabel style={at={(axis description cs:-0.12,0.6)}, anchor=south, rotate=-90},
        xmin=1, xmax=20,
        ymin=0, ymax=1.1,
        xtick={1,5,10,15,20},
        ytick={0,0.5,1},
        grid=major,
        legend style={
            at={(0.98,0.98)},
            anchor=north east,
            font=\scriptsize,
            legend cell align=left,
            fill=white
        },
        label style={font=\footnotesize},
        tick label style={font=\footnotesize}
    ]
    \def\J{20}
    \def\betamin{1e-6}
    \def\gamma{0.5}

    % Exponential
    \addplot+[color=green!60!black, line width=1pt, mark=None, domain=1:20, samples=100] {(\gamma)^((x-1)/2)};
    \addlegendentry{$\mathfrak{B}_{\text{exp}}(j)$}

    % Linear
    %\addplot+[color=red!75!white, line width=1pt, mark=None, domain=1:20, samples=2] {1 - (1 - \betamin) * (x-1)/(\J-1)};
    %\addlegendentry{Linear}

    % Cosine
    \addplot+[color=cyan!80!blue, line width=1pt, mark=None, domain=1:20, samples=100] {1 - (1 - \betamin) * 0.5 * (1 - cos(deg(pi * (x-1)/(\J-1))))};
    \addlegendentry{$\mathfrak{B}_{\text{cos}}(j)$}

    \end{axis}
\end{tikzpicture}}
        \caption{Decay rules.}
    \end{subfigure}
    \hfill
    \begin{subfigure}[b]{0.49\columnwidth}
        \centering
        \def\svgwidth{\linewidth}
        {\footnotesize
            %% Creator: Inkscape 1.4.2 (1:1.4.2+202505120737+ebf0e940d0), www.inkscape.org
%% Accompanies image file 'pi_result_running_example.pdf' (pdf, eps, ps)
\begingroup%
  \makeatletter%
  \providecommand\color[2][]{%
    \errmessage{(Inkscape) Color is used for the text in Inkscape, but the package 'color.sty' is not loaded}%
    \renewcommand\color[2][]{}%
  }%
  \providecommand\transparent[1]{%
    \errmessage{(Inkscape) Transparency is used (non-zero) for the text in Inkscape, but the package 'transparent.sty' is not loaded}%
    \renewcommand\transparent[1]{}%
  }%
  \providecommand\rotatebox[2]{#2}%
  \newcommand*\fsize{\dimexpr\f@size pt\relax}%
  \newcommand*\lineheight[1]{\fontsize{\fsize}{#1\fsize}\selectfont}%
  \ifx\svgwidth\undefined%
    \setlength{\unitlength}{125.57480099bp}%
    \ifx\svgscale\undefined%
      \relax%
    \else%
      \setlength{\unitlength}{\unitlength * \real{\svgscale}}%
    \fi%
  \else%
    \setlength{\unitlength}{\svgwidth}%
  \fi%
  \global\let\svgwidth\undefined%
  \global\let\svgscale\undefined%
  \makeatother%
  \begin{picture}(1,1)%
    \lineheight{1}%
    \setlength\tabcolsep{0pt}%
    \put(0,0){\includegraphics[width=\unitlength,page=1]{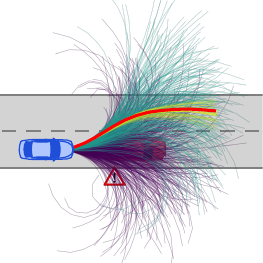}}%
    \put(0.17617073,0.25010294){\color[rgb]{0,0,0}\makebox(0,0)[t]{\lineheight{1.25}\smash{\begin{tabular}[t]{c}Ego\\vehicle\end{tabular}}}}%
    \put(0,0){\includegraphics[width=\unitlength,page=2]{pi_result_running_example.pdf}}%
    \put(0.20522121,0.68907272){\color[rgb]{0,0,0}\makebox(0,0)[t]{\lineheight{1.25}\smash{\begin{tabular}[t]{c}$\by^{*}$\end{tabular}}}}%
    \put(0.2586638,0.90958527){\color[rgb]{0,0,0}\makebox(0,0)[t]{\lineheight{1.25}\smash{\begin{tabular}[t]{c}$j = 0$\end{tabular}}}}%
    \put(0.27127159,0.7625769){\color[rgb]{0,0,0}\makebox(0,0)[t]{\lineheight{1.25}\smash{\begin{tabular}[t]{c}$j = 18$\end{tabular}}}}%
    \put(0.27308451,0.83608108){\color[rgb]{0,0,0}\makebox(0,0)[t]{\lineheight{1.25}\smash{\begin{tabular}[t]{c}$j = 10$\end{tabular}}}}%
    \put(0,0){\includegraphics[width=\unitlength,page=3]{pi_result_running_example.pdf}}%
  \end{picture}%
\endgroup%
}
        \caption{Overtaking scenario.}
    \end{subfigure}
    \caption{Visualization of decay rules and example solution of the proposed algorithm. (a) Visualization of the exponential and cosine decay rule. (b) Overtaking scenario from the running example with sampled trajectories at three selected iterations and the optimal trajectory.}
    \label{fig:solver_example}
    \vspace{-4mm}
\end{figure}

\subsection{Additional Adaptations}

The deterministic MPPI solver draws samples from a normal distribution $\mathcal{N}(0, \bSigma)$. However, in our autonomous driving problem setup, the system inputs must satisfy strict bounds. We thus clip the sampled inputs component-wise to $[\underline{\bu}, \overline{\bu}]$ (see line~\ref{alg:clipping}), similar to~\cite{Williams2018}. While input clipping can break the theoretical guarantees of standard stochastic MPPI, it does not compromise convergence in deterministic setups (cf.~\cite{Homburger2023}).

The runtime of the solver strongly depends on the total number of evaluated samples. Let $\mathfrak{M}: \mathbb{N} \rightarrow \mathbb{N}$ provide the number $M$ of samples for iteration $j$. In our previous work~\cite{Halder2025ICRA}, the number of samples per iteration is constant, i.e., $\mathfrak{M}_{\text{const}}(j):= M_{\text{init}}$, with $M_{\text{init}} \in \mathbb{N}$. To reduce the runtime without compromising solution quality, we couple the number of samples per iteration with the covariance decay. The idea is to use more samples in early iterations when the covariance is large and fewer in later iterations when the covariance is smaller. Analogous to $\mathfrak{B}_{\text{cos}}(j)$, we define a cosine-based decay rule for the number of samples per iteration:
\begin{equation*}
    \mathfrak{M}_{\text{cos}}(j) := \left\lceil M_{\text{init}} - (M_{\text{init}} - M_{\text{final}}) \frac{1}{2} \left(1 - \cos\left(\pi \tfrac{j-1}{J-1} \right)\right) \right\rceil,
\end{equation*}
where $M_{\text{final}} \in \mathbb{N}$ is the final number of samples and $M_{\text{init}} \geq M_{\text{final}}$.

Finally, instead of returning the trajectory $\by_{\text{mppi}}$ which is generated by the input trajectory $\hat \bu$ obtained at the final iteration $J$, we track the sample with the lowest cost across all iterations (see lines \ref{alg:TrackingStart} to \ref{alg:TrackingEnd}) and return its corresponding trajectory $\by_{\text{best}}$ (see line \ref{alg:ReturnBestSample}). This ensures that the executed maneuver corresponds to the best-known solution encountered during the optimization process.
\section{Numerical Experiments}\label{sec:NumericalExperiments}

We evaluate our minimum-violation motion planning approach in several experiments. Specifically, we aim to validate the following hypotheses:
\begin{enumerate}[label=\textit{H\arabic*}]
    \item \label{hyp:Discretization} A finer discretization granularity mitigates the effects of order relaxation and order inversion, and a close approximation of the continuous lexicographic optimum can be achieved in practice with a limited number of violation intervals;
    \item \label{hyp:Mppi} Our proposed solver adaptations improve solution quality while reducing the number of required samples;
    \item \label{hyp:CRScenario} Our minimum-violation motion planning approach effectively handles conflicting specifications in complex, real-world autonomous driving scenarios;
    \item \label{hyp:AlternativeSolvers} Our solution approach achieves substantially lower solve times than the preemptive scheme, and our deterministic MPPI solver is competitive with state-of-the-art solvers;
    \item \label{hyp:RobustnessMeasures} Our proposed space-left-time robustness measure resolves the trajectory indistinguishability of standard space and time robustness measures while remaining computationally competitive.
\end{enumerate}
The experiments are implemented in Python and C++ and are executed on an \texttt{AMD Ryzen Threadripper PRO 5975WX} CPU; the code can be accessed at \href{https://github.com/TUMcps/STL-Lex-MPPI-Planner}{\mbox{https://github.com/TUMcps/STL-Lex-MPPI-Planner}}.

% ##################################################################################
%  Effects of Discretization
% ##################################################################################
\subsection{Effects of Discretization}\label{sec:NumericalEffectsOfDiscretization}

Based on a linear problem setup, we numerically investigate how discretization granularity and the distribution of violation intervals affect the approximation of the continuous lexicographic optimum by the discretized cost function.

\subsubsection{Setup}

We consider a linear scalar discrete-time integrator with the dynamics
\begin{equation*}
    \begin{aligned}
        x_{k+1} & = x_k + u_k, \\
        y_k     & = x_k.
    \end{aligned}
\end{equation*}
The inputs are bounded by $\underline{u} = -1.35$ and $\overline{u} = 1.35$, the time horizon is $K=8$, and the initial state is $\hat x_0 = 0$. We introduce eight STL specifications $\varphi_k$, with $k \in \{1,\dots, 8\}$, which are designed such that they only target specific discrete times:
\begin{equation*}
    \varphi_k := \begin{cases}
        y_k < r_k    & \text{if } k \text{ is odd},  \\
        y_k \geq r_k & \text{if } k \text{ is even},
    \end{cases}
\end{equation*}
where $r_k \in \mathbb{R}$ are user-defined thresholds. As an example, the highest-priority specification $\varphi_1$ requires the output at $k=1$ to be below $r_1$, while the second-highest priority specification $\varphi_2$ requires the output at $k=2$ to be above or equal to $r_2$. This setup allows us to clearly observe the degradation in the optimization result due to discretization, as approximation errors from higher-priority specifications are directly propagated to lower-priority ones.

Given a user-defined upper bound cost value $\overline{c} = 10$, we divide the range $[0, \overline{c}]$ into equally sized intervals. Following \eqref{eq:DefinitionIntervals}, the positive thresholds $\alpha_{\xi, k}$ for a specification $\varphi_k$ and $m_{\varphi_k} > 1$ are generated by
\begin{equation}\label{eq:IntervalBoundariesGenerator}
    \alpha_{\xi, k} := \frac{\overline{c}}{m_{\varphi_k}-1}\,\xi,
\end{equation}
for  $\xi \in \{1, \dots, m_{\varphi_k}-1\}$. Further, let the total number of used violation intervals be denoted as $m_{\text{total}}:= \sum_{k=1}^{8} m_{\varphi_k}$.

The linear problem setup (i.e., linear system dynamics and predicates) allows us to exactly compute the sets of lexicographically optimal solutions $\mathcal{Y}_{\text{cont}}^*$ and $\mathcal{Y}_{\text{disc}}^*$ using standard forward and backward reachability analysis (cf.~\cite{Wetzlinger2026}). To evaluate how much more violation is generated by $\mathcal{Y}_{\text{disc}}^*$ compared to the violation that is lexicographically necessary according to $\mathcal{Y}_{\text{cont}}^*$, we compute the worst-case difference of the continuous cost between both solutions and average this error over all $N$ specifications:
\begin{align*}
    \varepsilon_{\text{viol}} & := \frac{1}{N}\sum_{i=1}^{N} \max(0, \Delta c_{\varphi_i}),\, \text{ with}                                                        \\
    \Delta c_{\varphi_i}      & := \max_{\by \in \mathcal{Y}^*_{\text{disc}}} c_{\varphi_i}(\by) - \min_{\by \in \mathcal{Y}^*_{\text{cont}}} c_{\varphi_i}(\by).
\end{align*}

\subsubsection{Example Scenario}

Fig.~\ref{fig:ExampleScenario}~(a) presents an example scenario for $m_{\varphi_1} = \dots = m_{\varphi_8} = 5$. The deviation of $\mathcal{Y}_{\text{disc}}^*$ from $\mathcal{Y}_{\text{cont}}^*$ due to order relaxation and order inversion (cf. Sec.~\ref{sec:EffectsofDiscretization}) is evident. For instance, the solution of the discretized problem violates $\varphi_2$ more than the solution of the continuous problem, which subsequently enables satisfying $\varphi_3$, which the continuous solution cannot achieve. The violation error for this scenario is $\varepsilon_{\text{viol}} = 0.9$.

\begin{figure}[htb]
    \centering
    \begin{subfigure}[b]{0.99\columnwidth}
        \centering
        \def\svgwidth{\linewidth}
        {\footnotesize
            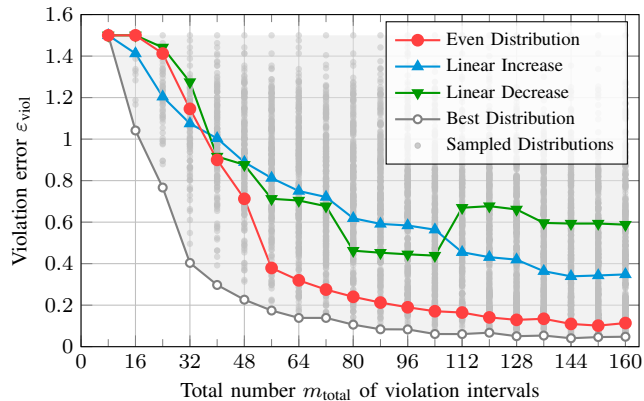}
        \caption{Visualization of the reachable set for the solution using the continuous cost function (blue) and for the solution using the discretized cost function (orange). Additionally, the continuous and discretized costs corresponding to the specifications $\varphi_k$ are presented (arrows pointing towards lower cost values).}
        \vspace{2mm}
    \end{subfigure}
    \begin{subfigure}[b]{0.99\columnwidth}
        \centering
        {\footnotesize
            \begin{tikzpicture}
    \begin{axis}[
        width=9cm,
        height=6cm,
        xlabel={Total number $m_{\text{total}}$ of violation intervals},
        ylabel={Violation error $\varepsilon_{\text{viol}}$},
        xmin=0, xmax=164,
        xtick={0,16,...,160},
        minor xtick={0,8,...,160},
        ytick={0,0.2,0.4,0.6,0.8,1.0,1.2,1.4,1.6},
        ymin=0, ymax=1.6,
        grid=major,
        legend style={
            at={(0.98,0.98)},
            anchor=north east,
            font=\scriptsize,
            fill=white,
            fill opacity=0.8,
            text opacity=1,
            cells={anchor=west}
        },
        label style={font=\footnotesize},
        tick label style={font=\footnotesize},
        cycle list name=color list,
    ]

    % --- Legend Construction (Manual) ---
    % 1. Even Distribution
    \addlegendimage{color=red!75!white, thick, mark=*, mark options={fill=red!75!white, fill opacity=1}}
    \addlegendentry{Even Distribution}

    % 2. Linear Increase
    \addlegendimage{color=cyan!80!blue, thick, mark=triangle*, mark options={fill=cyan!80!blue, fill opacity=1}}
    \addlegendentry{Linear Increase}

    % 3. Linear Decrease
    \addlegendimage{color=green!60!black, thick, mark=triangle*, mark options={fill=green!60!black, rotate=180, fill opacity=1}}
    \addlegendentry{Linear Decrease}

    % 4. Best Distribution
    \addlegendimage{color=black!50!white, thick, mark=*, mark size=1.5pt, mark options={fill=white, fill opacity=1}}
    \addlegendentry{Best Distribution}

    % 5. Raw Data
    \addlegendimage{only marks, mark=*, mark size=1.0pt, color=gray!50}
    \addlegendentry{Sampled Distributions}

    % --- Plots (Drawing Order: Bottom to Top) ---

    % 1. Range (Background)
    \addplot[name path=upper, draw=none, forget plot] table[x=x, y=y_max] {figures/adce_stats_data_scenario_3.txt};
    \addplot[name path=lower, draw=none, forget plot] table[x=x, y=y_min] {figures/adce_stats_data_scenario_3.txt};
    \addplot[fill=gray!20, fill opacity=0.5, forget plot] fill between[of=upper and lower];

    % 2. Raw Data (Scatter)
    \addplot[
        only marks,
        mark=*,
        mark size=1.0pt,
        color=gray!50,
        opacity=0.5,
        forget plot
    ] table[x=x, y=y] {figures/adce_scatter_data_scenario_3.txt};

    % 3. Best Distribution (Min Value)
    \addplot[
        color=black!50!white,
        thick,
        mark=*,
        mark size=1.5pt,
        mark options={fill=white},
        forget plot
    ] table[x=x, y=y_min] {figures/adce_stats_data_scenario_3.txt};

    % 4. Linear Increase
    \addplot[
        color=cyan!80!blue,
        thick,
        mark=triangle*,
        mark options={fill=cyan!80!blue},
        forget plot
    ] table[x=x, y=y_lin_inc] {figures/adce_stats_data_scenario_3.txt};

    % 5. Linear Decrease
    \addplot[
        color=green!60!black,
        thick,
        mark=triangle*,
        mark options={fill=green!60!black, rotate=180},
        forget plot
    ] table[x=x, y=y_lin_dec] {figures/adce_stats_data_scenario_3.txt};

    % 6. Even Distribution (Top)
    \addplot[
        color=red!75!white,
        thick,
        mark=*,
        mark options={fill=red!75!white},
        forget plot
    ] table[x=x, y=y_even] {figures/adce_stats_data_scenario_3.txt};

    \end{axis}
\end{tikzpicture}}
        \caption{Violation error for different interval distributions over the total number of violation intervals. Highlighted are the even distribution (red), distribution of linear increase (blue), distribution of linear decrease (green), and the best found distribution (gray).}
    \end{subfigure}
    \caption{Results of the discretization experiments. (a) Comparison of the lexicographically optimal solutions for the continuous and discretized cost function. (b) Influence of the discretization granularity and violation interval distribution on the violation error.}
    \label{fig:ExampleScenario}
    \vspace{-5mm}
\end{figure}

To investigate the effect of discretization granularity, we increase the total number of violation intervals $m_{\text{total}}$ stepwise from $8$ to $160$ and allocate them among the eight specifications by generating $10000$ unbiased random compositions via the stars-and-bars method~\cite{Feller1991}. By evaluating the violation error for all compositions at each step, we identify the distribution that minimizes the error. Also, we evaluate an even distribution (i.e., $m_{\varphi_1} = \dots = m_{\varphi_8}$), a linear increase, and a linear decrease, mimicking our hierarchical structure. Fig.~\ref{fig:ExampleScenario}~(b) illustrates the resulting violation errors. The graph of the best distribution shows that the finer the discretization, the lower the error. As the total number of violation intervals increases, the individual violation intervals become smaller, thereby mitigating the influence of order relaxation and order inversion. Furthermore, the error strongly depends on the chosen distribution of violation intervals. For a few violation intervals ($m_{\text{total}} < 40$), the linear decrease performs better. However, the even distribution consistently outperforms the other strategies for $m_{\text{total}} \geq 40$.

% ---------------------------
\subsubsection{Batch Evaluation}\label{sec:BatchEvaluation}

Next, we conduct a batch evaluation of $10000$ scenarios to validate our previous findings. For each scenario, the thresholds $r_k$ are randomly sampled from the interval $[-3, 3]$. As in the previous analysis, we vary the total number $m_{\text{total}}$ of violation intervals and sample $10000$ distinct interval distributions for each $m_{\text{total}}$.

Fig.~\ref{fig:AverageViolationErrorBatch} illustrates the average violation error $\bar\varepsilon_{\text{viol}}$ aggregated across all scenarios for different interval distribution strategies. As the total number of intervals increases, the average violation error decreases significantly across all strategies. This confirms that a finer discretization effectively mitigates the influence of order relaxation and order inversion. Also, the results show that even with a limited number of violation intervals, we can achieve a close approximation of the continuous lexicographic optimum, i.e., low violation errors. These findings strongly \textbf{support hypothesis~\ref{hyp:Discretization}}.

\begin{figure}[!htb]
    \centering
    {\footnotesize
        \begin{tikzpicture}
    \definecolor{color_lin_dec_std}{HTML}{CCEACC}
    \definecolor{color_lin_inc_std}{HTML}{CDE9F8}
    \definecolor{color_even_std}{HTML}{FFD8D8}
    \begin{axis}[
        width=9cm,
        height=6cm,
        xlabel={Total number $m_{\text{total}}$ of violation intervals},
        ylabel={Average violation error $\bar \varepsilon_{\text{viol}}$},
        xmin=0, xmax=164,
        xtick={0,16,...,160},
        minor xtick={0,8,...,160},
        ytick={0,0.2,0.4,0.6,0.8,1.0,1.2,1.4,1.6},
        ymin=0, ymax=1.5,
        grid=major,
        legend style={
            at={(0.98,0.98)},
            anchor=north east,
            font=\scriptsize,
            fill=white,
            fill opacity=0.8,
            text opacity=1,
            cells={anchor=west},
            legend columns=1 %2
        },
        label style={font=\footnotesize},
        tick label style={font=\footnotesize},
        cycle list name=color list,
    ]

    % --- Legend Construction (Manual) ---

    % 1. Linear Decrease
    \addlegendimage{color=green!60!black, thick, mark=triangle*, mark options={fill=green!60!black, rotate=180, fill opacity=1}}
    \addlegendentry{Linear Decrease}
    %\addlegendimage{area legend, fill=color_lin_dec_std, fill opacity=1, draw=green!60!black}
    %\addlegendentry{Lin. Dec. $\pm 1\sigma$}

    % 2. Linear Increase
    \addlegendimage{color=cyan!80!blue, thick, mark=triangle*, mark options={fill=cyan!80!blue, fill opacity=1}}
    \addlegendentry{Linear Increase}
    %\addlegendimage{area legend, fill=color_lin_inc_std, fill opacity=1, draw=cyan!80!blue}
    %\addlegendentry{Lin. Inc. $\pm 1\sigma$}

    % 3. Even Distribution
    \addlegendimage{color=red!75!white, thick, mark=*, mark options={fill=red!75!white, fill opacity=1}}
    \addlegendentry{Even Distribution}
    %\addlegendimage{area legend, fill=color_even_std, fill opacity=1, draw=red!75!white}
    %\addlegendentry{Even Dist. $\pm 1\sigma$}

    % 1. Linear Decrease Area
    \addplot[name path=lindec_upper, draw=none, forget plot] table[x=interval_sum, y=adce_lin_dec_upper] {figures/adce_statistics_tikz.txt};
    \addplot[name path=lindec_lower, draw=none, forget plot] table[x=interval_sum, y=adce_lin_dec_lower] {figures/adce_statistics_tikz.txt};
    \addplot[fill=color_lin_dec_std, fill opacity=0.5, draw=green!60!black, draw opacity=0.5, forget plot] fill between[of=lindec_upper and lindec_lower];

    % 2. Linear Increase Area
    \addplot[name path=lininc_upper, draw=none, forget plot] table[x=interval_sum, y=adce_lin_inc_upper] {figures/adce_statistics_tikz.txt};
    \addplot[name path=lininc_lower, draw=none, forget plot] table[x=interval_sum, y=adce_lin_inc_lower] {figures/adce_statistics_tikz.txt};
    \addplot[fill=color_lin_inc_std, fill opacity=0.5, draw=cyan!80!blue, draw opacity=0.5, forget plot] fill between[of=lininc_upper and lininc_lower];

    % 3. Even Distribution Area
    \addplot[name path=even_upper, draw=none, forget plot] table[x=interval_sum, y=adce_even_upper] {figures/adce_statistics_tikz.txt};
    \addplot[name path=even_lower, draw=none, forget plot] table[x=interval_sum, y=adce_even_lower] {figures/adce_statistics_tikz.txt};
    \addplot[fill=color_even_std, fill opacity=0.5, draw=red!75!white, draw opacity=0.5, forget plot] fill between[of=even_upper and even_lower];

    % --- Mean Plot Lines ---

    % 1. Linear Decrease Line
    \addplot[
        color=green!60!black,
        thick,
        mark=triangle*,
        mark options={fill=green!60!black, rotate=180},
        forget plot
    ] table[x=interval_sum, y=adce_lin_dec_mean] {figures/adce_statistics_tikz.txt};

    % 2. Linear Increase Line
    \addplot[
        color=cyan!80!blue,
        thick,
        mark=triangle*,
        mark options={fill=cyan!80!blue},
        forget plot
    ] table[x=interval_sum, y=adce_lin_inc_mean] {figures/adce_statistics_tikz.txt};

    % 3. Even Distribution Line
    \addplot[
        color=red!75!white,
        thick,
        mark=*,
        mark options={fill=red!75!white},
        forget plot
    ] table[x=interval_sum, y=adce_even_mean] {figures/adce_statistics_tikz.txt};
    \end{axis}
\end{tikzpicture}
    }
    \vspace{-5mm}
    \caption{Average violation error $\bar\varepsilon_{\text{viol}}$ over $10000$ scenarios for a varying total number of violation intervals and different distribution strategies. Solid lines represent the mean values, while the shaded areas indicate the standard deviation.}
    \label{fig:AverageViolationErrorBatch}
    \vspace{-5mm}
\end{figure}

% ##################################################################################
%  Solver Evaluations
% ##################################################################################
\subsection{Solver Evaluations}\label{sec:SolverEvaluations}

We investigate the influence of our solver adaptations on the solution quality using the $10000$ scenarios from Sec.~\ref{sec:BatchEvaluation}. Specifically, we examine the impact of the $\beta$-decay ($\mathfrak{B}_{\text{cos}}$ vs. $\mathfrak{B}_{\text{exp}}$), the sample size decay ($\mathfrak{M}_{\text{cos}}$ vs. $\mathfrak{M}_{\text{const}}$), and the output trajectory selection ($\by_{\text{mppi}}$ vs. $\by_{\text{best}}$). We compare a baseline solver (which uses $\mathfrak{B}_{\text{exp}}$, $\mathfrak{M}_{\text{const}}$, and $\by_{\text{mppi}}$) against configurations that incrementally enable our adaptations (see Tab.~\ref{tab:SolverEvaluation}), leading to our proposed Alg.~\ref{alg:PISolver}. We again use reachability analysis to compute the true lexicographic optimum $\mathcal{Y}^*_{\text{disc}}$ of the respective discretized problems for reference. The used parameters are: $J=20$, $\bSigma = 0.5$, $\lambda = 1.0$, $\hat x_0 = 0$, $\hat u_k = 0, \forall k \in [0, K]$, $\gamma =0.6$, $\beta_{\text{min}}=1.0 \mathrm{e}{-6}$, $M_{\text{init}} = 400$, and $M_{\text{final}} = 250$. The solvers are evaluated using the following metrics:
\begin{itemize}
    \item $P_{\text{lower}}, P_{\text{equal}}$, and $P_{\text{higher}}$: the percentage over all scenarios where the evaluated solver yields a lower, equal, or higher cost than the baseline solver, respectively;
    \item $\bar{\varepsilon}_{\text{solv}}$: the mean over all scenarios of the optimality gap $\varepsilon_{\text{solv}} := \frac{s(\by^*_{\text{solv}}) - s(\by^*_{\text{disc}}) }{s(\by^*_{\text{disc}})} \cdot 100\%$, where $\by^*_{\text{solv}}$ is the solution returned by the solver and $\by^*_{\text{disc}} \in \mathcal{Y}^*_{\text{disc}}$;
    \item $\bar{\Delta}{\varepsilon_{\text{solv}}}$: the mean over all scenarios of the gap improvement compared to the baseline, defined as $\Delta{\varepsilon_{\text{solv}}}:= \varepsilon_{\text{solv}} - \varepsilon_{\text{base}}$, where $\varepsilon_{\text{base}}$ is the optimality gap of the baseline solver.
\end{itemize}

Tab.~\ref{tab:SolverEvaluation} summarizes the results. Comparing the baseline solver to the true optimum, it already finds the true optimal solution in $63.17\%$ of the scenarios. Further, implementing our adaptations without reducing the sample size (Config.~1 to Config.~3) consistently outperforms the baseline and substantially lowers the mean optimality gap and the gap improvement. The configurations that use the cosine sample-size decay rule (Config.~4 to Config.~6 and Alg.~\ref{alg:PISolver}) reduce the overall number of used samples, while the degradation in solution quality is negligible compared to the other configurations. Overall, our Alg.~\ref{alg:PISolver} improves the solution quality while also reducing the number of evaluated samples, thereby \textbf{confirming hypothesis~\ref{hyp:Mppi}}.

\begin{table}[htb]
    \vspace{-2mm}
    \caption{Comparison of solver configurations. The baseline solver configuration uses $\mathfrak{B}_{\text{exp}}$, $\mathfrak{M}_{\text{const}}$, and $\by_{\text{mppi}}$. Arrows indicate whether higher or lower values are preferred. All values are percentages.}
    \centering
    \begin{adjustbox}{max width=1.0\columnwidth}
        \centering
        \setlength{\tabcolsep}{5pt}
        \begin{tabular}{
                l
                c
                c
                c
                S[table-format=1.3, round-mode=places, round-precision=3]
                S[table-format=1.3, round-mode=places, round-precision=3]
                S[table-format=1.2, round-mode=places, round-precision=2]
                S[table-format=1.2, round-mode=places, round-precision=2]
                S[table-format=3.0, round-mode=places, round-precision=0]
            }
            \toprule
            \multirow{2}{*}{\textbf{Method}} & \multicolumn{3}{c}{\textbf{Configuration}} & \multicolumn{5}{c}{\textbf{Metrics}}                                                                                                                                                                                                                               \\
            \cmidrule(lr){2-4} \cmidrule(lr){5-9}
                                             & $\mathfrak{M}_{\text{cos}}$                & $\mathfrak{B}_{\text{cos}}$          & $\by_{\text{best}}$ & {$P_{\text{lower}}\!\uparrow$} & {$P_{\text{equal}}$} & {$P_{\text{higher}}\!\downarrow$} & {$\bar{\varepsilon}_{\text{solv}}\!\downarrow$} & {$\bar{\Delta}{\varepsilon_{\text{solv}}}\!\downarrow$} \\
            \midrule
            Baseline                         & $\circ$                                    & $\circ$                              & $\circ$             & 0.00                           & 100.00               & 0.00                              & 4.127                                           & 0.000                                                   \\
            \midrule
            Config. 1                        & $\circ$                                    & $\circ$                              & $\bullet$           & 3.94                           & 96.06                & 0.00                              & 4.013                                           & -0.114                                                  \\
            Config. 2                        & $\circ$                                    & $\bullet$                            & $\circ$             & 23.01                          & 69.37                & 7.62                              & 0.700                                           & -3.427                                                  \\
            Config. 3                        & $\circ$                                    & $\bullet$                            & $\bullet$           & 29.59                          & 66.86                & 3.55                              & 0.233                                           & -3.894                                                  \\
            \midrule
            Config. 4                        & $\bullet$                                  & $\circ$                              & $\circ$             & 2.36                           & 94.25                & 3.39                              & 4.083                                           & -0.045                                                  \\
            Config. 5                        & $\bullet$                                  & $\circ$                              & $\bullet$           & 5.74                           & 91.82                & 2.44                              & 3.961                                           & -0.167                                                  \\
            Config. 6                        & $\bullet$                                  & $\bullet$                            & $\circ$             & 21.01                          & 69.43                & 9.56                              & 0.998                                           & -3.130                                                  \\
            Alg.~\ref{alg:PISolver}          & $\bullet$                                  & $\bullet$                            & $\bullet$           & 28.92                          & 67.12                & 3.96                              & 0.259                                           & -3.868                                                  \\
            \midrule
            Optimum                          & \multicolumn{3}{c}{--}                     & 36.83                                & 63.17               & 0.00                           & 0.000                & -4.127                                                                                                                                        \\
            \bottomrule
        \end{tabular}
    \end{adjustbox}
    \label{tab:SolverEvaluation}
\end{table}

% ##################################################################################
%  CommonRoad Experiments
% ##################################################################################
\subsection{CommonRoad Evaluations}\label{sec:CommonRoadExperiment}
We demonstrate the practical applicability of our minimum-violation motion planner in traffic scenarios from the CommonRoad benchmark suite~\cite{Althoff2017}. We first introduce the problem setup, then present the planning results for an intersection scenario, and finally compare our approach with the preemptive scheme and state-of-the-art solvers on $1750$ CommonRoad scenarios.
\vspace{-3mm}

\subsubsection{Problem Setup}

We use a  kinematic single-track model with state $\bx_k := [\mathtt{x}, \mathtt{y}, \theta, \delta, \mathtt{v}]^{\top}$, where $(\mathtt{x}, \mathtt{y})$ is the rear axle position, $\theta$ is the orientation, $\delta$ is the steering angle, and $\mathtt{v}$ is the velocity. The input is $\bu_k := [\mathtt{v}_{\delta}, \mathtt{a}]^{\top}$, where $\mathtt{v}_{\delta}$ is the steering velocity and $\mathtt{a}$ is the longitudinal acceleration. Given the wheelbase $L_{\text{wb}} \in \mathbb{R}$, the discrete-time system, obtained by explicit Euler discretization, is described by
\begin{equation*}
    \bx_{k+1} = \bx_k + \Delta t
    \begin{bmatrix}
        \mathtt{v}_k \cos(\theta_k)                       \\
        \mathtt{v}_k \sin(\theta_k)                       \\
        \frac{\mathtt{v}_k}{L_{\text{wb}}} \tan(\delta_k) \\
        \mathtt{v}_{\delta, k}                            \\
        \mathtt{a}_k
    \end{bmatrix},
    \quad \by_k =
    \begin{bmatrix}
        \bx_k \\
        \bu_k
    \end{bmatrix}.
\end{equation*}

We define 12 STL specifications that reflect a realistic autonomous-driving setting, including collision avoidance, maintaining safe distances, adhering to speed limits, and complying with a schedule. These specifications are inspired by~\cite{Maierhofer2020,Maierhofer2022}, and are defined as follows:
\begin{align*}
    \varphi_{1}  & := \mathbf{G}(\mu_{\text{velocity-in-limits}} \land \mu_{\text{steering-angle-in-limits}}),                                                                                                          \\
    \varphi_{2}  & := \textstyle\bigwedge_{o \in \mathcal{O}} (\mathbf{G}(\lnot \mu_{\text{collision}}^o)),                                                                                                             \\
    \varphi_{3}  & := \textstyle\bigwedge_{o \in \mathcal{O}} (\mathbf{G}((\mu_{\text{is-ambulance}}^o \land \mu_{\text{is-close}}^o)                                                                                   \\
                 & \qquad \Longrightarrow (\mu_{\text{is-shifted-to-left}} \land \mu_{\text{is-slow}}))),                                                                                                               \\
    \varphi_{4}  & := \textstyle\bigwedge_{o \in \mathcal{O}} (\mathbf{G}(\mu_{\text{keeps-static-distance}}^o)),                                                                                                       \\
    \varphi_{5}  & := \textstyle\bigwedge_{o \in \mathcal{O}} (\mathbf{G}((\mu_{\text{in-same-lane}}^{o} \land \mu_{\text{in-front-of}}^{o}                                                                             \\
                 & \qquad \land \lnot \mathbf{O}_{[0, 10]}(\mu_{\text{cut-in}}^{o} \land \mathbf{O}_{[1,1]}(\lnot \mu_{\text{cut-in}}^{o})))                                                                            \\
                 & \qquad \Longrightarrow \mu_{\text{keeps-safe-distance-prec}}^{o})),                                                                                                                                  \\
    \varphi_{6}  & := \mathbf{G}(\mu_{\text{in-lane}}),                                                                                                                                                                 \\
    \varphi_{7}  & := \textstyle\bigwedge_{\text{in} \in \mathcal{I}} \textstyle\bigwedge_{o \in \mathcal{O}} (\mathbf{G}(\mu_{\text{has-priority}}^o \Longrightarrow \lnot \mu_{\text{in-intersection}}^{\text{in}})), \\
    \varphi_{8}  & := \mathbf{G}(\mu_{\text{below-speed-limit}}),                                                                                                                                                       \\
    \varphi_{9}  & := \mathbf{G}(\mu_{\text{brakes-abruptly}} \Longrightarrow \textstyle\bigwedge_{o \in \mathcal{O}}(\mu_{\text{in-front-of}}^{o} \land \mu_{\text{in-same-lane}}^{o}                                  \\
                 & \qquad  \land (\lnot\mu_{\text{keeps-safe-distance-prec}}^{o} \lor \mu_{\text{brakes-abruptly-relative}}^{o}))),                                                                                     \\
    \varphi_{10} & := \mathbf{G}(( \textstyle\bigwedge_{o \in \mathcal{O}}(\lnot(\mu_{\text{is-slow}}^o \land \mu_{\text{in-same-lane}}^{o} \land \mu_{\text{in-front-of}}^{o})))                                       \\
                 & \qquad \Longrightarrow \mu_{\text{preserves-flow}}),                                                                                                                                                 \\
    \varphi_{11} & := \mathbf{G}(\mu_{\text{below-long-acc-limit}} \land \mu_{\text{below-lat-acc-limit}}),                                                                                                             \\
    \varphi_{12} & := \mathbf{F}_{[K,K]}(\mu_{\text{at-scheduled-long-pos}}),
\end{align*}
where $\mathcal{O}$ is the set of obstacles and $\mathcal{I}$ is the set of intersections in the CommonRoad scenario, respectively. The movement of the obstacles is predicted using a constant-velocity model. For details on the implementation of the predicates, see our source code.

As we investigate the in-lane-driving specification $\varphi_{6}$ later in detail, we formally introduce the used predicate $\mu_{\text{in-lane}}$ subsequently. Provided with a reference path, we assume a constant lane width $d_{\text{lane}}$ along this path. Furthermore, we introduce the functions $\operatorname{leftDeviation}(\by_k)$ and $\operatorname{rightDeviation}(\by_k)$, which determine the maximum left and right deviation from the reference path of the ego vehicle, respectively, utilizing a three-disc approximation of the ego vehicle occupancy~\cite{Ziegler2010}. The predicate is defined as
\begin{align*}
    \mu_{\text{in-lane}} := \min\Bigl( & \frac{d_{\text{lane}}}{2} - \operatorname{leftDeviation}(\by_k),               \\
                                       & \operatorname{rightDeviation}(\by_k) - \frac{d_{\text{lane}}}{2}\Bigr) \geq 0.
\end{align*}

We adopt the following discretization scheme: For the highest-priority specifications $\varphi_1$ and $\varphi_2$, we only differentiate between satisfaction and violation, i.e., $m_{\varphi_1} = m_{\varphi_2} = 1$. Note that other discretizations are also possible. For specifications $\varphi_3$ to $\varphi_{12}$, we use \eqref{eq:IntervalBoundariesGenerator} to define the interval bounds. The specific upper bound cost values $\overline{c}_{\varphi_{i}}$ for each specification are provided in our source code. Tuning $\overline{c}_{\varphi_{i}}$ is merely required because the value ranges of the considered predicate functions $p_\mu(\by, k)$ vary greatly, since they express diverse quantities, e.g., distances, velocities, or abstract measures like criticality, risk, or comfort (cf. \cite{Lin2023a, Westhofen2022, Geisslinger2023, Trauth2023, Rajesh2023}). Normalizing the predicate functions to eliminate manual tuning of $\overline{c}_{\varphi_{i}}$ is left for future work. We set $m_{\varphi_i} = 31$ for all $i \in \{3, \dots, 12\}$ and utilize our space-left-time robustness $\eta_{\text{space-left-time}}$ with $w = 15$ for all predicate functions. The remaining parameters are summarized in Tab.~\ref{tab:CommonRoadParameters}.
\begin{table}[b]
    \centering
    \caption{Parameters of the CommonRoad experiments.}
    \begin{tabular}{ll}
        \toprule
        \textbf{Description}      & \textbf{Notation and value}                                                           \\
        \midrule
        \textbf{Problem Parameters}                                                                                       \\
        Time increment            & $\Delta t = 0.2\,\text{s}$                                                            \\
        Planning time horizon     & $K = 15$                                                                              \\
        Upper input bounds        & $\overline{\bu} = [0.3\,\text{rad/s}, 8.0\,\text{m/s}^2]$                             \\
        Lower input bounds        & $\underline{\bu} = [-0.3\,\text{rad/s}, -8.0\,\text{m/s}^2]$                          \\
        Wheelbase                 & $L_{\text{wb}} = 3.0\,\text{m}$                                                       \\
        Lane width                & $d_{\text{lane}} = 4.0\,\text{m}$                                                     \\
        Min-Max functions         & $\nu_1=\nu_2=10, \nu_3=1, \nu_4=\nu_5=2$                                              \\
        \midrule
        \textbf{Solver Parameters}                                                                                        \\
        Number of iterations      & $J=20$                                                                                \\
        Initial covariance        & $\bSigma = \left[\begin{smallmatrix} 0.1 & 0.0 \\ 0.0 & 6.0 \end{smallmatrix}\right]$ \\
        Initial temperature       & $\lambda = 1.0$                                                                       \\
        Initial input trajectory  & $\hat\bu_k = [0\,\text{rad/s}, 0\,\text{m/s}^2], \forall k \in [0, K]$                \\

        Minimum decay value       & $\beta_{\text{min}}=1.0 \mathrm{e}{-6}$                                               \\
        Initial number of samples & $M_{\text{init}} = 1000$                                                              \\
        Final number of samples   & $M_{\text{final}} = 100$                                                              \\
        \bottomrule
    \end{tabular}
    \label{tab:CommonRoadParameters}
\end{table}

\subsubsection{Intersection Scenario}

We analyze the CommonRoad scenario depicted in Fig.~\ref{fig:CRScenario}, representing a busy intersection on which the ego vehicle intends to turn left. This scenario includes obstacles approaching from the left and right with priority, an approaching ambulance from behind the ego vehicle, and an impatient driver cutting in front of the ego vehicle. We apply our minimum-violation planner in an MPC scheme. At each MPC iteration, the planner computes an optimal trajectory $\by^*$, which the ego vehicle follows for one discrete time instance. The initial state is then updated accordingly. This procedure is repeated $H=70$ times.

\begin{figure*}[htb]
    \centering
    \begin{subfigure}[b]{1.98\columnwidth}
        \centering
        \def\svgwidth{\linewidth}
        {\footnotesize
            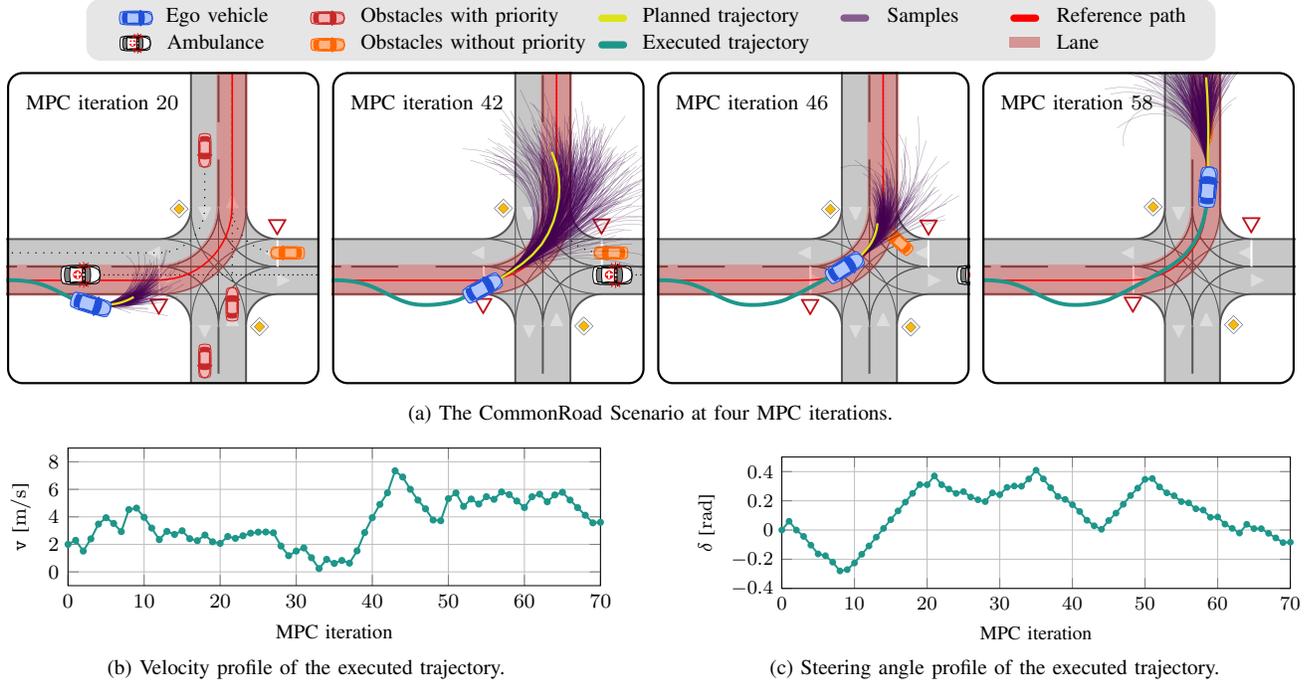}
        \caption{The CommonRoad Scenario at four MPC iterations.}
    \end{subfigure}
    \begin{subfigure}[b]{0.99\columnwidth}
        \centering
        \resizebox{\columnwidth}{!}{
            \definecolor{customgreen}{HTML}{1E968A}
\begin{tikzpicture}
    \begin{axis}[
        width=9cm,
        height=3.5cm,
        xlabel={MPC iteration},
        ylabel={$\mathtt{v}$ [\si{m/s}]},
        xmin=0, xmax=70,
        ymin=-1, ymax=9,
        grid=major,
        label style={font=\footnotesize},
        tick label style={font=\footnotesize},
    ]
        \addplot[color=customgreen, thick, mark=*, mark size=1pt] table [x=step, y=velocity] {figures/executed_trajectory_data.dat};
    \end{axis}
\end{tikzpicture}}
        \vspace{-5mm}
        \caption{Velocity profile of the executed trajectory.}
    \end{subfigure}
    \hfill
    \begin{subfigure}[b]{0.99\columnwidth}
        \centering
        \resizebox{\columnwidth}{!}{
            \definecolor{customgreen}{HTML}{1E968A}
\begin{tikzpicture}
    \begin{axis}[
        width=9cm,
        height=3.5cm,
        xlabel={MPC iteration},
        ylabel={$\delta$ [\si{rad}]},
        xmin=0, xmax=70,
        ymin=-0.4, ymax=0.5,
        ytick={-0.4, -0.2, 0, 0.2, 0.4},
        grid=major,
        label style={font=\footnotesize},
        tick label style={font=\footnotesize},
    ]
        \addplot[color=customgreen, thick, mark=*, mark size=1pt] table [x=step, y=steering_angle] {figures/executed_trajectory_data.dat};
    \end{axis}
\end{tikzpicture}}
        \vspace{-5mm}
        \caption{Steering angle profile of the executed trajectory.}
    \end{subfigure}
    \caption{Planning results of the CommonRoad intersection scenario. (a) Scenario at four MPC iterations, including the planned trajectory, sampled trajectories, and executed trajectory. (b) Velocity profile and (c) steering angle profile of the executed trajectory over the entire MPC horizon.}
    \label{fig:CRScenario}
    \vspace{-5mm}
\end{figure*}

Fig.~\ref{fig:CRScenario} shows the scenario at four MPC iterations, visualizing the planned, sampled, and executed trajectories, respectively, along with the velocity and steering angle profiles for the entire MPC horizon. A video of this maneuver is provided as a supplementary attachment. During the maneuver, the ego vehicle shifts to the right to allow the ambulance to pass (see MPC iteration $20$). Subsequently, it returns to the lane and plans to cross the intersection, as the impatient driver indicates no intention to move at this time and no other priority vehicles are present (see MPC iteration $42$). However, the impatient driver abruptly cuts in front, necessitating a hard-braking maneuver of the ego vehicle (see MPC iteration $46$). Following this, the ego vehicle proceeds by following the leading vehicle (see MPC iteration $58$). Overall, given the specifications and system bounds, our minimum-violation planning approach yields a reasonable driving behavior in this highly complex scenario, thereby \textbf{confirming hypothesis \ref{hyp:CRScenario}}.

\subsubsection{Comparison with Alternative Approaches}\label{sec:AlternativeSolvers}

On $1750$ scenarios from the CommonRoad benchmark suite, we compare our solution approach with the preemptive scheme and evaluate our deterministic MPPI solver against state-of-the-art solvers.

We implement the preemptive scheme according to \eqref{eq:LexMinProblem} using our deterministic MPPI solver and the parameters in Tab.~\ref{tab:CommonRoadParameters}. Instead of minimizing $s(\by)$, we sequentially optimize the respective cost functions $\tilde c_{\varphi_i}(\by)$. The constraints $c_j(\by)=c_j(\by_j^*)$ in \eqref{eq:LexMinProblem} are enforced by discarding violating samples. Unlike our approach, the preemptive scheme must enforce these constraints and therefore fails to find feasible solutions in $51.3\%$ of the considered scenarios. In the remaining scenarios, both approaches achieve nearly identical solution quality. However, due to the sequential nature of the preemptive scheme, its average solve time of $4.9\,\si{\second}$ is significantly higher than the $0.38\,\si{\second}$ of our solution approach. These results strongly \textbf{support hypothesis~\ref{hyp:AlternativeSolvers}}.

Next, we compare our deterministic MPPI solver against the following state-of-the-art solvers: random shooting (RS), covariance matrix adaptation evolution strategy (CMA-ES), differential evolution (DE), simulated annealing (SA), and Frenetix (FREN). RS is our own implementation, FREN is adopted from~\cite{Trauth2024}, and for CMA-ES, DE, and SA, we use the implementations provided by the optimization library \texttt{pagmo}~\cite{Biscani2020}. To ensure a fair comparison, we choose the solver hyperparameters such that each solver evaluates approximately the same number of trajectory samples. We then extensively tune the remaining hyperparameters with \texttt{Optuna}~\cite{Akiba2019} on $150$ randomly selected CommonRoad scenarios, maximizing the number of scenarios in which the respective solver achieves a lower cost than our deterministic MPPI solver. The resulting parameters are provided in our source code.

Across the $1750$ CommonRoad scenarios, we compare the percentage $P_{\text{equal-or-lower}}$ of scenarios in which our deterministic MPPI solver achieves an equal or lower cost than the respective alternative solver. The results are presented in Tab.~\ref{tab:AlternativeSolvers}.

\begin{table}[htb]
    \centering
    \caption{Percentage of CommonRoad scenarios in which our deterministic MPPI solver achieves an equal or lower cost than state-of-the-art solvers.}
    \label{tab:AlternativeSolvers}
    \begin{tabular}{
            l
            S[table-format=2.1]
            S[table-format=2.1]
            S[table-format=2.1]
            S[table-format=2.1]
            S[table-format=2.1]
        }
        \toprule
        \textbf{Solver} & {RS} & {CMA-ES} & {DE} & {SA} & {FREN} \\
        \midrule
        {$P_{\text{equal-or-lower}}\,[\%]$}
                        & 93.1 & 77.2     & 75.0 & 83.5 & 74.2   \\
        \bottomrule
    \end{tabular}
\end{table}

The results indicate that, although the state-of-the-art solvers occasionally obtain solutions with lower cost, our deterministic MPPI solver achieves equal or lower cost in the majority of the evaluated scenarios. The consistently high percentages show that our deterministic MPPI solver is competitive with these solvers, thereby \textbf{confirming hypothesis~\ref{hyp:AlternativeSolvers}}.

% ------------------------------------------------------------------
\subsection{Comparison of Robustness Measures}\label{sec:ComparisonOfRobustnessMeasures}

Finally, we compare the comprehensive set of robustness measures defined in~Tab.~\ref{tab:RobustnessDefinitions} and conduct experiments using our running example and the previously introduced CommonRoad scenario. We also assess the computational performance of the measures across $1750$ CommonRoad scenarios.

\subsubsection{Running Example}\label{sec:ComparisonOfRobustnessMeasures_RunningExample}

We revisit the overtaking maneuver of a broken-down obstacle from the running example. Fig.~\ref{fig:robustness_comparison_scenarios}~(a) illustrates two distinct planning times $t_1$ and $t_2$, and $21$ manually created trajectories, respectively. The lane width is $d_{\text{lane}} = 6\,\text{m}$. We evaluate the robustness measures from Tab.~\ref{tab:RobustnessDefinitions} on the trajectories. For clarity, we omit $\eta_{\text{left-time}}$, $\eta_{\text{right-time}}$, and $\eta_{\text{dur}}$ as they behave similarly to $\eta_{\text{comb-time}}$, as well as $\eta_{\text{agm}}$ and $\eta_{\text{new}}$ which behave similar to $\eta_{\text{pm}}$. Fig.~\ref{fig:robustness_comparison_scenarios}~(b) shows the resulting robustness values for all samples for both time steps.
\begin{figure}[!htb]
    \centering
    \begin{subfigure}[b]{0.99\columnwidth}
        \centering
        \def\svgwidth{\linewidth}
        {\footnotesize
            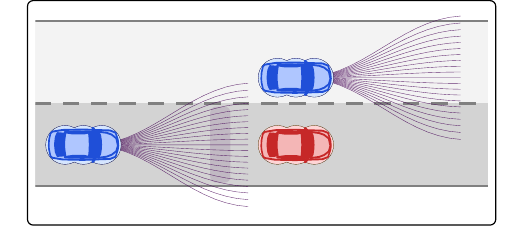}
        \caption{Scenario at time $t_1$ and time $t_2$ with the respective trajectory samples.}
        \vspace{2mm}
    \end{subfigure}
    \begin{subfigure}[b]{0.99\columnwidth}
        \centering
        {\footnotesize
            \resizebox{\columnwidth}{!}{
                \begin{tikzpicture}
    % Group 1: Space (Strong Red)
    \definecolor{spaceRed}{HTML}{C1121F}

    % Group 2: Averaging (Teal → Cyan)
    \definecolor{avgDark}{HTML}{005F73}
    \definecolor{avgLight}{HTML}{8EDED5}

    % Group 3: Duration (Burnt Orange)
    \definecolor{durLight}{HTML}{F08C00}

    % Group 4: Time (Deep Blue → Sky)
    \definecolor{timeLight}{HTML}{4EA8DE}

    % Group 5: Space Left Time (Royal Purple)
    \definecolor{spaceTime}{HTML}{009900}

    \pgfplotsset{
        my style/.style={
            width=9cm,
            height=5cm,
            xlabel={Sample index},
            ylabel={Robustness value $\eta^{\varphi_{\text{lane}}}$},
            xmin=0, xmax=20,
            xtick={0,2,...,20},
            minor xtick={0,1,...,20},
            grid=major,
            cycle list name=color list,
            %yticklabel style={text width=4ex, align=right}, % Ensure vertical alignment of axes
        }
    }

    % ==========================================================================================
    % Scenario 0 (Top)
    % ==========================================================================================
    \begin{axis}[
        name=plot0,
        my style,
        ymin=-1.1, ymax=1.1,
        ytick={-1,-0.5,0,0.5,1},
        xticklabels={}, % Hide x labels
        xlabel= {},
        extra y ticks={0},
        extra y tick style={grid style={black, thick}},
        legend style={
            at={(0.5,1.025)},
            anchor=south,
            font=\footnotesize,
            draw=none,
            cells={anchor=west},
            legend columns=3,
        },
    ]

    %\addplot[color=gray, thin, forget plot] table[x=SampleIndex, y=AGM] {figures/robustness_values_scenario_0.txt};
    %\addplot[color=gray, thin, forget plot] table[x=SampleIndex, y=New] {figures/robustness_values_scenario_0.txt};
    %\addplot[color=gray, thin, forget plot] table[x=SampleIndex, y=Duration] {figures/robustness_values_scenario_0.txt};
    %\addplot[color=gray, thin, forget plot] table[x=SampleIndex, y=TimeLeft] {figures/robustness_values_scenario_0.txt};
    %\addplot[color=gray, thin, forget plot] table[x=SampleIndex, y=TimeRight] {figures/robustness_values_scenario_0.txt};

    % Smooth
    \addplot[color=avgDark, thick, mark=square*, mark size=1pt, mark options={fill=avgDark}]
    table[x=SampleIndex, y=Smooth] {figures/robustness_values_scenario_0.txt};
    \addlegendentry{$\eta_{\text{smooth}}$}

    % PowerMean
    \addplot[color=avgLight, thick, mark=square*, mark size=1pt, mark options={fill=avgLight}]
    table[x=SampleIndex, y=PowerMean] {figures/robustness_values_scenario_0.txt};
    \addlegendentry{$\eta_{\text{pm}}$}

    % DurationSeverity
    \addplot[color=durLight, thick, mark=square*, mark size=1pt, mark options={fill=durLight}]
    table[x=SampleIndex, y=DurationSeverity] {figures/robustness_values_scenario_0.txt};
    \addlegendentry{$\eta_{\text{dur-sev}}$}

    % Space
    \addplot[color=spaceRed, line width=1pt, mark=*, mark size=1pt, mark options={fill=spaceRed}]
    table[x=SampleIndex, y=Space] {figures/robustness_values_scenario_0.txt};
    \addlegendentry{$\eta_{\text{space}}$}

    % TimeCombined
    \addplot[color=timeLight, line width=1pt, mark=triangle*, mark size=1pt, mark options={fill=timeLight}]
    table[x=SampleIndex, y=TimeCombined] {figures/robustness_values_scenario_0.txt};
    \addlegendentry{$\eta_{\text{comb-time}}$}

    % SpaceLeftTime
    \addplot[color=spaceTime, line width=1pt, mark=*, mark size=1pt, mark options={fill=spaceTime}]
    table[x=SampleIndex, y=SpaceLeftTime] {figures/robustness_values_scenario_0.txt};
    \addlegendentry{$\eta_{\text{space-left-time}}$}

    \node[anchor=north west, fill=white, fill opacity=0.8, text opacity=1, inner sep=1pt] at (rel axis cs: 0.02, 0.95) {\scriptsize Ego vehicle at $t_1$};

    \end{axis}

    % ==========================================================================================
    % Scenario 1 (Bottom)
    % ==========================================================================================
    \begin{axis}[
        name=plot1,
        at={(plot0.south west)},
        anchor=north west,
        ymin=-1.05, ymax=-0.25,
        ytick={-1, -0.8, -0.6, -0.4, -0.2},
        yshift=-4mm, % Small vertical gap
        my style,
    ]

    %\addplot[color=gray, thin, forget plot] table[x=SampleIndex, y=AGM] {figures/robustness_values_scenario_1.txt};
    %\addplot[color=gray, thin, forget plot] table[x=SampleIndex, y=New] {figures/robustness_values_scenario_1.txt};
    %\addplot[color=gray, thin, forget plot] table[x=SampleIndex, y=Duration] {figures/robustness_values_scenario_1.txt};
    %\addplot[color=gray, thin, forget plot] table[x=SampleIndex, y=TimeLeft] {figures/robustness_values_scenario_1.txt};
    %\addplot[color=gray, thin, forget plot] table[x=SampleIndex, y=TimeRight] {figures/robustness_values_scenario_1.txt};

    % Smooth
    \addplot[color=avgDark, thick, mark=square*, mark size=1pt, mark options={fill=avgDark}]
    table[x=SampleIndex, y=Smooth] {figures/robustness_values_scenario_1.txt};
    %\addlegendentry{$\eta_{\text{smooth}}$}

    % PowerMean
    \addplot[color=avgLight, thick, mark=square*, mark size=1pt, mark options={fill=avgLight}]
    table[x=SampleIndex, y=PowerMean] {figures/robustness_values_scenario_1.txt};
    %\addlegendentry{$\eta_{\text{pm}}$}

    % DurationSeverity
    \addplot[color=durLight, thick, mark=square*, mark size=1pt, mark options={fill=durLight}]
    table[x=SampleIndex, y=DurationSeverity] {figures/robustness_values_scenario_1.txt};
    %\addlegendentry{$\eta_{\text{dur-sev}}$}

    % Space
    \addplot[color=spaceRed, line width=1pt, mark=*, mark size=1pt, mark options={fill=spaceRed}]
    table[x=SampleIndex, y=Space] {figures/robustness_values_scenario_1.txt};
    %\addlegendentry{$\eta_{\text{space}}$}

    % TimeCombined
    \addplot[color=timeLight, line width=1pt, mark=triangle*, mark size=1pt, mark options={fill=timeLight}]
    table[x=SampleIndex, y=TimeCombined] {figures/robustness_values_scenario_1.txt};
    %\addlegendentry{$\eta_{\text{comb-time}}$}

    % SpaceLeftTime
    \addplot[color=spaceTime, line width=1pt, mark=*, mark size=1pt, mark options={fill=spaceTime}]
    table[x=SampleIndex, y=SpaceLeftTime] {figures/robustness_values_scenario_1.txt};
    %\addlegendentry{$\eta_{\text{space-left-time}}$}

    \node[anchor=north west, fill=white, fill opacity=0.8, text opacity=1, inner sep=1pt] at (rel axis cs: 0.02, 0.95) {\scriptsize Ego vehicle at $t_2$};

    \end{axis}
\end{tikzpicture}}}
        \vspace{-5mm}
        \caption{Robustness values of the specification $\varphi_{\text{lane}}$ for different robustness measures at time $t_1$ (top) and time $t_2$ (bottom). The values are normalized to $[-1,1]$ for each measure.}
    \end{subfigure}
    \caption{Comparison of robustness measures for the running example. (a) Visualization of the scenario and trajectory samples. (b) Robustness values for $\varphi_{\text{lane}}$.}
    \label{fig:robustness_comparison_scenarios}
    \vspace{-7mm}
\end{figure}

For our motion planning purposes, trajectories describing different maneuvers must be numerically distinguishable. This does not hold for the space robustness $\eta_{\text{space}}$, which is evident at $t_2$, as the samples that would turn right towards the lane (samples $10$ to $20$) share the same robustness value. This effect is even more dominant for the combined-time robustness $\eta_{\text{comb-time}}$. At time $t_1$, we cannot differentiate the samples that remain in the lane (samples $7$ to $13$). At time $t_2$, only the samples $18$ to $20$ can be distinguished, as only for these samples the ego vehicle is completely back in the lane. For the smooth robustness $\eta_{\text{smooth}}$ at $t_2$, the samples $12$ to $20$ have almost identical robustness values, making differentiation challenging. In contrast, the power-mean robustness $\eta_{\text{pm}}$, the duration-severity robustness $\eta_{\text{dur-sev}}$, and our space-left-time robustness $\eta_{\text{space-left-time}}$ provide a good differentiation of the trajectories at both time steps. These results \textbf{support hypothesis~\ref{hyp:RobustnessMeasures}}.

\subsubsection{CommonRoad Scenario}

We now evaluate the robustness measures on the CommonRoad scenario from Sec.~\ref{sec:CommonRoadExperiment}. We plan over $H=70$ MPC iterations using our Alg.~\ref{alg:PISolver}. Fig.~\ref{fig:cr_robustness_comparison_scenarios}~(a) visualizes the scenario, and Fig.~\ref{fig:cr_robustness_comparison_scenarios}~(b) presents the robustness values of the executed trajectory per measure over the MPC iterations for the in-lane-driving specification $\varphi_{6}$. Again, we only present a subset of the robustness measures for clarity.
\begin{figure}[!htb]
    \centering
    \begin{subfigure}[b]{0.99\columnwidth}
        \centering
        \def\svgwidth{\linewidth}
        {\footnotesize
            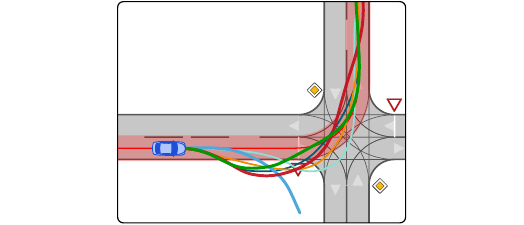}
        \caption{Executed trajectories for different robustness measures.}
        \vspace{2mm}
    \end{subfigure}
    \begin{subfigure}[b]{0.99\columnwidth}
        \centering
        {\footnotesize
            \resizebox{\columnwidth}{!}{
                \begin{tikzpicture}[spy using outlines={circle, magnification=1.9, size=1.5cm, connect spies}]
    % Group 1: Space (Strong Red)
    \definecolor{spaceRed}{HTML}{C1121F}

    % Group 2: Averaging (Teal → Cyan)
    \definecolor{avgDark}{HTML}{005F73}
    \definecolor{avgLight}{HTML}{8EDED5}

    % Group 3: Duration (Burnt Orange)
    \definecolor{durLight}{HTML}{F08C00}

    % Group 4: Time (Deep Blue → Sky)
    \definecolor{timeLight}{HTML}{4EA8DE}

    % Group 5: Space Left Time (Royal Purple)
    \definecolor{spaceTime}{HTML}{009900}

    \begin{axis}[
        name=plot0,
        width=9cm,
        height=6cm,
        xlabel={Discrete time $k$},
        ylabel={Robustness value $\eta^{\varphi_6}$},
        xmin=0, xmax=70,
        xtick={0,10,...,70},
        minor xtick={0,5,...,70},
        grid=major,
        cycle list name=color list,
        ymin=-1.1, ymax=0.25,
        ytick={-1,-0.75, -0.5,-0.25, 0, 0.25, 0.5},
        extra y ticks={0},
        extra y tick style={grid style={black, thick}},
        legend style={
            at={(0.5,1.025)},
            anchor=south,
            font=\footnotesize,
            fill=none,
            draw=none,
            cells={anchor=west},
            legend columns=3,
        },
    ]

    % Background plots (commented out as in reference)
    %\addplot[color=gray, thin, forget plot] table[x=TimeStep, y=AGM] {figures/scaled_inlanedriving_robustness.txt};
    %\addplot[color=gray, thin, forget plot] table[x=TimeStep, y=New] {figures/scaled_inlanedriving_robustness.txt};
    %\addplot[color=gray, thin, forget plot] table[x=TimeStep, y=Duration] {figures/scaled_inlanedriving_robustness.txt};
    %\addplot[color=gray, thin, forget plot] table[x=TimeStep, y=TimeLeft] {figures/scaled_inlanedriving_robustness.txt};
    %\addplot[color=gray, thin, forget plot] table[x=TimeStep, y=TimeRight] {figures/scaled_inlanedriving_robustness.txt};

    % Smooth
    \addplot[color=avgDark, thick, mark=square*, mark size=1pt, mark repeat=5, mark options={fill=avgDark}]
    table[x=TimeStep, y=Smooth] {figures/scaled_inlanedriving_robustness.txt};
    \addlegendentry{$\eta_{\text{smooth}}$}

    % PowerMean
    \addplot[color=avgLight, thick, mark=square*, mark size=1pt, mark repeat=5, mark options={fill=avgLight}]
    table[x=TimeStep, y=PowerMean] {figures/scaled_inlanedriving_robustness.txt};
    \addlegendentry{$\eta_{\text{pm}}$}

    % DurationSeverity
    \addplot[color=durLight, thick, mark=square*, mark size=1pt, mark repeat=5, mark options={fill=durLight}]
    table[x=TimeStep, y=DurationSeverity] {figures/scaled_inlanedriving_robustness.txt};
    \addlegendentry{$\eta_{\text{dur-sev}}$}

    % Space
    \addplot[color=spaceRed, line width=1pt, mark=*, mark size=1pt, mark repeat=5, mark options={fill=spaceRed}]
    table[x=TimeStep, y=Space] {figures/scaled_inlanedriving_robustness.txt};
    \addlegendentry{$\eta_{\text{space}}$}

    % TimeCombined
    \addplot[color=timeLight, line width=1pt, mark=triangle*, mark size=1pt, mark repeat=5, mark options={fill=timeLight}]
    table[x=TimeStep, y=TimeCombined] {figures/scaled_inlanedriving_robustness.txt};
    \addlegendentry{$\eta_{\text{comb-time}}$}

    % SpaceLeftTime
    \addplot[color=spaceTime, line width=1pt, mark=*, mark size=1pt, mark repeat=5, mark options={fill=spaceTime}]
    table[x=TimeStep, y=SpaceLeftTime] {figures/scaled_inlanedriving_robustness.txt};
    \addlegendentry{$\eta_{\text{space-left-time}}$}

    \coordinate (spypoint) at (axis cs:64.5,0.02);
    \coordinate (magnifyglass) at (axis cs:60,-0.65);

    \end{axis}
    \spy [black, size=2.5cm] on (spypoint) in node [fill=white] at (magnifyglass);
\end{tikzpicture}}}
        \vspace{-5mm}
        \caption{Robustness values for the specification $\varphi_6$ over the MPC iterations. The values are normalized to $[-1,1]$ for each measure.}
    \end{subfigure}
    \caption{Comparison of robustness measures in the CommonRoad scenario. (a) Visualization of the executed trajectories. (b) Robustness values for the in-lane-driving specification $\varphi_6$.}
    \label{fig:cr_robustness_comparison_scenarios}
\end{figure}

Fig.~\ref{fig:cr_robustness_comparison_scenarios}~(a) shows that all robustness measures, except the combined-time robustness $\eta_{\text{comb-time}}$, initially swerve to the right to clear the way for the ambulance and subsequently return to the lane. In contrast, $\eta_{\text{comb-time}}$ causes the ego vehicle to drive off the road and fails to return to the lane, which is also evident in its robustness values in Fig.~\ref{fig:cr_robustness_comparison_scenarios}~(b). Similarly, the values of the space robustness $\eta_{\text{space}}$ remain negative after the ego vehicle leaves the lane, clearly emphasizing the lack of differentiation between trajectory samples (as demonstrated in Sec.~\ref{sec:ComparisonOfRobustnessMeasures_RunningExample}). For the space-left-time robustness $\eta_{\text{space-left-time}}$, the ego vehicle rapidly moves to the right, as indicated by the steep decline of the robustness values during MPC iterations $0$ to $10$. Subsequently, the ego vehicle quickly returns to the lane, and the robustness remains positive after MPC iteration $45$. The remaining robustness measures exhibit similar behavior.

These results show that, unlike our combined measure, the pure space and time robustness measures are not suitable for our motion planning purposes. This again \textbf{supports hypothesis~\ref{hyp:RobustnessMeasures}}.

\subsubsection{Computational Performance}

Finally, we compare the computational performance of the robustness measures over $1750$ CommonRoad scenarios. We compare the following metrics over all scenarios:
\begin{itemize}
    \item $P_{\text{calls}}$: the number of calls of the predicate function $p_{\mu_{\text{in-lane}}}(\by,k)$ for one trajectory evaluation;
    \item $t_{\text{rob}}$: the time required to evaluate the respective robustness measure for one trajectory for the in-lane-driving specification $\varphi_6$;
    \item $t_{\text{sol}}$: the total solve time of Alg.~\ref{alg:PISolver} per MPC iteration.
\end{itemize}

The results are presented in Tab.~\ref{tab:CalculationTimes}: Due to their mathematical definition (cf. Tab.~\ref{tab:PredicateRobustnessDefinitions}), the space-left-time robustness and the time robustness measures inherently require the highest number of calls $P_{\text{calls}}$ to the predicate function $p_{\mu_{\text{in-lane}}}(\by,k)$. However, for a given trajectory $\by$, any predicate function $p_{\mu}(\by,k)$ can take at most $K+1$ distinct values. By caching these values, we avoid redundant predicate function evaluations, which significantly reduces the required computation time for the space-left-time robustness and the time robustness measures. As a result, evaluating our space-left-time robustness for specification $\varphi_6$ only doubles $t_{\text{rob}}$ compared to the standard space robustness. For specifications with more complex predicates (e.g., the collision avoidance specification $\varphi_2$, which requires expensive geometric computations for each obstacle), this relative difference in evaluation time for the robustness measures becomes even smaller. In these cases, $t_{\text{rob}}$ is heavily dominated by the evaluation of the predicate functions rather than the calculation of the robustness measure itself. This is also reflected in the total solve times $t_{\text{sol}}$, which are almost identical across all measures, indicating that the choice of the robustness measure has only a minor influence on the overall solve time. Consequently, the total solve time grows proportionally with the planning time horizon $K$ due to the corresponding linear increase in required predicate function evaluations.

\begin{table}[htb]
    \vspace{-2mm}
    \centering
    \caption{Computational performance of robustness measures.}
    \label{tab:CalculationTimes}
    \begin{tabular}{
            l
            S[table-format=3.0]
            S[table-format=2.0]
            S[table-format=1.2]
            S[table-format=1.2]
            S[table-format=3.1]
            S[table-format=2.1]
        }
        \toprule
        \multirow{2}{*}[-0.6ex]{\shortstack[l]{\textbf{Robustness}                                                                     \\\textbf{Measure}}}        &
        \multicolumn{2}{c}{\text{$P_{\text{calls}}$}}                     &
        \multicolumn{2}{c}{\text{$t_{\text{rob}}\ [\si{\micro\second}]$}} &
        \multicolumn{2}{c}{\text{$t_{\text{sol}}\ [\si{\milli\second}]$}}                                                              \\
        \cmidrule(lr){2-3} \cmidrule(lr){4-5} \cmidrule(lr){6-7}
                                                                          & \multicolumn{1}{c}{Mean}
                                                                          & \multicolumn{1}{c}{Std.}
                                                                          & \multicolumn{1}{c}{Mean}
                                                                          & \multicolumn{1}{c}{Std.}
                                                                          & \multicolumn{1}{c}{Mean}
                                                                          & \multicolumn{1}{c}{Std.}                                   \\
        \midrule
        $\eta_{\text{space}}$                                             & 16                       & 0  & 0.87 & 0.03 & 355.4 & 38.2 \\
        $\eta_{\text{left-time}}$                                         & 100                      & 23 & 1.60 & 0.09 & 377.0 & 44.2 \\
        $\eta_{\text{right-time}}$                                        & 100                      & 23 & 1.48 & 0.07 & 369.1 & 43.5 \\
        $\eta_{\text{comb-time}}$                                         & 95                       & 21 & 1.62 & 0.08 & 373.1 & 43.2 \\
        \midrule
        $\eta_{\text{dur}}$                                               & 16                       & 0  & 0.91 & 0.03 & 357.4 & 40.0 \\
        $\eta_{\text{dur-sev}}$                                           & 16                       & 0  & 0.89 & 0.03 & 355.9 & 40.6 \\
        $\eta_{\text{smooth}}$                                            & 16                       & 0  & 1.02 & 0.03 & 359.9 & 41.4 \\
        $\eta_{\text{agm}}$                                               & 16                       & 0  & 0.90 & 0.03 & 356.7 & 36.6 \\
        $\eta_{\text{new}}$                                               & 16                       & 0  & 1.31 & 0.06 & 383.7 & 40.3 \\
        $\eta_{\text{pm}}$                                                & 16                       & 0  & 1.05 & 0.05 & 366.8 & 38.4 \\
        \midrule
        $\eta_{\text{space-left-time}}$                                   & 100                      & 24 & 1.64 & 0.08 & 377.0 & 43.9 \\
        \bottomrule
    \end{tabular}
    %\vspace{-3mm}
\end{table}

These results demonstrate that our space-left-time robustness remains computationally competitive, thus \textbf{supporting hypothesis~\ref{hyp:RobustnessMeasures}}. Note that a qualitative comparison of the resulting trajectories is omitted, as solution quality depends heavily on the parameterization and specifications used.

\section{Conclusions}\label{sec:Conclusions}

We present a comprehensive framework for minimum-violation motion planning that resolves conflicting, totally ordered specifications formalized in STL. By discretizing the underlying lexicographic optimization problem and transforming it into a single-objective problem via bit-shifting, we avoid the high computational complexity of the preemptive scheme. This enables the optimization of non-smooth STL robustness measures without restricting the STL fragment, thereby allowing the efficient solution of highly complex autonomous driving scenarios in a computationally tractable manner. To address the resulting non-smooth optimization problem, we extend the deterministic MPPI solver by eliminating the mandatory quadratic input cost and introducing additional adaptations that improve solution quality and computation time. Our extensive experimental evaluations demonstrate that for the proposed scalarization, even a limited number of evenly distributed violation intervals achieves low violation errors. Finally, our benchmark confirms that our proposed space-left-time robustness measure resolves the trajectory indistinguishability of standard measures while remaining computationally competitive. Future work will focus on automatically normalizing the predicate functions to eliminate manual tuning of the upper bound cost values. We also plan to validate the approach through real-world testing.

% ===================================
% Acknowledgement section
% ===================================

\section{Acknowledgment}

The authors gratefully thank Fabian Christ for his help with the code implementations.

% ===================================
% Appendix section
% ===================================

\appendices

\section{Minimum and Maximum Operators Definitions}\label{app:MinMaxDefinitions}

This appendix provides the definitions of the minimum operators $\amin_{\star}$ and the maximum operators $\amax_{\star}$. Let $\bkappa = [\kappa_1, \dots, \kappa_z] \in \overline{\mathbb{R}}^z$ with $z \ge 1$, and let $\nu_1, \nu_2, \nu_3, \nu_4, \nu_5 \in \mathbb{R}_{>0}$. For $\star \in \{\text{std}, \text{dur}\}$, we define the associated maximum operators by the duality rule $\amax_{\star}(\bkappa) := - \amin_{\star}(-\bkappa)$. For $\star \in \{\text{agm}, \text{new}, \text{pm}, \text{dur-sev}\}$, we likewise define $\amax_{\star}$ by duality, but additionally enforce the following conventions on $\overline{\mathbb{R}}$ for the corresponding minimum operators: if any component satisfies $\kappa_i = -\infty$, then $\amin_{\star}(\bkappa) := -\infty$; if all components satisfy $\kappa_i = +\infty$, then $\amin_{\star}(\bkappa) := +\infty$. The smooth operators are defined explicitly (i.e., $\amin_{\text{smooth}}$ and $\amax_{\text{smooth}}$ are not linked by duality); therefore, to ensure soundness on $\overline{\mathbb{R}}$ we also impose extended-real conventions: if any $\kappa_i = -\infty$, then $\amin_{\text{smooth}}(\bkappa) := -\infty$; if all $\kappa_i = +\infty$, then $\amin_{\text{smooth}}(\bkappa) := +\infty$. Similarly for the maximum operator, if any $\kappa_i = +\infty$, then $\amax_{\text{smooth}}(\bkappa) := +\infty$, and if all $\kappa_i = -\infty$, then $\amax_{\text{smooth}}(\bkappa) := -\infty$. In all remaining cases (i.e., after applying the above conventions), the operators are evaluated using standard extended-real arithmetic conventions (e.g., $e^{-\infty}=0$). The operators are defined as follows:
\begingroup
\allowdisplaybreaks
\begin{align*}
    \amin_{\text{std}}(\bkappa)     & := \kappa_{\text{min}},                                                                                                                                                                                                                  \\
    \amin_{\text{dur}}(\bkappa)     & := \begin{cases}
                                             \kappa_{\text{min}},                                 & \text{if } \kappa_{\text{min}} > 0, \\
                                             -\frac{1}{z} \sum_{i=1}^z \mathbf{1}_{\kappa_i < 0}, & \text{else},
                                         \end{cases}                                                                                                                      \\
    \amin_{\text{dur-sev}}(\bkappa) & := \begin{cases}
                                             \kappa_{\text{min}},                        & \text{if } \kappa_{\text{min}} > 0, \\
                                             \frac{1}{z} \sum_{i=1}^z \min(\kappa_i, 0), & \text{else},
                                         \end{cases}                                                                                                                             \\
    \amin_{\text{smooth}}(\bkappa)  & := -\frac{1}{\nu_1} \log \left( \sum_{i=1}^{z} e^{-\nu_1 \kappa_i} \right),                                                                                                                                                              \\
    \amax_{\text{smooth}}(\bkappa)  & := \frac{\sum_{i=1}^{z} \kappa_i e^{\nu_2 \kappa_i}}{\sum_{i=1}^{z} e^{\nu_2 \kappa_i}},                                                                                                                                                 \\
    \amin_{\text{agm}}(\bkappa)     & := \begin{cases}
                                             \frac{1}{z} \sum_{i=1}^z \min(\kappa_i, 0), & \text{if } \kappa_{\text{min}} \leq 0, \\[0.5em]
                                             \sqrt[z]{\prod_{i=1}^z (1+\kappa_i)} - 1,   & \text{else},                           \\
                                         \end{cases}                                                                                                                             \\
    \amin_{\text{new}}(\bkappa)     & := \begin{cases}
                                             \dfrac{\sum_{i=1}^z \kappa_{\text{min}} e^{\tilde\kappa_i} e^{\nu_3 \tilde\kappa_i}}{\sum_{i=1}^z e^{\nu_3 \tilde\kappa_i}}, & \text{if } \kappa_{\text{min}} < 0, \\[1em]
                                             \dfrac{\sum_{i=1}^z \kappa_i e^{-\nu_3 \tilde\kappa_i}}{\sum_{i=1}^z e^{-\nu_3 \tilde\kappa_i}},                             & \text{if } \kappa_{\text{min}} > 0, \\[1em]
                                             0,                                                                                                                           & \text{else},
                                         \end{cases} \\
    \amin_{\text{pm}}(\bkappa)      & := \begin{cases}
                                             \left(\frac{1}{z} \sum_{i=1}^z \kappa_i^{\nu_4} \right)^{\frac{1}{\nu_4}}, \text{if } \kappa_{\text{min}} > 0, \\[1em]
                                             -\left(\frac{1}{z} \sum_{i=1}^z (-\min(\kappa_i, 0))^{\nu_5} \right)^{\frac{1}{\nu_5}}, \text{else},
                                         \end{cases}
\end{align*}
\endgroup
with $\kappa_{\text{min}} := \min(\bkappa)$, $\tilde\kappa_i := (\kappa_i-\kappa_{\text{min}})/\kappa_{\text{min}}$ for $\kappa_{\text{min}} \neq 0$, and $\mathbf{1}_{\kappa_i < 0} := 1$ if $\kappa_i < 0$, else $0$.

\section{Proof of Order Representation}\label{app:ProofOrderRepresentation}

This appendix provides the proof for Lem.~\ref{lem:OrderRepresentation}, establishing the equivalence between the lexicographic order on the discrete cost vectors and the scalar cost values.

\begin{proof}
    If $\tilde\bc(\by) = \tilde\bc(\by')$, then $s(\by) = s(\by')$ holds trivially.
    Assume $\tilde\bc(\by) \neq \tilde\bc(\by')$ and $\tilde\bc(\by) \prelex \tilde\bc(\by')$. Let $i^* \in \{1,\dots, N\}$ be the smallest index at which the discrete cost vectors differ, i.e., $\tilde c_{\varphi_i}(\by) = \tilde c_{\varphi_i}(\by')$ for all $i < i^*$ and $\tilde c_{\varphi_{i^*}}(\by) < \tilde c_{\varphi_{i^*}}(\by')$ (cf. Def.~\ref{def:LexicographicallyOrderedSet}). This strict inequality implies $\tilde c_{\varphi_{i^*}}(\by') - \tilde c_{\varphi_{i^*}}(\by) \geq 1$.
    The difference in scalar cost is
    \begin{equation*}
        \begin{aligned}
            s(\by') - s(\by) =\; & (\tilde c_{\varphi_{i^*}}(\by') - \tilde c_{\varphi_{i^*}}(\by)) 2^{B_{i^*}}         \\
                                 & + \sum_{i=i^*+1}^N (\tilde c_{\varphi_i}(\by') - \tilde c_{\varphi_i}(\by)) 2^{B_i}.
        \end{aligned}
    \end{equation*}
    Since $0 \le \tilde c_{\varphi_i}(\cdot) \le m_{\varphi_i} < 2^{b_i}$, the magnitude of the residual sum is strictly bounded by $\sum_{i=i^*+1}^N (2^{b_i}-1) 2^{B_i} = 2^{B_{i^*}} - 1$. Consequently, the term $(\tilde c_{\varphi_{i^*}}(\by') - \tilde c_{\varphi_{i^*}}(\by)) 2^{B_{i^*}} \ge 2^{B_{i^*}}$ dominates the lower-priority terms, ensuring $s(\by') - s(\by) > 0$. The reverse direction follows from the uniqueness of the base representation, which is guaranteed by the disjoint bit ranges assigned to each cost component in~(\ref{eq:ScalarizationFunction}).
\end{proof}

% ===================================
% Bibliography
% ===================================

\bibliographystyle{IEEEtran}
\bibliography{2026_journal.bib}

@inproceedings{Althoff2017,
  author    = {Althoff, Matthias and Koschi, Markus and Manzinger, Stefanie},
  booktitle = {Proc. of the {IEEE} Intelligent Vehicles Symposium},
  title     = {{CommonRoad}: {Composable} Benchmarks for Motion Planning on Roads},
  year      = {2017},
  pages     = {719--726}
}

@inproceedings{Halder2025ICRA,
  author    = {Halder, Patrick and Homburger, Hannes and Kiltz, Lothar and Reuter, Johannes and Althoff, Matthias},
  booktitle = {Proc. of the {IEEE} Int. Conf. on Robotics and Automation},
  title     = {Trajectory Planning with Signal Temporal Logic Costs Using Deterministic Path Integral Optimization},
  year      = {2025},
  pages     = {4221--4228}
}

@article{Homburger2025,
  author  = {Homburger, Hannes and Messerer, Florian and Diehl, Moritz and Reuter, Johannes},
  journal = {IEEE Control Systems Letters},
  title   = {Optimality and Suboptimality of {MPPI} Control in Stochastic and Deterministic Settings},
  year    = {2025},
  pages   = {763--768},
  volume  = {9}
}

@article{Homburger2025Gauss,
  author  = {Homburger, Hannes and Baumg{\"a}rtner, Katrin and Diehl, Moritz and Reuter, Johannes},
  journal = {arXiv preprint arXiv:2512.04579},
  title   = {Gauss-Newton accelerated MPPI Control},
  year    = {2025},
  pages   = {1--6}
}

@article{Lin2025,
  author  = {Lin, Yuanfei and Xing, Zekun and Han, Xuyuan and Althoff, Matthias},
  journal = {IEEE Trans. on Robotics},
  title   = {Traffic-Rule-Compliant Trajectory Repair via Satisfiability Modulo Theories and Reachability Analysis},
  year    = {2025},
  pages   = {5932--5950},
  volume  = {41}
}

@inproceedings{Halder2025IV,
  author    = {Halder, Patrick and Althoff, Matthias},
  booktitle = {Proc. of the {IEEE} Intelligent Vehicles Symposium},
  title     = {Sampling-Based Motion Planning with Preordered Objectives},
  year      = {2025},
  pages     = {125--131}
}

@inproceedings{Halder2022,
  author    = {Halder, Patrick and Althoff, Matthias},
  booktitle = {Proc. of the {IEEE} Int. Conf. on Intelligent Transportation Systems},
  title     = {Minimum-Violation Velocity Planning with Temporal Logic Constraints},
  year      = {2022},
  pages     = {2520--2527}
}

@inproceedings{Halder2023,
  author    = {Halder, Patrick and Christ, Fabian and Althoff, Matthias},
  booktitle = {Proc. of the {IEEE} Int. Conf. on Intelligent Transportation Systems},
  title     = {Lexicographic Mixed-Integer Motion Planning with {STL} Constraints},
  year      = {2023},
  pages     = {1361--1367}
}

@inproceedings{Lercher2025,
  author    = {Lercher, Florian and Althoff, Matthias},
  booktitle = {Proc. of the {IEEE} Int. Conf. on Intelligent Transportation Systems},
  title     = {Efficient Driving Corridor Extraction for Autonomous Vehicles Using Best-First Reachability Analysis},
  year      = {2025},
  pages     = {407--414}
}

@book{Ehrgott2005,
  author    = {Ehrgott, Matthias},
  publisher = {Springer},
  title     = {Multicriteria Optimization},
  year      = {2005},
  volume    = {2}
}

@book{Feller1991,
  author    = {Feller, William},
  publisher = {John Wiley \& Sons},
  title     = {An introduction to probability theory and its applications},
  year      = {1991},
  volume    = {2}
}

@inproceedings{Homburger2023,
  author    = {Homburger, Hannes and Wirtensohn, Stefan and Reuter, Johannes},
  booktitle = {Proc. of the European Control Conf.},
  title     = {Efficient Nonlinear Model Predictive Path Integral Control for Stochastic Systems considering Input Constraints},
  year      = {2023},
  pages     = {1--6}
}

@book{Schroeder2016,
  author    = {Schr{\"o}der, Bernd Siegfried Walter},
  publisher = {Springer},
  title     = {Ordered Sets: {An} Introduction with Connections from Combinatorics to Topology},
  year      = {2016},
  volume    = {2}
}

@article{Beardon2002,
  author  = {Beardon, Alan F. and Candeal, Juan C. and Herden, Gerhard and Indur{\'a}in, Esteban and Mehta, Ghanshyam B.},
  journal = {Journal of Mathematical Economics},
  title   = {The Non-Existence of a Utility Function and the Structure of Non-Representable Preference Relations},
  year    = {2002},
  number  = {1},
  pages   = {17--38},
  volume  = {37}
}

@incollection{Debreu1983,
  author    = {Debreu, Gerard and Hildenbrand, Werner},
  booktitle = {Mathematical Economics: {T}wenty Papers of Gerard Debreu},
  title     = {Representation of a Preference Ordering by a Numerical Function},
  year      = {1983},
  pages     = {105--110}
}

@article{Mehdipour2023,
  author  = {Mehdipour, Noushin and Althoff, Matthias and Duintjer Tebbens, Radboud and Belta, Calin},
  journal = {Automatica},
  title   = {Formal Methods to Comply With Rules of the Road in Autonomous Driving: {State} of the Art and Grand Challenges},
  year    = {2023},
  pages   = {1--15},
  volume  = {152}
}

@article{Wetzlinger2026,
  author  = {Wetzlinger, Mark and Althoff, Matthias},
  journal = {IEEE Trans. on Automatic Control},
  title   = {Backward Reachability Analysis of Perturbed Continuous-Time Linear Systems Using Set Propagation},
  year    = {2026},
  number  = {1},
  pages   = {352--367},
  volume  = {71}
}

@inproceedings{Rizaldi2017,
  author    = {Rizaldi, Albert and Keinholz, Jonas and Huber, Monika and Feldle, Jochen and Immler, Fabian and Althoff, Matthias and Hilgendorf, Eric and Nipkow, Tobias},
  booktitle = {Proc. of the Int. Conf. on Integrated Formal Methods},
  title     = {Formalising and Monitoring Traffic Rules for Autonomous Vehicles in {Isabelle}/{HOL}},
  year      = {2017},
  pages     = {50--66}
}

@inproceedings{Esterle2020,
  author    = {Esterle, Klemens and Gressenbuch, Luis and Knoll, Alois},
  booktitle = {Proc. of the {IEEE} Connected and Automated Vehicles Symposium},
  title     = {Formalizing Traffic Rules for Machine Interpretability},
  year      = {2020},
  pages     = {1--7}
}

@inproceedings{Tumova2013,
  author    = {T{\r{u}}mov{\'a}, Jana and Hall, Gavin C. and Karaman, Sertac and Frazzoli, Emilio and Rus, Daniela},
  booktitle = {Proc. of the {ACM} Int. Conf. on Hybrid Systems: {Computation} and Control},
  title     = {Least-Violating Control Strategy Synthesis with Safety Rules},
  year      = {2013},
  pages     = {1--10}
}

@inproceedings{ReyesCastro2013,
  author    = {Reyes Castro, Luis I. and Chaudhari, Pratik and T{\r{u}}mov{\'a}, Jana and Karaman, Sertac and Frazzoli, Emilio and Rus, Daniela},
  booktitle = {Proc. of the {IEEE} Conf. on Decision and Control},
  title     = {Incremental Sampling-Based Algorithm for Minimum-Violation Motion Planning},
  year      = {2013},
  pages     = {3217--3224}
}

@inproceedings{Rong2020,
  author    = {Rong, Jikun and Luan, Nan},
  booktitle = {Proc. of the {IEEE} Int. Conf. on Mechatronics and Automation},
  title     = {Safe Reinforcement Learning with Policy-Guided Planning for Autonomous Driving},
  year      = {2020},
  pages     = {320--326}
}

@inproceedings{Schlueter2018,
  author    = {Schl{\"u}ter, Henning and Schillinger, Philipp and B{\"u}rger, Mathias},
  booktitle = {Proc. of the {IEEE} Conf. on Decision and Control},
  title     = {On the Design of Penalty Structures for Minimum-Violation {LTL} Motion Planning},
  year      = {2018},
  pages     = {4153--4158}
}

@inproceedings{Vasile2017,
  author    = {Vasile, Cristian-Ioan and T{\r{u}}mov{\'a}, Jana and Karaman, Sertac and Belta, Calin and Rus, Daniela},
  booktitle = {Proc. of the {IEEE} Int. Conf. on Robotics and Automation},
  title     = {Minimum-Violation {scLTL} Motion Planning for Mobility-on-Demand},
  year      = {2017},
  pages     = {1481--1488}
}

@inproceedings{Maierhofer2020,
  author    = {Maierhofer, Sebastian and Rettinger, Anna-Katharina and Mayer, Eva Charlotte and Althoff, Matthias},
  booktitle = {Proc. of the {IEEE} Intelligent Vehicles Symposium},
  title     = {Formalization of Interstate Traffic Rules in Temporal Logic},
  year      = {2020},
  pages     = {752--759}
}

@article{Shi2025,
  author  = {Shi, Ruolin and Wang, Xuesong},
  journal = {Expert Systems with Applications},
  title   = {Digitizing Traffic Rules to Guide Automated Vehicle Trajectory Planning},
  year    = {2025},
  pages   = {1--22},
  volume  = {272}
}

@inproceedings{Hekmatnejad2019,
  author    = {Hekmatnejad, Mohammad and Yaghoubi, Shakiba and Dokhanchi, Adel and Amor, Heni Ben and Shrivastava, Aviral and Karam, Lina and Fainekos, Georgios},
  booktitle = {Proc. of the {ACM}/{IEEE} Int. Conf. on Formal Methods and Models for System Design},
  title     = {Encoding and Monitoring Responsibility Sensitive Safety Rules for Automated Vehicles in Signal Temporal Logic},
  year      = {2019},
  pages     = {1--11}
}

@inproceedings{Arechiga2019,
  author    = {Ar{\'e}chiga, Nikos},
  booktitle = {Proc. of the {IEEE} Intelligent Vehicles Symposium},
  title     = {Specifying Safety of Autonomous Vehicles in Signal Temporal Logic},
  year      = {2019},
  pages     = {58--63}
}

@incollection{Bartocci2018,
  author    = {Bartocci, Ezio and Deshmukh, Jyotirmoy and Donz{\'e}, Alexandre and Fainekos, Georgios and Maler, Oded and Ni{\v{c}}kovi{\'c}, Dejan and Sankaranarayanan, Sriram},
  booktitle = {Lectures on Runtime Verification: {I}ntroductory and Advanced Topics},
  title     = {Specification-Based Monitoring of Cyber-Physical Systems: {A} Survey on Theory, Tools and Applications},
  year      = {2018},
  pages     = {135--175}
}

@inproceedings{Siu2023,
  author    = {Siu, Ho Chit and Leahy, Kevin and Mann, Makai},
  booktitle = {Proc. of the {IEEE}/{RSJ} Int. Conf. on Intelligent Robots and Systems},
  title     = {{STL}: {S}urprisingly Tricky Logic (for System Validation)},
  year      = {2023},
  pages     = {8613--8620}
}

@inproceedings{Hurley2024,
  author    = {Hurley, Isabelle and Paleja, Rohan and Suh, Ashley and Pe\~{n}a, Jaime D. and Siu, Ho Chit},
  booktitle = {Proc. of the Conf. on Neural Information Processing Systems},
  title     = {{STL}: {Still} Tricky Logic (for System Validation, Even When Showing Your Work)},
  year      = {2024},
  pages     = {119099--119122}
}

@inproceedings{Maierhofer2022,
  author    = {Maierhofer, Sebastian and Moosbrugger, Paul and Althoff, Matthias},
  booktitle = {Proc. of the {IEEE} Intelligent Vehicles Symposium},
  title     = {Formalization of Intersection Traffic Rules in Temporal Logic},
  year      = {2022},
  pages     = {1135--1144}
}

@inproceedings{Krasowski2021,
  author    = {Krasowski, Hanna and Althoff, Matthias},
  booktitle = {Proc. of the {IEEE} Intelligent Vehicles Symposium},
  title     = {Temporal Logic Formalization of Marine Traffic Rules},
  year      = {2021},
  pages     = {186--192}
}

@article{Uzun2024,
  author  = {Uzun, Samet and Elango, Purnanand and Garoche, Pierre-Loic and Acikmese, Behcet},
  journal = {arXiv preprint arXiv:2405.10996},
  title   = {Optimization with Temporal and Logical Specifications via Generalized Mean-Based Smooth Robustness Measures},
  year    = {2024},
  pages   = {1--16}
}

@article{Welikala2023,
  author  = {Welikala, Shirantha and Lin, Hai and Antsaklis, Panos J.},
  journal = {arXiv preprint arXiv:2305.09116},
  title   = {Smooth Robustness Measures for Symbolic Control via Signal Temporal Logic},
  year    = {2023},
  pages   = {1--16}
}

@article{Lin2020,
  author  = {Lin, Zhenyu and Baras, John S.},
  journal = {IFAC-PapersOnLine},
  title   = {Optimization-Based Motion Planning and Runtime Monitoring for Robotic Agent with Space and Time Tolerances},
  year    = {2020},
  number  = {2},
  pages   = {1874--1879},
  volume  = {53}
}

@inproceedings{Karlsson2021,
  author    = {Karlsson, Jesper and van Waveren, Sanne and Pek, Christian and Torre, Ilaria and Leite, Iolanda and T{\r{u}}mov{\'a}, Jana},
  booktitle = {Proc. of the {IEEE} Int. Conf. on Robotics and Automation},
  title     = {Encoding Human Driving Styles in Motion Planning for Autonomous Vehicles},
  year      = {2021},
  pages     = {1050--1056}
}

@article{Karlsson2020,
  author  = {Karlsson, Jesper and Barbosa, Fernando S. and T{\r{u}}mov{\'a}, Jana},
  journal = {IFAC-PapersOnLine},
  title   = {Sampling-Based Motion Planning with Temporal Logic Missions and Spatial Preferences},
  year    = {2020},
  number  = {2},
  pages   = {15537--15543},
  volume  = {53}
}

@article{Rodionova2023,
  author  = {Rodionova, Al{\"e}na and Lindemann, Lars and Morari, Manfred and Pappas, George J.},
  journal = {IEEE Control Systems Letters},
  title   = {Combined Left and Right Temporal Robustness for Control Under {STL} Specifications},
  year    = {2023},
  pages   = {619--624},
  volume  = {7}
}

@inproceedings{Verhagen2024,
  author    = {Verhagen, Joris and Lindemann, Lars and T{\r{u}}mov{\'a}, Jana},
  booktitle = {Proc. of the American Control Conf.},
  title     = {Temporally Robust Multi-Agent {STL} Motion Planning in Continuous Time},
  year      = {2024},
  pages     = {251--258}
}

@inproceedings{Yu2023,
  author    = {Yu, Xinyi and Yin, Xiang and Lindemann, Lars},
  booktitle = {Proc. of the {IEEE} Conf. on Decision and Control},
  title     = {Efficient {STL} Control Synthesis Under Asynchronous Temporal Robustness Constraints},
  year      = {2023},
  pages     = {6847--6854}
}

@inproceedings{Lindemann2022,
  author    = {Lindemann, Lars and Rodionova, Alena and Pappas, George},
  booktitle = {Proc. of the {ACM} Int. Conf. on Hybrid Systems: {C}omputation and Control},
  title     = {Temporal Robustness of Stochastic Signals},
  year      = {2022},
  pages     = {1--11}
}

@article{Mehdipour2025,
  author  = {Mehdipour, Noushin and Vasile, Cristian-Ioan and Belta, Calin},
  journal = {IEEE Trans. on Automatic Control},
  title   = {Generalized Mean Robustness for Signal Temporal Logic},
  year    = {2025},
  number  = {3},
  pages   = {1949--1956},
  volume  = {70}
}

@article{Lin2023,
  author  = {Lin, Yuanfei and Li, Haoxuan and Althoff, Matthias},
  journal = {IEEE Robotics and Automation Letters},
  title   = {Model Predictive Robustness of Signal Temporal Logic Predicates},
  year    = {2023},
  number  = {12},
  pages   = {8050--8057},
  volume  = {8}
}

@inproceedings{Lin2023a,
  author    = {Lin, Yuanfei and Althoff, Matthias},
  booktitle = {Proc. of the {IEEE} Intelligent Vehicles Symposium},
  title     = {{CommonRoad-CriMe}: {A} Toolbox for Criticality Measures of Autonomous Vehicles},
  year      = {2023},
  pages     = {1-8}
}

@inproceedings{Zhao2022,
  author    = {Zhao, Deng and Zhou, Zhangbing and Cai, Zhipeng and Long, Teng and Yangui, Sami and Xue, Xiao},
  booktitle = {Proc. of the {IEEE} Int. Conf. on Web Services},
  title     = {{ASTL}: {A}ccumulative Signal Temporal Logic for {IoT} Service Monitoring},
  year      = {2022},
  pages     = {256--265}
}

@inproceedings{Haghighi2019,
  author    = {Haghighi, Iman and Mehdipour, Noushin and Bartocci, Ezio and Belta, Calin},
  booktitle = {Proc. of the {IEEE} Conf. on Decision and Control},
  title     = {Control from Signal Temporal Logic Specifications with Smooth Cumulative Quantitative Semantics},
  year      = {2019},
  pages     = {4361--4366}
}

@article{Kurtz2021,
  author  = {Kurtz, Vince and Lin, Hai},
  journal = {IEEE Control Systems Letters},
  title   = {Trajectory Optimization for High-Dimensional Nonlinear Systems Under {STL} Specifications},
  year    = {2021},
  number  = {4},
  pages   = {1429--1434},
  volume  = {5}
}

@article{Gilpin2021,
  author  = {Gilpin, Yann and Kurtz, Vince and Lin, Hai},
  journal = {IEEE Control Systems Letters},
  title   = {A Smooth Robustness Measure of Signal Temporal Logic for Symbolic Control},
  year    = {2021},
  number  = {1},
  pages   = {241--246},
  volume  = {5}
}

@inproceedings{Dawson2022,
  author    = {Dawson, Charles and Fan, Chuchu},
  booktitle = {Proc. of the {IEEE}/{RSJ} Int. Conf. on Intelligent Robots and Systems},
  title     = {Robust Counterexample-Guided Optimization for Planning from Differentiable Temporal Logic},
  year      = {2022},
  pages     = {7205--7212}
}

@inproceedings{Pant2018,
  author    = {Pant, Yash Vardhan and Abbas, Houssam and Quaye, Rhudii A. and Mangharam, Rahul},
  booktitle = {Proc. of the {ACM}/{IEEE} Int. Conf. on Cyber-Physical Systems},
  title     = {Fly-by-Logic: {C}ontrol of Multi-Drone Fleets with Temporal Logic Objectives},
  year      = {2018},
  pages     = {186--197}
}

@inproceedings{Pant2017,
  author    = {Pant, Yash Vardhan and Abbas, Houssam and Mangharam, Rahul},
  booktitle = {Proc. of the {IEEE} Conf. on Control Technology and Applications},
  title     = {Smooth Operator: {Control} Using the Smooth Robustness of Temporal Logic},
  year      = {2017},
  pages     = {1235--1240}
}

@article{Baharisangari2022,
  author  = {Baharisangari, Nasim and Hirota, Kazuma and Yan, Ruixuan and Julius, Agung and Xu, Zhe},
  journal = {IEEE Control Systems Letters},
  title   = {Weighted Graph-Based Signal Temporal Logic Inference Using Neural Networks},
  year    = {2022},
  pages   = {2096--2101},
  volume  = {6}
}

@inproceedings{Cardona2023,
  author    = {Cardona, Gustavo A. and Vasile, Cristian-Ioan},
  booktitle = {Proc. of the European Control Conf.},
  title     = {Preferences on Partial Satisfaction using Weighted Signal Temporal Logic Specifications},
  year      = {2023},
  pages     = {1--6}
}

@article{Mehdipour2021,
  author  = {Mehdipour, Noushin and Vasile, Cristian-Ioan and Belta, Calin},
  journal = {IEEE Control Systems Letters},
  title   = {Specifying User Preferences Using Weighted Signal Temporal Logic},
  year    = {2021},
  number  = {6},
  pages   = {2006--2011},
  volume  = {5}
}

@inproceedings{Varnai2020,
  author    = {Varnai, Peter and Dimarogonas, Dimos V.},
  booktitle = {Proc. of the American Control Conf.},
  title     = {On Robustness Metrics for Learning {STL} Tasks},
  year      = {2020},
  pages     = {5394--5399}
}

@inproceedings{Mehdipour2019,
  author    = {Mehdipour, Noushin and Vasile, Cristian-Ioan and Belta, Calin},
  booktitle = {Proc. of the {IEEE} Conf. on Decision and Control},
  title     = {Average-Based Robustness for Continuous-Time Signal Temporal Logic},
  year      = {2019},
  pages     = {5312--5317}
}

@inproceedings{Mehdipour2019a,
  author    = {Mehdipour, Noushin and Vasile, Cristian-Ioan and Belta, Calin},
  booktitle = {Proc. of the American Control Conf.},
  title     = {Arithmetic-Geometric Mean Robustness for Control from Signal Temporal Logic Specifications},
  year      = {2019},
  pages     = {1690--1695}
}

@inproceedings{Lindemann2017,
  author    = {Lindemann, Lars and Dimarogonas, Dimos V.},
  booktitle = {Proc. of the American Control Conf.},
  title     = {Robust Motion Planning Employing Signal Temporal Logic},
  year      = {2017},
  pages     = {2950--2955}
}

@inproceedings{Donze2010,
  author    = {Donz{\'e}, Alexandre and Maler, Oded},
  booktitle = {Proc. of the Int. Conf. on Formal Modeling and Analysis of Timed Systems},
  title     = {Robust Satisfaction of Temporal Logic over Real-Valued Signals},
  year      = {2010},
  pages     = {92--106}
}

@article{Rodionova2022,
  author  = {Rodionova, Alena and Lindemann, Lars and Morari, Manfred and Pappas, George},
  journal = {ACM Trans. on Embedded Computing Systems},
  title   = {Temporal Robustness of Temporal Logic Specifications: {Analysis} and Control Design},
  year    = {2022},
  number  = {1},
  pages   = {1--44},
  volume  = {22}
}

@article{Leung2023,
  author  = {Leung, Karen and Ar{\'e}chiga, Nikos and Pavone, Marco},
  journal = {The Int. Journal of Robotics Research},
  title   = {Backpropagation Through Signal Temporal Logic Specifications: {Infusing} Logical Structure Into Gradient-Based Methods},
  year    = {2023},
  number  = {6},
  pages   = {356--370},
  volume  = {42}
}

@article{Marchesini2025,
  author      = {Marchesini, Gregorio and Liu, Siyuan and Lindemann, Lars and Dimarogonas, Dimos V.},
  journal     = {{IEEE} Trans. on Automatic Control},
  title       = {Sampling-Based Planning Under {STL} Specifications: {A} Forward Invariance Approach},
  year        = {2026},
  pages       = {1-16},
  earlyaccess = {true}
}

@article{Farahani2015,
  author  = {Farahani, Samira S. and Raman, Vasumathi and Murray, Richard M.},
  journal = {IFAC-PapersOnLine},
  title   = {Robust Model Predictive Control for Signal Temporal Logic Synthesis},
  year    = {2015},
  number  = {27},
  pages   = {323--328},
  volume  = {48}
}

@inproceedings{Sadraddini2015,
  author    = {Sadraddini, Sadra and Belta, Calin},
  booktitle = {Proc. of the Allerton Conf. on Communication, Control, and Computing},
  title     = {Robust Temporal Logic Model Predictive Control},
  year      = {2015},
  pages     = {772--779}
}

@inproceedings{Kapoor2024,
  author    = {Kapoor, Parv and Kang, Eunsuk and Meira-G{\'o}es, R{\^o}mulo},
  booktitle = {Proc. of the {NASA} Formal Methods Symposium},
  title     = {Safe Planning Through Incremental Decomposition of Signal Temporal Logic Specifications},
  year      = {2024},
  pages     = {377--396}
}

@article{Barbosa2019,
  author  = {Barbosa, Fernando S. and Duberg, Daniel and Jensfelt, Patric and T{\r{u}}mov{\'a}, Jana},
  journal = {IEEE Robotics and Automation Letters},
  title   = {Guiding Autonomous Exploration With Signal Temporal Logic},
  year    = {2019},
  number  = {4},
  pages   = {3332--3339},
  volume  = {4}
}

@article{Althoff2025,
  author  = {Althoff, Matthias and Maierhofer, Sebastian and W{\"u}rsching, Gerald and Lin, Yuanfei and Lercher, Florian and Stolz, Roland},
  journal = {Proc. of the IEEE},
  title   = {No More Traffic Tickets: {A} Tutorial to Ensure Traffic-Rule Compliance of Automated Vehicles},
  year    = {2025},
  pages   = {1--30}
}

@inproceedings{Leung2022,
  author    = {Leung, Karen and Pavone, Marco},
  booktitle = {Proc. of the American Control Conf.},
  title     = {Semi-Supervised Trajectory-Feedback Controller Synthesis for Signal Temporal Logic Specifications},
  year      = {2022},
  pages     = {178--185}
}

@article{Liu2022,
  author  = {Liu, Wenliang and Mehdipour, Noushin and Belta, Calin},
  journal = {IEEE Control Systems Letters},
  title   = {Recurrent Neural Network Controllers for Signal Temporal Logic Specifications Subject to Safety Constraints},
  year    = {2022},
  pages   = {91--96},
  volume  = {6}
}

@inproceedings{Saxena2023,
  author    = {Saxena, Naman and Sandeep, Gorantla and Jagtap, Pushpak},
  booktitle = {Proc. of the Indian Control Conf.},
  title     = {Reinforcement Learning for Signal Temporal Logic using Funnel-Based Approach},
  year      = {2023},
  pages     = {1--6}
}

@inproceedings{Singh2023,
  author    = {Singh, Nikhil Kumar and Saha, Indranil},
  booktitle = {Proc. of the {AAAI} Conf. on Artificial Intelligence},
  title     = {{STL}-Based Synthesis of Feedback Controllers Using Reinforcement Learning},
  year      = {2023},
  pages     = {15118--15126}
}

@article{Meng2023,
  author  = {Meng, Yue and Fan, Chuchu},
  journal = {IEEE Robotics and Automation Letters},
  title   = {Signal Temporal Logic Neural Predictive Control},
  year    = {2023},
  number  = {11},
  pages   = {7719--7726},
  volume  = {8}
}

@inproceedings{Sewlia2023,
  author    = {Sewlia, Mayank and Verginis, Christos K. and Dimarogonas, Dimos V.},
  booktitle = {Proc. of the American Control Conf.},
  title     = {Cooperative Sampling-Based Motion Planning under Signal Temporal Logic Specifications},
  year      = {2023},
  pages     = {2697--2702}
}

@inproceedings{Linard2023,
  author    = {Linard, Alexis and Torre, Ilaria and Bartoli, Ermanno and Sleat, Alex and Leite, Iolanda and T{\r{u}}mov{\'a}, Jana},
  booktitle = {Proc. of the {IEEE}/{RSJ} Int. Conf. on Intelligent Robots and Systems},
  title     = {Real-Time {RRT}* with Signal Temporal Logic Preferences},
  year      = {2023},
  pages     = {8621--8627}
}

@inproceedings{Barbosa2021,
  author    = {Barbosa, Fernando S. and Karlsson, Jesper and Tajvar, Pouria and T{\r{u}}mov{\'a}, Jana},
  booktitle = {Proc. of the Int. Conf. on Formal Modeling and Analysis of Timed Systems},
  title     = {Formal Methods for Robot Motion Planning with Time and Space Constraints},
  year      = {2021},
  pages     = {1--14}
}

@article{Yu2024,
  author  = {Yu, Pian and Tan, Xiao and Dimarogonas, Dimos V.},
  journal = {IEEE Trans. on Robotics},
  title   = {Continuous-Time Control Synthesis Under Nested Signal Temporal Logic Specifications},
  year    = {2024},
  pages   = {2272--2286},
  volume  = {40}
}

@inproceedings{Ruo2024,
  author    = {Ruo, Andrea and Sabattini, Lorenzo and Villani, Valeria},
  booktitle = {Proc. of the European Control Conf.},
  title     = {{CBF}-Based Motion Planning for Socially Responsible Robot Navigation Guaranteeing {STL} Specification*},
  year      = {2024},
  pages     = {122--127}
}

@article{Huang2024,
  author  = {Huang, Zhiyuan and Lan, Weiyao and Yu, Xiao},
  journal = {IEEE Trans. on Intelligent Vehicles},
  title   = {A Formal Control Framework of Autonomous Vehicle for Signal Temporal Logic Tasks and Obstacle Avoidance},
  year    = {2024},
  number  = {1},
  pages   = {1930--1940},
  volume  = {9}
}

@article{Lindemann2020,
  author  = {Lindemann, Lars and Dimarogonas, Dimos V.},
  journal = {IEEE Trans. on Control of Network Systems},
  title   = {Barrier Function Based Collaborative Control of Multiple Robots Under Signal Temporal Logic Tasks},
  year    = {2020},
  number  = {4},
  pages   = {1916--1928},
  volume  = {7}
}

@inproceedings{Xiao2021,
  author    = {Xiao, Wei and Mehdipour, Noushin and Collin, Anne and Bin-Nun, Amitai Y. and Frazzoli, Emilio and Tebbens, Radboud Duintjer and Belta, Calin},
  booktitle = {Proc. of the {ACM}/{IEEE} Int. Conf. on Cyber-Physical Systems},
  title     = {Rule-based optimal control for autonomous driving},
  year      = {2021},
  pages     = {143--154}
}

@inproceedings{Charitidou2021,
  author    = {Charitidou, Maria and Dimarogonas, Dimos V.},
  booktitle = {Proc. of the European Control Conf.},
  title     = {Barrier Function-based Model Predictive Control under Signal Temporal Logic Specifications},
  year      = {2021},
  pages     = {734--739}
}

@inproceedings{Lindemann2020a,
  author    = {Lindemann, Lars and Pappas, George J. and Dimarogonas, Dimos V.},
  booktitle = {Proc. of the {IEEE} Conf. on Decision and Control},
  title     = {Control Barrier Functions for Nonholonomic Systems under Risk Signal Temporal Logic Specifications},
  year      = {2020},
  pages     = {1422--1428}
}

@article{Lindemann2019,
  author  = {Lindemann, Lars and Dimarogonas, Dimos V.},
  journal = {IEEE Control Systems Letters},
  title   = {Control Barrier Functions for Signal Temporal Logic Tasks},
  year    = {2019},
  number  = {1},
  pages   = {96--101},
  volume  = {3}
}

@article{Takayama2025,
  author  = {Takayama, Yoshinari and Hashimoto, Kazumune and Ohtsuka, Toshiyuki},
  journal = {IEEE Trans. on Automatic Control},
  title   = {{STLCCP}: {E}fficient Convex Optimization-Based Framework for Signal Temporal Logic Specifications},
  year    = {2025},
  number  = {9},
  pages   = {6064--6079},
  volume  = {70}
}

@article{Claudet2024,
  author  = {Claudet, Thomas and Martire, Davide and Losa, Damiana and Sanfedino, Francesco and Alazard, Daniel},
  journal = {IFAC Journal of Systems and Control},
  title   = {A novel graph-based theory for convexification of mission-planning constraints and generative pre-trained trajectory optimization},
  year    = {2024},
  pages   = {1--13},
  volume  = {30}
}

@inproceedings{Mao2022,
  author    = {Mao, Yuanqi and Acikmese, Behcet and Garoche, Pierre-Loic and Chapoutot, Alexandre},
  booktitle = {Proc. of the {ACM} Int. Conf. on Hybrid Systems: {C}omputation and Control},
  title     = {Successive Convexification for Optimal Control with Signal Temporal Logic Specifications},
  year      = {2022},
  pages     = {1--7}
}

@article{Buyukkocak2025,
  author  = {Buyukkocak, Ali Tevfik and Aksaray, Derya},
  journal = {IEEE Robotics and Automation Letters},
  title   = {Resilient Online Planning for Mobile Robots With Minimal Relaxation of Signal Temporal Logic Specifications},
  year    = {2025},
  number  = {6},
  pages   = {5935--5942},
  volume  = {10}
}

@article{Belta2019,
  author  = {Belta, Calin and Sadraddini, Sadra},
  journal = {Annual Review of Control, Robotics, and Autonomous Systems},
  title   = {Formal Methods for Control Synthesis: {A}n Optimization Perspective},
  year    = {2019},
  number  = {1},
  pages   = {115--140},
  volume  = {2}
}

@inproceedings{Rodionova2021,
  author    = {Rodionova, Al{\"e}na and Lindemann, Lars and Morari, Manfred and Pappas, George J.},
  booktitle = {Proc. of the {IEEE} Conf. on Decision and Control},
  title     = {Time-Robust Control for {STL} Specifications},
  year      = {2021},
  pages     = {572--579}
}

@article{Aasi2025,
  author  = {Aasi, Erfan and Cai, Mingyu and Vasile, Cristian-Ioan and Belta, Calin},
  journal = {IEEE Trans. on Intelligent Transportation Systems},
  title   = {A Two-Level Control Algorithm for Autonomous Driving in Urban Environments},
  year    = {2025},
  number  = {1},
  pages   = {410--424},
  volume  = {26}
}

@article{Yuan2024,
  author  = {Yuan, Yating and Quartz, Thanin and Liu, Jun},
  journal = {IEEE Control Systems Letters},
  title   = {Signal Temporal Logic Planning With Time-Varying Robustness},
  year    = {2024},
  pages   = {3015--3020},
  volume  = {8}
}

@article{Kurtz2022,
  author  = {Kurtz, Vincent and Lin, Hai},
  journal = {IEEE Control Systems Letters},
  title   = {Mixed-Integer Programming for Signal Temporal Logic With Fewer Binary Variables},
  year    = {2022},
  pages   = {2635--2640},
  volume  = {6}
}

@inproceedings{Aasi2021,
  author    = {Aasi, Erfan and Vasile, Cristian Ioan and Belta, Calin},
  booktitle = {Proc. of the American Control Conf.},
  title     = {A Control Architecture for Provably-Correct Autonomous Driving},
  year      = {2021},
  pages     = {2913--2918}
}

@inproceedings{Sahin2020,
  author    = {Sahin, Yunus Emre and Quirynen, Rien and Cairano, Stefano Di},
  booktitle = {Proc. of the American Control Conf.},
  title     = {Autonomous Vehicle Decision-Making and Monitoring based on Signal Temporal Logic and Mixed-Integer Programming},
  year      = {2020},
  pages     = {454--459}
}

@article{Sadraddini2019,
  author  = {Sadraddini, Sadra and Belta, Calin},
  journal = {IEEE Trans. on Automatic Control},
  title   = {Formal Synthesis of Control Strategies for Positive Monotone Systems},
  year    = {2019},
  number  = {2},
  pages   = {480--495},
  volume  = {64}
}

@inproceedings{Gambo2021,
  author    = {Gambo, Ishaya Peni and Taveter, Kuldar},
  booktitle = {Proc. of the Int. Conf. on Evaluation of Novel Approaches to Software Engineering},
  title     = {Identifying and Resolving Conflicts in Requirements by Stakeholders: {A} Clustering Approach},
  year      = {2021},
  pages     = {158--169}
}

@article{Marler2004,
  author  = {Marler, R. Timothy and Arora, Jasbir S.},
  journal = {Structural and Multidisciplinary Optimization},
  title   = {Survey of Multi-Objective Optimization Methods for Engineering},
  year    = {2004},
  pages   = {369--395},
  volume  = {26}
}

@article{Das1997,
  author  = {Das, Indraneel and Dennis, John E},
  journal = {Structural Optimization},
  title   = {A Closer Look at Drawbacks of Minimizing Weighted Sums of Objectives for Pareto Set Generation in Multicriteria Optimization Problems},
  year    = {1997},
  number  = {1},
  pages   = {63--69},
  volume  = {14}
}

@book{Branke2008,
  author    = {Branke, J{\"u}rgen},
  publisher = {Springer},
  title     = {Multiobjective Optimization: {Interactive} and Evolutionary Approaches},
  year      = {2008},
  volume    = {1}
}

@article{Penlington2024,
  author  = {Penlington, Matteo and Zanardi, Alessandro and Frazzoli, Emilio},
  journal = {arXiv preprint arXiv:2409.11199},
  title   = {Optimization of Rulebooks via Asymptotically Representing Lexicographic Hierarchies for Autonomous Vehicles},
  year    = {2024},
  pages   = {1--7}
}

@article{Jordana2025,
  author  = {Jordana, Armand and Zhang, Jianghan and Amigo, Joseph and Righetti, Ludovic},
  journal = {arXiv preprint arXiv:2506.22087},
  title   = {An Introduction to Zero-order Optimization Techniques for Robotics},
  year    = {2025},
  pages   = {1--16}
}

@inproceedings{Molins2025,
  author    = {Molins, Pau de las Heras and Roy-Almonacid, Eric and Lee, Dong Ho and Peters, Lasse and Fridovich-Keil, David and Bakirtzis, Georgios},
  booktitle = {Proc. of the {IEEE} Int. Conf. on Intelligent Transportation Systems},
  title     = {Approximate Solutions to Games of Ordered Preference},
  year      = {2025},
  pages     = {3378--3385}
}

@inproceedings{Censi2019,
  author    = {Censi, Andrea and Slutsky, Konstantin and Wongpiromsarn, Tichakorn and Yershov, Dmitry and Pendleton, Scott and Fu, James and Frazzoli, Emilio},
  booktitle = {Proc. of the {IEEE} Int. Conf. on Robotics and Automation},
  title     = {Liability, Ethics, and Culture-Aware Behavior Specification Using Rulebooks},
  year      = {2019},
  pages     = {8536--8542}
}

@article{Lai2022,
  author  = {Lai, Leonardo and Fiaschi, Lorenzo and Cococcioni, Marco and Deb, Kalyanmoy},
  journal = {Natural Computing},
  title   = {Pure and Mixed Lexicographic-Paretian Many-Objective Optimization: {State} of the Art},
  year    = {2022},
  number  = {2},
  pages   = {227--242},
  volume  = {22}
}

@inproceedings{Helou2021,
  author    = {Helou, Bassam and Dusi, Aditya and Collin, Anne and Mehdipour, Noushin and Chen, Zhiliang and Lizarazo, Cristhian and Belta, Calin and Wongpiromsarn, Tichakorn and Tebbens, Radboud Duintjer and Beijbom, Oscar},
  booktitle = {Proc. of the {IEEE}/{RSJ} Int. Conf. on Intelligent Robots and Systems},
  title     = {The Reasonable Crowd: {Towards} Evidence-Based and Interpretable Models of Driving Behavior},
  year      = {2021},
  pages     = {6708--6715}
}

@article{Sherali1983,
  author  = {Sherali, Hanif D and Soyster, Allen L},
  journal = {Journal of Optimization Theory and Applications},
  title   = {Preemptive and Nonpreemptive Multi-Objective Programming: {Relationship} and Counterexamples},
  year    = {1983},
  number  = {2},
  pages   = {173--186},
  volume  = {39}
}

@article{Zarepisheh2011,
  author  = {Zarepisheh, Masoud and Khorram, Esmaile},
  journal = {Mathematical Methods of Operations Research},
  title   = {On the Transformation of Lexicographic Nonlinear Multiobjective Programs to Single Objective Programs},
  year    = {2011},
  number  = {2},
  pages   = {217--231},
  volume  = {74}
}

@article{Pek2021,
  author  = {Pek, Christian and Althoff, Matthias},
  journal = {IEEE Trans. on Robotics},
  title   = {Fail-Safe Motion Planning for Online Verification of Autonomous Vehicles Using Convex Optimization},
  year    = {2021},
  number  = {3},
  pages   = {798--814},
  volume  = {37}
}

@inproceedings{Baldini2024,
  author    = {Baldini, Francesca and Tariq, Faizan M. and Bae, Sangjae and Isele, David},
  booktitle = {Proc. of the {IEEE} Int. Conf. on Intelligent Transportation Systems},
  title     = {Don't Get Stuck: {A} Deadlock Recovery Approach},
  year      = {2024},
  pages     = {3688--3695}
}

@inproceedings{Wongpiromsarn2021,
  author    = {Wongpiromsarn, Tichakorn and Slutsky, Konstantin and Frazzoli, Emilio and Topcu, Ufuk},
  booktitle = {Proc. of the American Control Conf.},
  title     = {Minimum-Violation Planning for Autonomous Systems: {T}heoretical and Practical Considerations},
  year      = {2021},
  pages     = {4866--4872}
}

@inproceedings{Karlsson2018,
  author    = {Karlsson, Jesper and Vasile, Cristian-Ioan and T{\r{u}}mov{\'a}, Jana and Karaman, Sertac and Rus, Daniela},
  booktitle = {Proc. of the {IEEE} Int. Conf. on Robotics and Automation},
  title     = {Multi-Vehicle Motion Planning for Social Optimal Mobility-on-Demand},
  year      = {2018},
  pages     = {7298--7305}
}

@inproceedings{Lahijanian2015,
  author    = {Lahijanian, Morteza and Almagor, Shaull and Fried, Dror and Kavraki, Lydia and Vardi, Moshe},
  booktitle = {Proc. of the {AAAI} Conf. on Artificial Intelligence},
  title     = {This Time the Robot Settles for a Cost: {A} Quantitative Approach to Temporal Logic Planning with Partial Satisfaction},
  year      = {2015},
  pages     = {1--8}
}

@inproceedings{Tumova2013a,
  author    = {T{\r{u}}mov{\'a}, Jana and Reyes Castro, Luis I. and Karaman, Sertac and Frazzoli, Emilio and Rus, Daniela},
  booktitle = {Proc. of the American Control Conf.},
  title     = {Minimum-Violation {LTL} Planning with Conflicting Specifications},
  year      = {2013},
  pages     = {200--205}
}

@inproceedings{Veer2023,
  author    = {Veer, Sushant and Leung, Karen and Cosner, Ryan K. and Chen, Yuxiao and Karkus, Peter and Pavone, Marco},
  booktitle = {Proc. of the {IEEE} Int. Conf. on Robotics and Automation},
  title     = {Receding Horizon Planning with Rule Hierarchies for Autonomous Vehicles},
  year      = {2023},
  pages     = {1507--1513}
}

@inproceedings{Rahmani2020,
  author    = {Rahmani, Hazhar and O'Kane, Jason M.},
  booktitle = {Proc. of the {IEEE}/{RSJ} Int. Conf. on Intelligent Robots and Systems},
  title     = {What to Do When You Can't Do It All: {T}emporal Logic Planning with Soft Temporal Logic Constraints},
  year      = {2020},
  pages     = {6619--6626}
}

@article{Sathya2021,
  author  = {Sathya, Ajay Suresha and Pipeleers, Goele and Decr{\'e}, Wilm and Swevers, Jan},
  journal = {IEEE Robotics and Automation Letters},
  title   = {A Weighted Method for Fast Resolution of Strictly Hierarchical Robot Task Specifications Using Exact Penalty Functions},
  year    = {2021},
  number  = {2},
  pages   = {3057--3064},
  volume  = {6}
}

@article{MaristanydelasCasas2023,
  author  = {Maristany de las Casas, Pedro and Kraus, Luitgard and Sede{\~n}o-Noda, Antonio and Bornd{\"o}rfer, Ralf},
  journal = {Networks},
  title   = {Targeted Multiobjective {Dijkstra} Algorithm},
  year    = {2023},
  number  = {3},
  pages   = {277--298},
  volume  = {82}
}

@article{Fiaschi2022,
  author  = {Fiaschi, Lorenzo and Cococcioni, Marco},
  journal = {Mathematics},
  title   = {A Non-Archimedean Interior Point Method and Its Application to the Lexicographic Multi-Objective Quadratic Programming},
  year    = {2022},
  number  = {23},
  pages   = {1--34},
  volume  = {10}
}

@article{Zanardi2021,
  author  = {Zanardi, Alessandro and Mion, Enrico and Bruschetta, Mattia and Bolognani, Saverio and Censi, Andrea and Frazzoli, Emilio},
  journal = {IEEE Robotics and Automation Letters},
  title   = {Urban Driving Games With Lexicographic Preferences and Socially Efficient {Nash} Equilibria},
  year    = {2021},
  number  = {3},
  pages   = {4978--4985},
  volume  = {6}
}

@article{Hussain2021,
  author  = {Hussain, Akhtar and Kim, Hak-Man},
  journal = {IEEE Trans. on Transportation Electrification},
  title   = {{EV} Prioritization and Power Allocation During Outages: {A} Lexicographic Method-Based Multiobjective Optimization Approach},
  year    = {2021},
  number  = {4},
  pages   = {2474--2487},
  volume  = {7}
}

@article{Wang2020,
  author  = {Wang, Hong and Huang, Yanjun and Khajepour, Amir and Cao, Dongpu and Lv, Chen},
  journal = {IEEE Trans. on Vehicular Technology},
  title   = {Ethical Decision-Making Platform in Autonomous Vehicles With Lexicographic Optimization Based Model Predictive Controller},
  year    = {2020},
  number  = {8},
  pages   = {8164--8175},
  volume  = {69}
}

@article{Cococcioni2018,
  author  = {Cococcioni, Marco and Pappalardo, Massimo and Sergeyev, Yaroslav D.},
  journal = {Applied Mathematics and Computation},
  title   = {Lexicographic Multi-Objective Linear Programming Using Grossone Methodology: {Theory} and Algorithm},
  year    = {2018},
  pages   = {298--311},
  volume  = {318}
}

@inproceedings{Shan2020,
  author    = {Shan, Tixiao and Wang, Wei and Englot, Brendan and Ratti, Carlo and Rus, Daniela},
  booktitle = {Proc. of the {IEEE} Conf. on Decision and Control},
  title     = {A Receding Horizon Multi-Objective Planner for Autonomous Surface Vehicles in Urban Waterways},
  year      = {2020},
  pages     = {4085--4092}
}

@article{Khosravani2018,
  author  = {Khosravani, Saeid and Jalali, Milad and Khajepour, Amir and Kasaiezadeh, Alireza and Chen, Shih-Ken and Litkouhi, Bakhtiar},
  journal = {IEEE Trans. on Industrial Electronics},
  title   = {Application of Lexicographic Optimization Method to Integrated Vehicle Control Systems},
  year    = {2018},
  number  = {12},
  pages   = {9677--9686},
  volume  = {65}
}

@inproceedings{Schnurr2018,
  author    = {Schnurr, Christoph and Hohmann, S{\"o}ren and Kolb, Johannes},
  booktitle = {Proc. of the {IEEE} Conf. on Control Technology and Applications},
  title     = {Novel Lexicographic {MPC} for Loss Optimized Torque Control of Nonlinear {PMSM}},
  year      = {2018},
  pages     = {690--697}
}

@inproceedings{Abouelazm2024,
  author    = {Abouelazm, Ahmed and Michel, Jonas and Z{\"o}llner, J. Marius},
  booktitle = {Proc. of the {IEEE} Intelligent Vehicles Symposium},
  title     = {A Review of Reward Functions for Reinforcement Learning in the Context of Autonomous Driving},
  year      = {2024},
  pages     = {156--163}
}

@article{Wilde2024,
  author  = {Wilde, Nils and Smith, Stephen L. and Alonso-Mora, Javier},
  journal = {IEEE Robotics and Automation Letters},
  title   = {Scalarizing Multi-Objective Robot Planning Problems Using Weighted Maximization},
  year    = {2024},
  number  = {3},
  pages   = {2503--2510},
  volume  = {9}
}

@article{Williams2018,
  author  = {Williams, Grady and Drews, Paul and Goldfain, Brian and Rehg, James M. and Theodorou, Evangelos A.},
  journal = {IEEE Trans. on Robotics},
  title   = {Information-Theoretic Model Predictive Control: {Theory} and Applications to Autonomous Driving},
  year    = {2018},
  number  = {6},
  pages   = {1603--1622},
  volume  = {34}
}

@article{Thoma2023,
  author  = {Thoma, A. and Thomessen, K. and Gardi, A. and Fisher, A. and Braun, C.},
  journal = {The Aeronautical Journal},
  title   = {Prioritising Paths: {An} Improved Cost Function for Local Path Planning for {UAV} in Medical Applications},
  year    = {2023},
  number  = {1318},
  pages   = {2125--2142},
  volume  = {127}
}

@inproceedings{Yi2015,
  author    = {Yi, Daqing and Goodrich, Michael A. and Seppi, Kevin D.},
  booktitle = {Proc. of the Int. joint Conf. on Artificial Intelligence},
  title     = {{MORRF}: {Sampling}-Based Multi-Objective Motion Planning},
  year      = {2015},
  pages     = {1733--1739}
}

@article{Paden2016,
  author  = {Paden, Brian and Č{\'a}p, Michal and Yong, Sze Zheng and Yershov, Dmitry and Frazzoli, Emilio},
  journal = {IEEE Trans. on Intelligent Vehicles},
  title   = {A Survey of Motion Planning and Control Techniques for Self-Driving Urban Vehicles},
  year    = {2016},
  number  = {1},
  pages   = {33--55},
  volume  = {1}
}

@inproceedings{Linard2022,
  author    = {Linard, Alexis and Torre, Ilaria and Leite, Iolanda and T{\r{u}}mov{\'a}, Jana},
  booktitle = {Proc. of the {IEEE}/{RSJ} Int. Conf. on Intelligent Robots and Systems},
  title     = {Inference of Multi-Class {STL} Specifications for Multi-Label Human-Robot Encounters},
  year      = {2022},
  pages     = {1305--1311}
}

@article{Khanna2024,
  author  = {Khanna, Parag and Fredberg, Jonathan and Bj{\"o}rkman, M{\aa}rten and Smith, Christian and Linard, Alexis},
  journal = {IEEE Robotics and Automation Letters},
  title   = {Hand It to Me Formally! {D}ata-Driven Control for Human-Robot Handovers With Signal Temporal Logic},
  year    = {2024},
  number  = {10},
  pages   = {9039--9046},
  volume  = {9}
}

@article{Kohli2007,
  author  = {Kohli, Rajeev and Jedidi, Kamel},
  journal = {Marketing Science},
  title   = {Representation and Inference of Lexicographic Preference Models and their Variants},
  year    = {2007},
  number  = {3},
  pages   = {380--399},
  volume  = {26}
}

@inproceedings{Finkeldei2025,
  author    = {Finkeldei, Florian and Wolf, Michael and Weghorn, Jan-Niklas and Pretschner, Alexander and Althoff, Matthias},
  booktitle = {Proc. of the {IEEE} Int. Conf. on Intelligent Transportation Systems},
  title     = {Enhanced Traffic Rule Monitoring Using Model-Predictive and Duration-Aware Robustness},
  year      = {2025},
  pages     = {1322--1329}
}

@phdthesis{Varnai2022,
  author = {V{\'a}rnai, P{\'e}ter},
  school = {KTH Royal Institute of Technology},
  title  = {Path integral control endowed robot planning under spatiotemporal logic specifications},
  year   = {2022}
}

@phdthesis{Homburger2026,
  author = {Homburger, Hannes},
  school = {Albert-Ludwigs-Universität Freiburg},
  title  = {Optimal Control for Efficient Vessel Operation: {From} Theory to Real-World Applications},
  year   = {2026}
}

@inproceedings{Ziegler2010,
  author    = {Ziegler, Julius and Stiller, Christoph},
  booktitle = {Proc. of the {IEEE} Intelligent Vehicles Symposium},
  title     = {Fast collision checking for intelligent vehicle motion planning},
  year      = {2010},
  pages     = {518--522}
}

@inproceedings{Akiba2019,
  author    = {Akiba, Takuya and Sano, Shotaro and Yanase, Toshihiko and Ohta, Takeru and Koyama, Masanori},
  booktitle = {Proc. of the {ACM} {SIGKDD} Int. Conf. on Knowledge Discovery \& Data Mining},
  title     = {Optuna: {A} next-generation hyperparameter optimization framework},
  year      = {2019},
  pages     = {2623--2631}
}

@article{Trauth2024,
  author  = {Trauth, Rainer and Moller, Korbinian and Würsching, Gerald and Betz, Johannes},
  journal = {{IEEE} Access},
  title   = {{FRENETIX}: {A} High-Performance and Modular Motion Planning Framework for Autonomous Driving},
  year    = {2024},
  volume  = {12},
  pages   = {127426--127439}
}

@article{Biscani2020,
  title   = {A parallel global multiobjective framework for optimization: {Pagmo}},
  author  = {Biscani, Francesco and Izzo, Dario},
  journal = {Journal of Open Source Software},
  volume  = {5},
  number  = {53},
  pages   = {1--7},
  year    = {2020}
}

@article{Geisslinger2023,
  author  = {Geisslinger, Maximilian and Trauth, Rainer and Kaljavesi, Gemb and Lienkamp, Markus},
  journal = {IEEE Open Journal of Intelligent Transportation Systems},
  title   = {Maximum Acceptable Risk as Criterion for Decision-Making in Autonomous Vehicle Trajectory Planning},
  year    = {2023},
  volume  = {4},
  pages   = {570--579}
}

@article{Tong2025,
  author  = {Tong, Kailin and Steinberger, Martin and Horn, Martin and Solmaz, Selim and Watzenig, Daniel},
  journal = {IEEE Open Journal of Intelligent Transportation Systems},
  title   = {Anytime Optimal Trajectory Repairing for Autonomous Vehicles},
  year    = {2025},
  volume  = {6},
  pages   = {537--553}
}

@article{Westhofen2022,
  author  = {Westhofen, Lukas and Neurohr, Christian and Butz, Martin and Scholtes, Maike and Schuldes, Michael},
  journal = {IEEE Open Journal of Intelligent Transportation Systems},
  title   = {Using Ontologies for the Formalization and Recognition of Criticality for Automated Driving},
  year    = {2022},
  volume  = {3},
  pages   = {519--538}
}

@article{Trauth2023,
  author  = {Trauth, Rainer and Moller, Korbinian and Betz, Johannes},
  journal = {IEEE Open Journal of Intelligent Transportation Systems},
  title   = {Toward Safer Autonomous Vehicles: {Occlusion}-Aware Trajectory Planning to Minimize Risky Behavior},
  year    = {2023},
  volume  = {4},
  pages   = {929--942}
}

@article{Rajesh2023,
  author  = {Rajesh, Nishant and Zheng, Yanggu and Shyrokau, Barys},
  journal = {IEEE Open Journal of Intelligent Transportation Systems},
  title   = {Comfort-Oriented Motion Planning for Automated Vehicles Using Deep Reinforcement Learning},
  year    = {2023},
  volume  = {4},
  pages   = {348--359}
}

% ===================================
% Author Biographies
% ===================================
% Author Biographies
\begin{IEEEbiography}[{\includegraphics[width=1in, height=1.25in, clip, keepaspectratio]{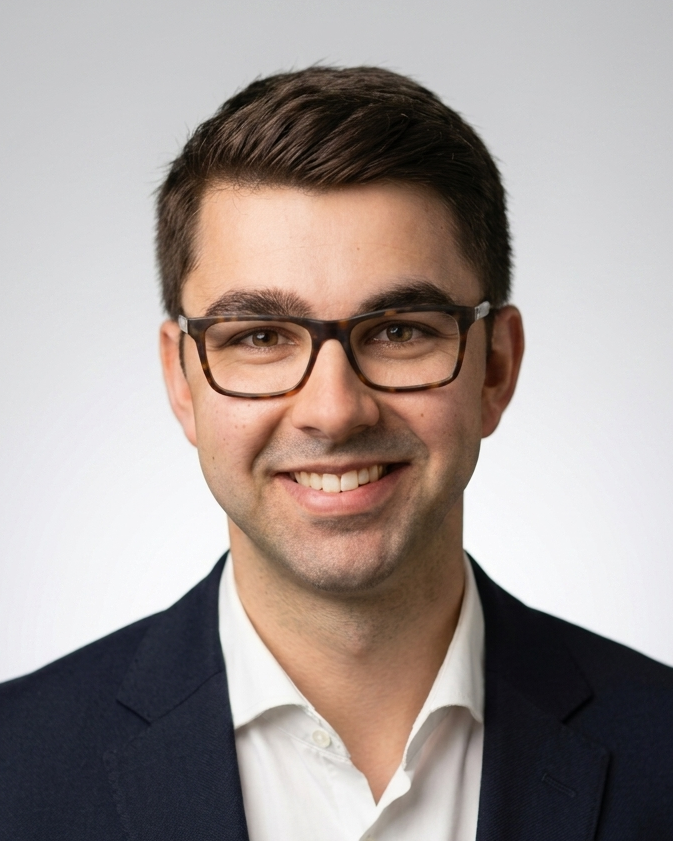}}]{Patrick Halder }
    is a Ph.D. candidate in the Cyber-Physical Systems Group at the Technical University of Munich, Germany, and is a software developer at ZF Friedrichshafen AG, Germany. He received his M.Sc. degree in Mechanical Engineering from the Technical University of Munich, Germany, in 2019, and his B.Eng. degree in Mechanical Engineering from Baden-Württemberg Cooperative State University Friedrichshafen, Germany, in 2016. His research interests include motion planning, formal methods, and multi-objective optimization.
\end{IEEEbiography}

\begin{IEEEbiography}[{\includegraphics[width=1in, height=1.25in, clip, keepaspectratio]{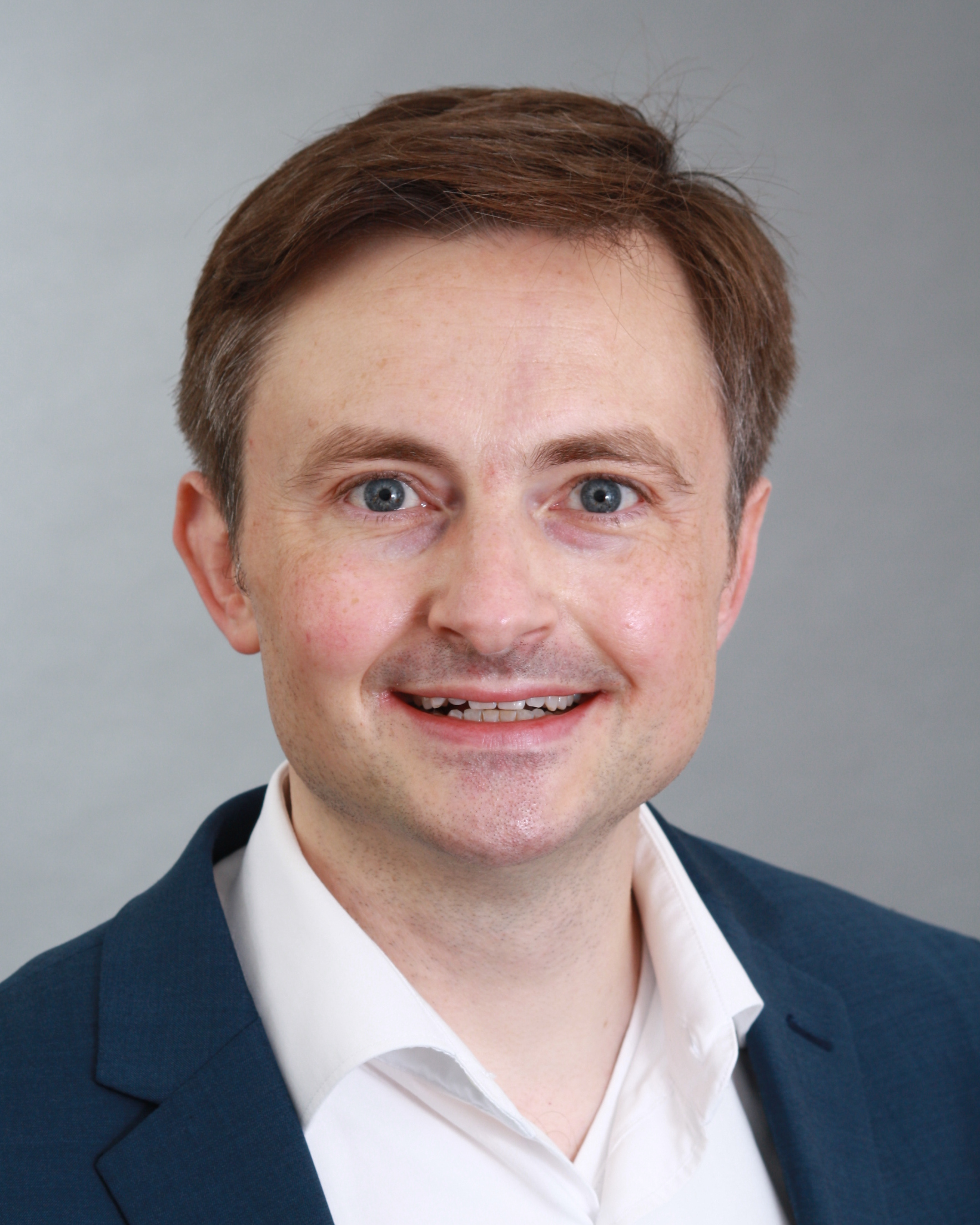}}]{Lothar Kiltz }
    received the Diploma degree in systems and electrical engineering and the Ph.D. degree in mechatronics engineering from Saarland University, Saarbrücken, Germany. He joined ZF Friedrichshafen AG in 2014 and is currently a Software Project Manager in the Industrial Technology Division. From 2020 to 2024, he led the Corporate Expert Team of Control Engineering. His research interests include modeling, estimation, and control of mechatronic systems and their robust, optimization-based design, as well as systems engineering with an emphasis on software-intensive, large-scale, variability-rich systems and their complexity and governance. He is an appointed member of the VDI/VDE Committee AUTOREG.
\end{IEEEbiography}

\begin{IEEEbiography}[{\includegraphics[width=1in, height=1.25in, clip, keepaspectratio]{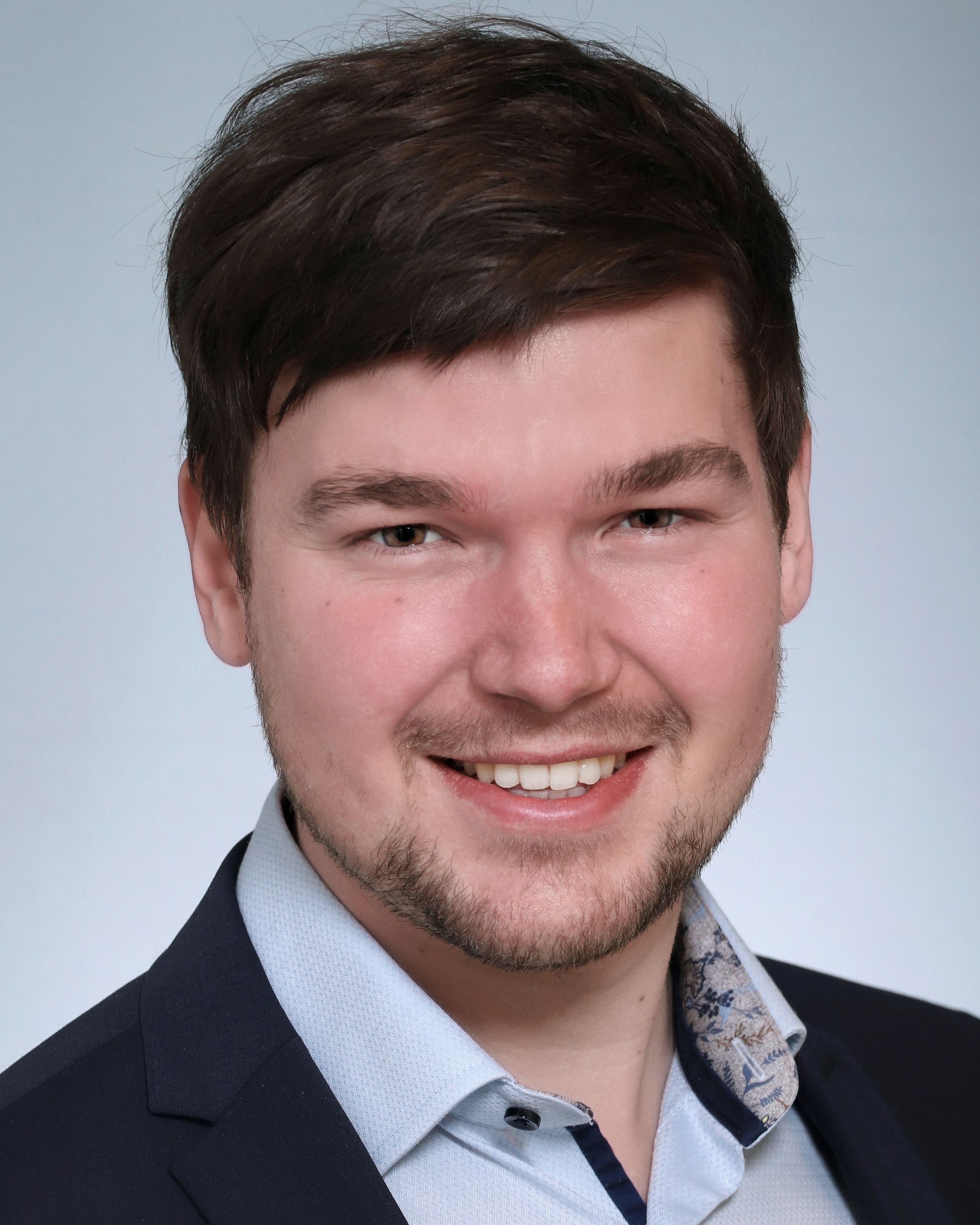}}]{Hannes Homburger }
    is a cooperative Ph.D. candidate at the Institute of System Dynamics at the University of Applied Sciences Konstanz, Germany, and the Systems Control and Optimization Laboratory at the University of Freiburg, Germany. He received his M.Eng. degree in Electrical Systems from the University of Applied Sciences Konstanz in 2021 and his B.Eng. degree in Electrical Engineering from the Baden-Württemberg Cooperative State University Stuttgart, Germany, in 2017. His research interests include optimal planning and model predictive control based on zero- and second-order methods with applications to autonomous systems.
\end{IEEEbiography}

\begin{IEEEbiography}[{\includegraphics[width=1in, height=1.25in, clip, keepaspectratio]{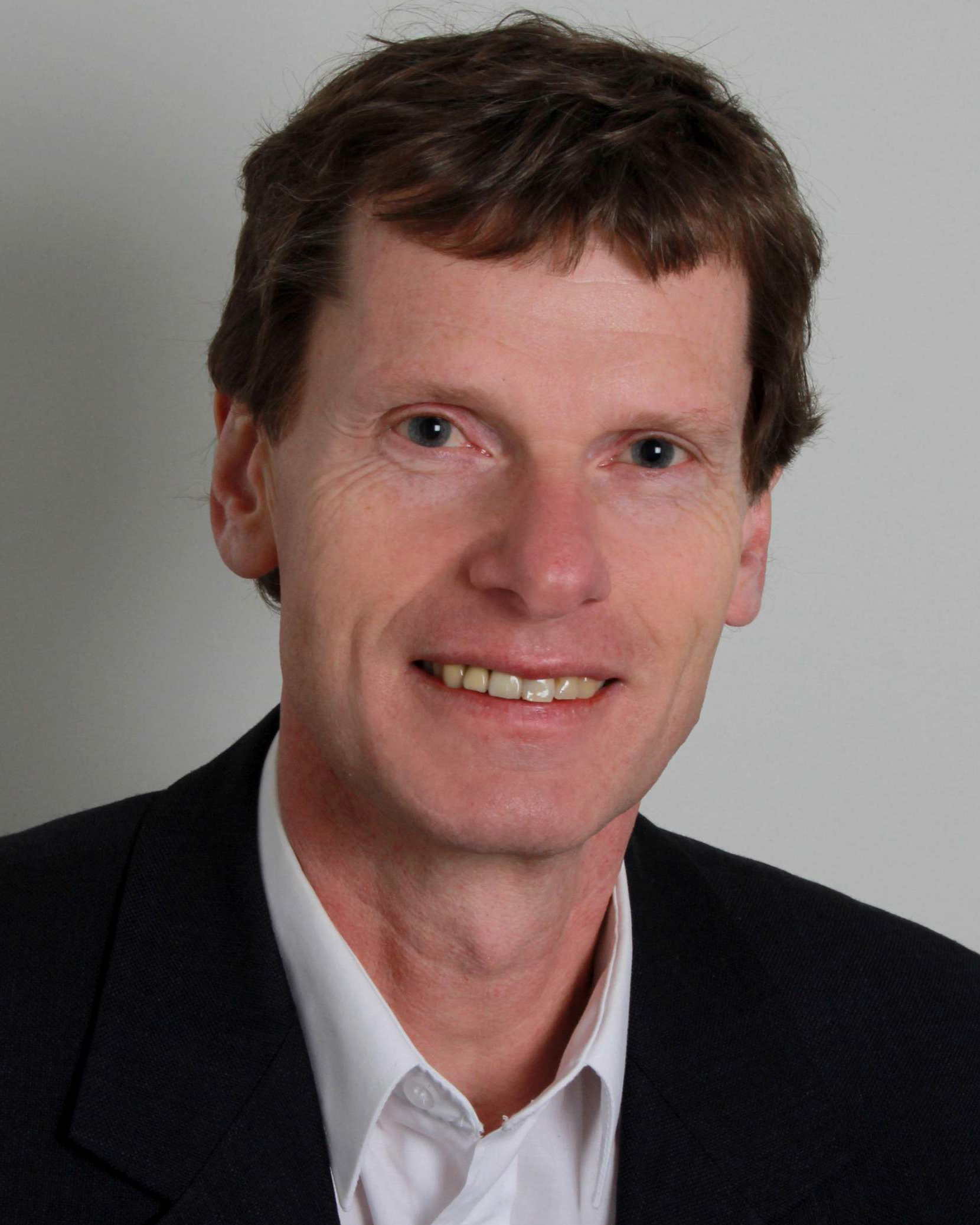}}]{Johannes Reuter }
    is a professor at the University of Applied Sciences Konstanz, Germany, since 2007, and the director of the Institute of System Dynamics at the University of Applied Sciences Konstanz since 2010. He received his Ph.D. degree from the Technical University of Berlin, Germany, in 2000. Between 2000 and 2007, he held various leadership positions at IAV GmbH in Berlin, Germany; IAV Automotive Engineering in Ann Arbor, USA; and the Eaton Corporation Innovation Center in Southfield, USA. His research interests include nonlinear control, multi-sensor data fusion, and autonomous systems.
\end{IEEEbiography}

\begin{IEEEbiography}[{\includegraphics[width=1in, height=1.25in, clip, keepaspectratio]{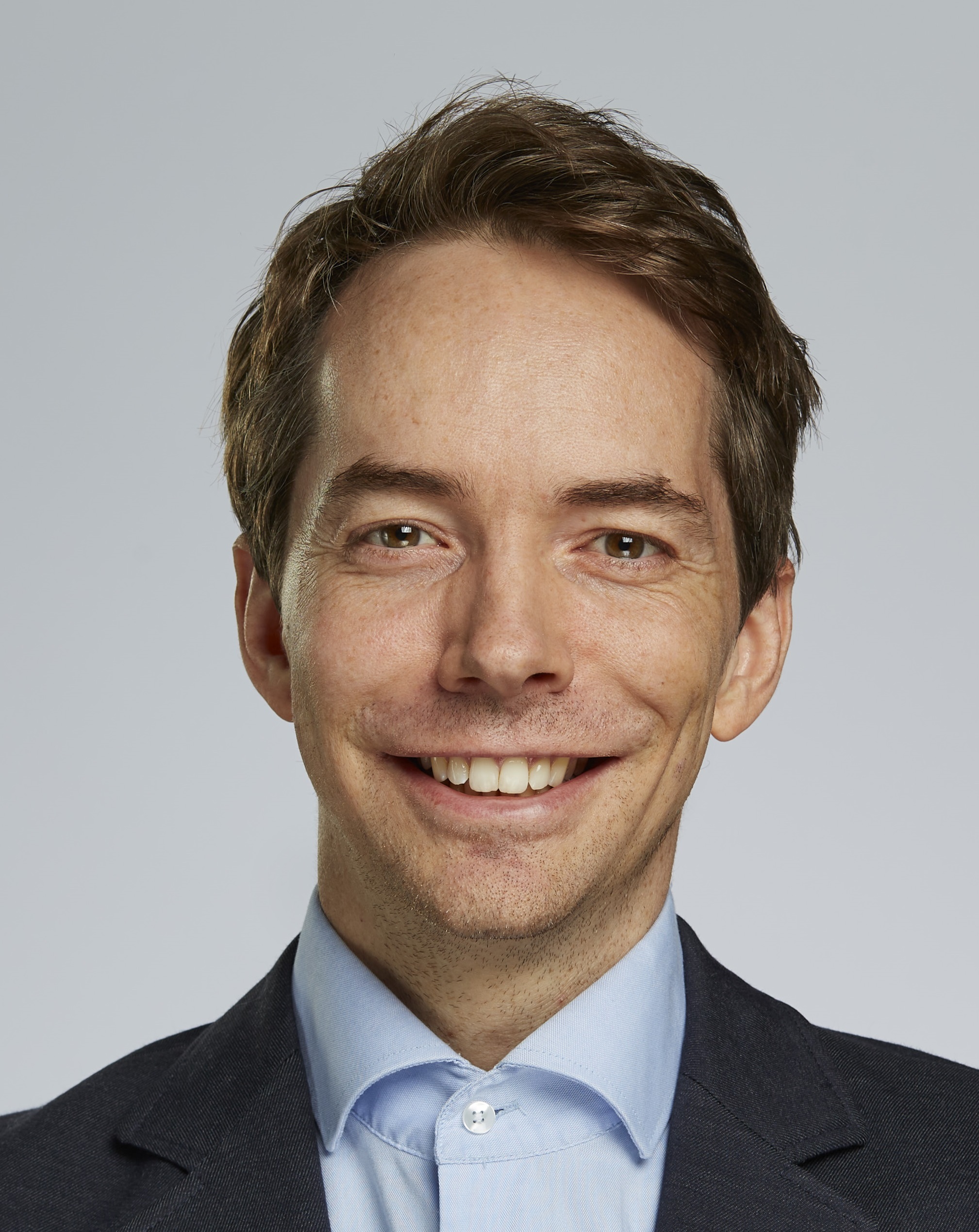}}]{Matthias Althoff }
    is an associate professor in computer science at the Technical University of Munich, Germany. He received his diploma engineering degree in Mechanical Engineering in 2005, and his Ph.D. degree in Electrical Engineering in 2010, both from the Technical University of Munich, Germany. From 2010 to 2012, he was a postdoctoral researcher at Carnegie Mellon University, Pittsburgh, USA, and from 2012 to 2013 an assistant professor at Technische Universität Ilmenau, Germany. His research interests include formal verification of continuous and hybrid systems, reachability analysis, planning algorithms, nonlinear control, automated vehicles, and power systems.
\end{IEEEbiography}

\end{document}